\documentclass[10pt]{article}
\usepackage{setspace}
\usepackage{cprotect}
\usepackage{amsmath, amsfonts, amsthm, amssymb, bm, bbm}
\usepackage[noend]{algorithmic}
\usepackage[ruled,vlined]{algorithm2e}
\usepackage[bookmarks=true,
            unicode=true,
                    bookmarksopen=true,
                    bookmarksnumbered=true,
                    colorlinks=true,
                    linkcolor=blue,
                    citecolor=ao,
                    ]{hyperref}
\usepackage{fullpage}
\usepackage{makeidx}
\usepackage{enumerate}
\usepackage{wrapfig}
\usepackage{graphicx,float,psfrag,epsfig}
\usepackage{epstopdf}
\usepackage{color}
\usepackage{enumitem}
\usepackage{subfig}
\usepackage{caption}
\usepackage{bigints}
\usepackage{mathtools}
\usepackage[mathscr]{euscript}
\usepackage{cite}
\definecolor{ao}{rgb}{0.0, 0.5, 0.0}
\DeclareSymbolFont{rsfs}{U}{rsfs}{m}{n}
\DeclareSymbolFontAlphabet{\mathscrsfs}{rsfs}

\numberwithin{equation}{section}

\newtheoremstyle{myexample} 
    {\topsep}                    
    {\topsep}                    
    {\rm }                   
    {}                           
    {\bf }                   
    {.}                          
    {.5em}                       
    {}  

\newtheoremstyle{myremark} 
    {\topsep}                    
    {\topsep}                    
    {\rm}                        
    {}                           
    {\bf}                        
    {.}                          
    {.5em}                       
    {}  

\newtheorem{claim}{Claim}[section]
\newtheorem{lemma}[claim]{Lemma}

\newtheorem{theorem}[claim]{Theorem}
\newtheorem{proposition}[claim]{Proposition}

\newtheorem*{theorem*}{Theorem}

\def\H{\mathcal{H}}
\def\L{\mathcal{L}}
\def\M{\mathcal{M}}
\def\R{\mathbb{R}}

\def\bB{{\boldsymbol B}}
\def\bA{{\boldsymbol A}}
\def\bx{{\boldsymbol x}}

\def\v{{\boldsymbol v}}
\def\u{{\boldsymbol u}}

\def\bGamma{{\boldsymbol \Gamma}}
\def\by{{\boldsymbol y}}

\def\bI{{\boldsymbol I}}

\def\0{{\boldsymbol 0}}
\def\1{{\boldsymbol 1}}
\def\x{{\boldsymbol x}}
\def\a{{\boldsymbol a}}
\def\b{{\boldsymbol b}}
\def\A{{\boldsymbol A}}
\def\B{{\boldsymbol B}}
\def\D{{\boldsymbol D}}
\def\U{{\boldsymbol U}}
\def\V{{\boldsymbol V}}
\def\Q{{\boldsymbol Q}}
\def\P{{\boldsymbol P}}
\def\T{{\boldsymbol T}}
\def\X{{\boldsymbol X}}
\def\Y{{\boldsymbol Y}}
\def\J{{\boldsymbol J}}
\def\bR{{\boldsymbol R}}
\def\bE{{\boldsymbol E}}
\def\Z{{\boldsymbol Z}}
\def\M{{\boldsymbol M}}
\def\N{{\boldsymbol N}}
\def\O{{\boldsymbol O}}
\def\u{{\boldsymbol u}}
\def\I{{\boldsymbol I}}
\def\g{{\boldsymbol g}}
\def\H{{\boldsymbol H}}
\def\1{{\boldsymbol 1}}
\def\bSigma{{\boldsymbol \Sigma}}

\def\Lam{{\boldsymbol\Lambda}}
\def\S{{\boldsymbol{S}}}

\usepackage{dsfont}
\def\E{{\mathbb{E}}}

\DeclareMathOperator*{\argmin}{arg\,min}
\newcommand{\tr}[1]{\mathrm{Tr}\left[ #1 \right]}

\newcommand{\norm}[1]{\left\lVert #1 \right\rVert}
\newcommand{\psin}[1]{\left\lVert #1 \right\rVert_{\psi_2}}
\newcommand{\abs}[1]{\left\lvert#1\right\rvert}
\newcommand\ind{{\ensuremath {\mathds 1} }}

\def\lam{\lambda}
\newcommand{\opn}[1]{\left\lVert #1 \right\rVert_{op}}
\newcommand{\opv}[1]{\left\lVert #1 \right\rVert_{2}}

\makeatletter
\newcommand{\myfnsymbol}[1]{%
  \expandafter\@myfnsymbol\csname c@#1\endcsname
}
\newcommand{\@myfnsymbol}[1]{%
  \ifcase #1
  \or 1
  \or 2
  \or \TextOrMath{\textasteriskcentered}{*}
  \or \TextOrMath{\textdagger}{\dagger}
  \fi
}
\newcommand{\affiliationA}{\@myfnsymbol{1}}
\newcommand{\affiliationB}{\@myfnsymbol{2}}
\newcommand{\equalcontributor}{\@myfnsymbol{3}}
\newcommand{\correspondingA}{\@myfnsymbol{4}}
\makeatother
\title{Fundamental Limits of Two-layer Autoencoders,\\ and Achieving Them with Gradient Methods
}
\date{}
\author{%
  Alexander Shevchenko\textsuperscript{\affiliationA,\equalcontributor}
   \and
   Kevin Kögler
   \textsuperscript{\affiliationA,\equalcontributor}
  \and
  Hamed Hassani\textsuperscript{\affiliationB}
  \and
  Marco Mondelli\textsuperscript{\affiliationA}
 }

\begin{document}

\renewcommand{\thefootnote}{\myfnsymbol{footnote}}
\maketitle
\footnotetext[1]{Institute of Science and Technology Austria}%
\footnotetext[2]{Department of Electrical and Systems Engineering, University of Pennsylvania}%
\footnotetext[3]{Authors contributed equally. Corresponding authors: \texttt{alex.shevchenko@ist.ac.at}, \texttt{kevin.koegler@ist.ac.at}}%

\setcounter{footnote}{0}
\renewcommand{\thefootnote}{\arabic{footnote}}

\begin{abstract}
Autoencoders are a popular model in many branches of machine learning and lossy data compression. However, their fundamental limits, the performance of gradient methods and the features learnt during optimization remain poorly understood, even in the two-layer setting. In fact, earlier work has considered either linear autoencoders or specific training regimes (leading to vanishing or diverging compression rates). 
Our paper addresses this gap by focusing on non-linear two-layer autoencoders trained in the challenging proportional regime in which the input dimension scales linearly with the size of the representation. Our results characterize 
the minimizers of the population risk, and show that such minimizers are achieved by gradient methods; their structure is also unveiled, thus leading to a concise description of the features obtained via training. 
For the special case of a sign activation function, 
our analysis establishes the fundamental limits for the lossy compression of Gaussian sources via (shallow) autoencoders. Finally, while the results are proved for Gaussian data, numerical simulations on standard datasets display the universality of the theoretical predictions.
\end{abstract}

\section{Introduction}
Autoencoders represent a
key building block in many branches of machine learning \cite{kingmaauto, rezende2014stochastic}, including generative modeling \cite{bengio2013generalized} and representation learning \cite{tschannen2018recent}. Prompted by the fact that autoencoders learn succinct representations,
neural autoencoding techniques have also
achieved remarkable success in lossy data compression, even outperforming classical methods, such as jpeg \cite{Balle2017, Theis2017a, SoftToHardVQ}. However, despite the large body of empirical work considering neural autoencoders and compressors, the most basic theoretical questions remain poorly understood even in the shallow case: 

\vspace{0.4em}
\begin{center}   
\begin{minipage}{0.85\textwidth}
\textit{What are the fundamental performance limits
of autoencoders? Can we achieve such limits with gradient methods? What features does the optimization procedure learn?
}
\end{minipage}
\end{center}
\vspace{0.4em}

Prior work has focused either on linear autoencoders \cite{baldi1989neural,kunin2019loss,gidel2019implicit}, on the severely under-parameterized setting in which the input dimension is much larger than the number of neurons \cite{pmlr-v162-refinetti22a}, or on specific training regimes (lazy training \cite{nguyen2021benefits} and mean-field regime with a polynomial number of neurons \cite{nguyen2021analysis}), see Section \ref{sec:related} for more details. In contrast, in this  paper we consider \emph{non-linear} autoencoders trained in the \emph{challenging proportional regime}, in which the number of inputs to compress scales linearly with the size of the representation.
More specifically, we consider the prototypical model of a two-layer autoencoder 
\begin{equation}\label{eq:model}
\hat{\x}(\x) := \hat{\x}(\x,\A,\B) = \A \sigma(\B\x).
\end{equation}
Here, $\x\in\mathbb{R}^d$ is the input vector to compress, $\hat{\x}\in\mathbb{R}^n$ the reconstruction, $\B\in\mathbb{R}^{n\times d}$ the encoding matrix, and $\A\in\mathbb{R}^{d\times n}$ the decoding matrix; the activation $\sigma:\mathbb{R}\rightarrow\mathbb{R}$ is applied \emph{element-wise} on its argument. We aim at minimizing the population risk
\begin{equation}\label{eq:loss}
\mathcal{R}(\A,\B) := d^{-1}\E_{\x} \opv{\x - \hat{\x}(\x)}^2,
\end{equation}
where the expectation is taken over the distribution of the input $\x$. Our focus is on Gaussian input data, i.e., $\x \sim \mathcal{N}(\0,\boldsymbol{\Sigma})$.
When the activation $\sigma$ is the sign function, the encoder $\sigma(\B\x)$ can be interpreted as a \emph{compressor}, namely, it compresses the $d$-dimensional input signal into $n$ bits. The problem \eqref{eq:loss} of compressing a Gaussian source with quadratic distortion has been studied in exquisite detail in the information theory literature \cite{CoverThomas}, and the optimal performance for general encoder/decoder pairs is known through the so-called rate-distortion formalism which characterizes the lowest achievable distortion in terms of the rate $r = n/d$. Here, we focus on encoders and decoders that form the two-layer autoencoder \eqref{eq:model}: we study the fundamental limits of this learning problem,
as well as the performance achieved by commonly used gradient descent methods.

\begin{figure}[t]
\begin{tabular}{@{}cc@{}@{}cc@{}}
    \raisebox{-\height}{\includegraphics[width=0.35\textwidth]{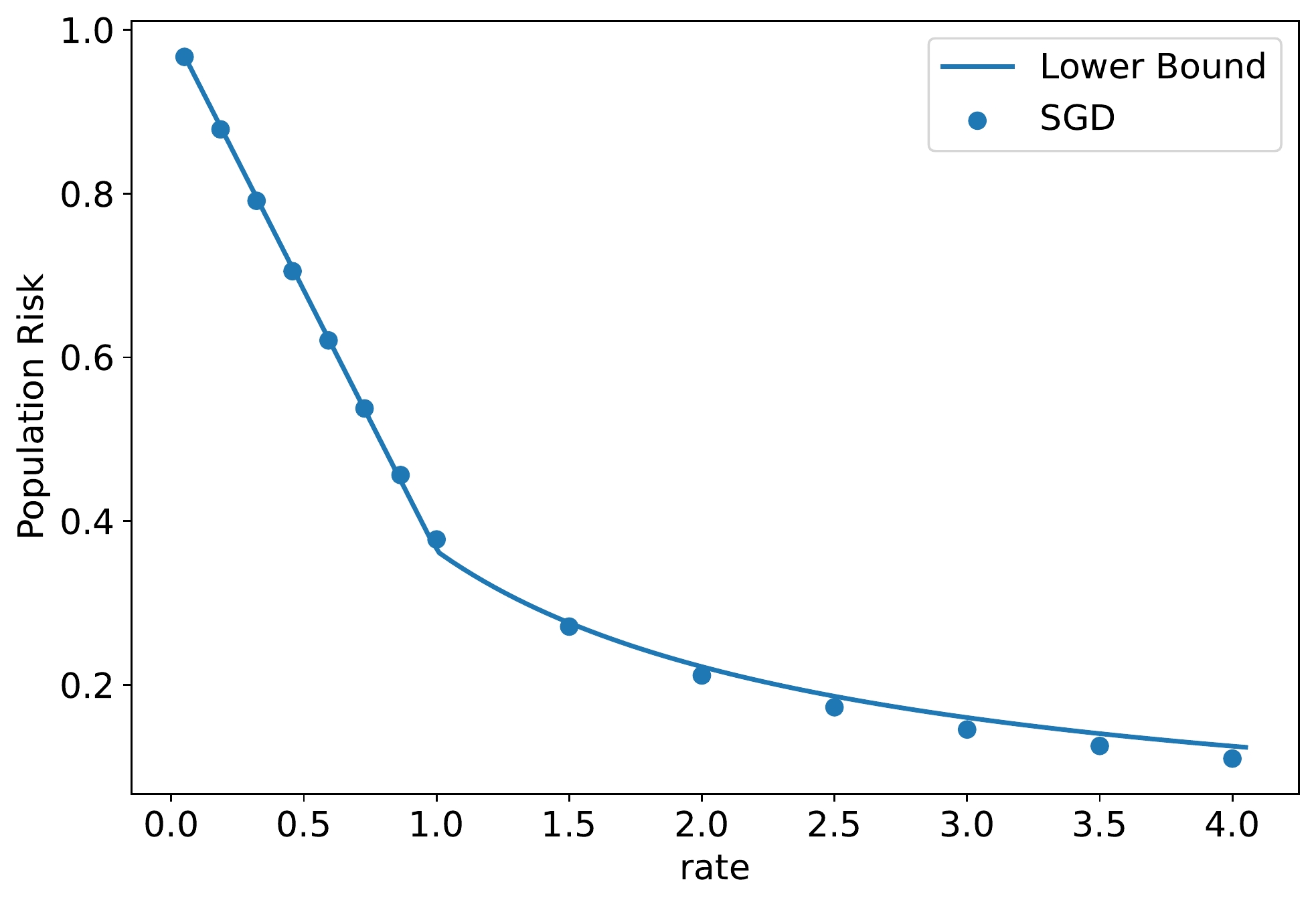}} & 
    \begin{tabular}[t]{@{}cc@{}}
        \raisebox{-\height}{\includegraphics[width=0.11\textwidth]{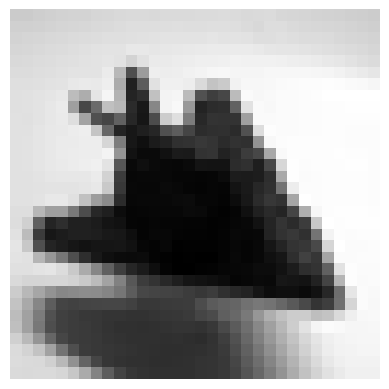}} &  \\[1cm]
        \raisebox{-\height}{\includegraphics[width=0.11\textwidth]{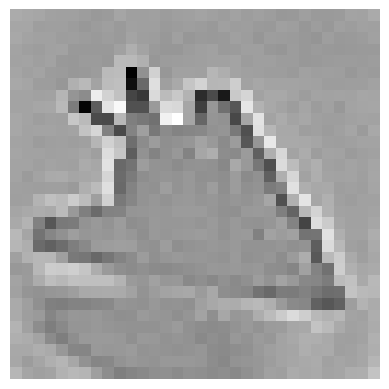}} & 
    \end{tabular}
    \raisebox{-\height}{\includegraphics[width=0.35\textwidth]{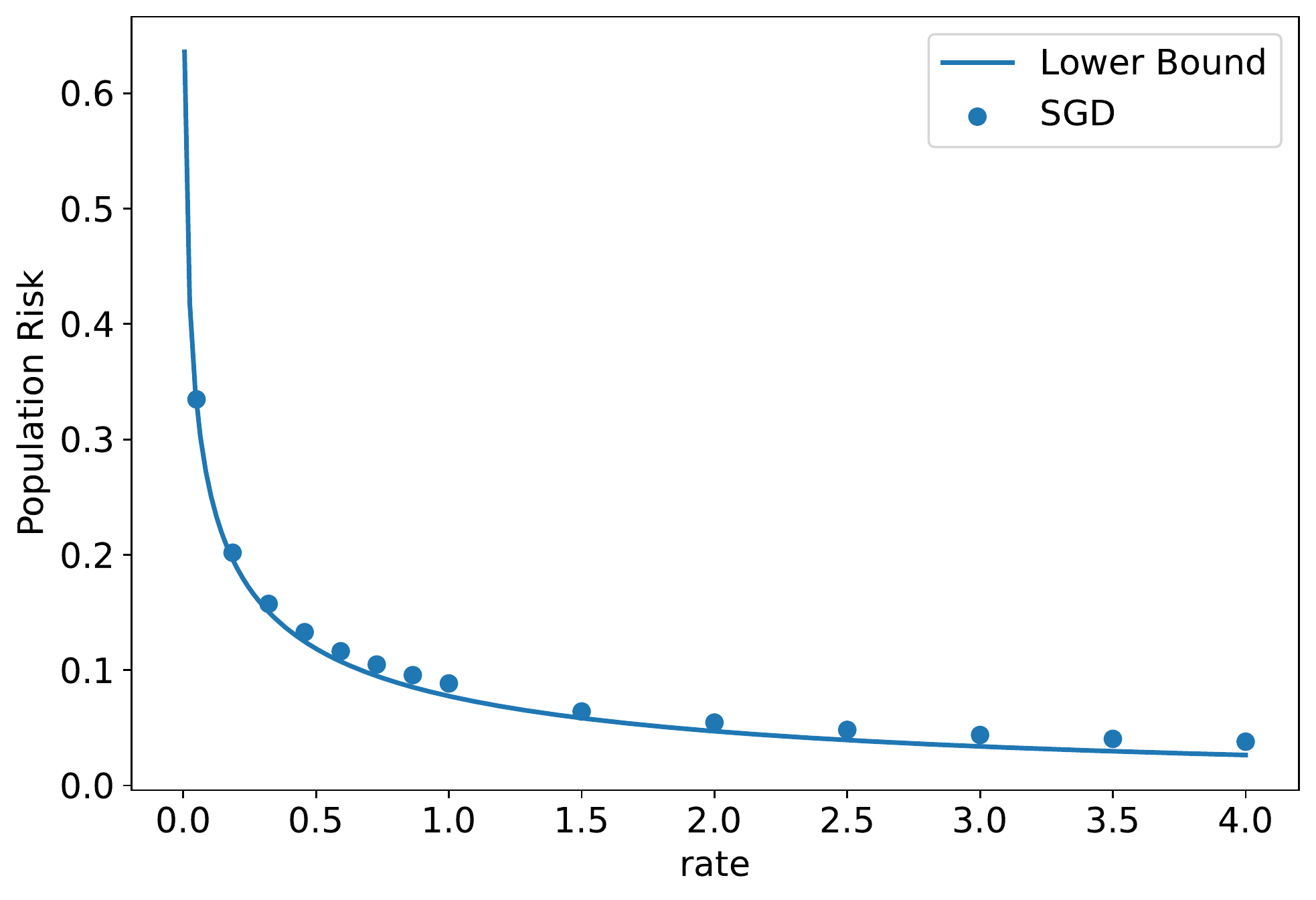}} & 
     \hspace{1.2em}\begin{tabular}[t]{@{}cc@{}}
        \raisebox{-\height}{\includegraphics[width=0.11\textwidth]{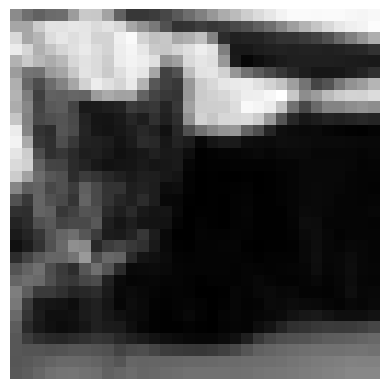}} &  \\[1cm]
        \raisebox{-\height}{\includegraphics[width=0.11\textwidth]{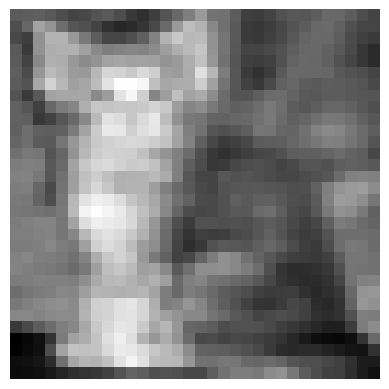}} & 
    \end{tabular}
\end{tabular}
\caption{ \textbf{Left plot.} Compression ($\sigma\equiv {\rm sign}$) of the grayscale CIFAR-10 ``airplane'' class with a two-layer autoencoder. The data is \emph{whitened} so that $\boldsymbol{\Sigma}=\bI$: on top, an example of a grayscale image; on the bottom, the corresponding whitening. The blue dots are the population risk obtained via SGD, and they agree well with the solid line corresponding to the lower bounds of Theorem \ref{thm:tightlb_lowrate} and Proposition \ref{prop:highratelb}. \textbf{Right plot.} Compression ($\sigma\equiv {\rm sign}$) of the grayscale CIFAR-10 ``cat'' class with a two-layer autoencoder. The data is \emph{not whitened} ($\boldsymbol{\Sigma}\neq \bI$). The blue dots are the SGD population risk, and they are close to the lower bound of Theorem \ref{thm:wtD_lb}.\vspace{-1.2em}}

\label{fig:cifar_white}
\end{figure}

\paragraph{Main contributions.} Taken all together, our results show that, for two-layer autoencoders, gradient descent methods achieve a global minimizer of the population risk: this is rigorously proved in the isotropic case ($\boldsymbol{\Sigma}=\bI$) and corroborated by numerical simulations for a general covariance $\boldsymbol{\Sigma}$. Furthermore, we unveil the structure of said minimizer: for $\boldsymbol{\Sigma}=\bI$, the optimal decoder has unit singular values; for general covariance, the spectrum of the decoder exhibits the same block structure as $\bSigma$, and it can be explicitly obtained from $\bSigma$ via a water-filling criterion; in all cases, weight-tying is optimal, i.e., $\A$ is proportional to $\B^\top$. Specifically, our technical results can be summarized as follows.

\begin{itemize}[leftmargin=0.1in]
    \item Section \ref{subsec:lb} characterizes the minimizers of the risk \eqref{eq:loss} for isotropic data: 
    Theorem \ref{thm:tightlb_lowrate} provides a tight lower bound, which is achieved by the set \eqref{haar:min} of weight-tied \emph{orthogonal} matrices, when the compression rate $r=n/d\le 1$; for $r>1$, Propositions \ref{prop:highratelb} and \ref{prop:highrate_min} give a lower bound, which is approached (as $d\to\infty$) by the set \eqref{eq:defoptB} of weight-tied \emph{rotationally invariant} matrices. 
     
    \item Section \ref{subsection:opt_results} shows that the above minimizers are reached by gradient descent methods for $r\le 1$: Theorem \ref{thm:wt_gf} shows linear convergence of \emph{gradient flow} for general initializations, under a weight-tying condition; Theorem \ref{thm:gd_min} considers a Gaussian initialization and proves global convergence of the \emph{projected gradient descent} algorithm, in which the encoder matrix $\B$ is optimized via a gradient method and the decoder matrix $\A$ is obtained directly via linear regression.

\item Section \ref{section:general_cov} focuses on data with general covariance $\boldsymbol{\Sigma}\neq \bI$. We observe that experimentally weight-tying is optimal and then derive the corresponding lower bound (see Theorem \ref{thm:wtD_lb}), which is also asymptotically achieved (as $d\to\infty$) by rotationally invariant matrices with a carefully designed spectrum (depending on $\boldsymbol{\Sigma}$), see Proposition \ref{prop:Dlb_achievability}.
\end{itemize}

When $\sigma \equiv {\rm sign}$, 
our analysis characterizes the fundamental limits of the lossy compression of a Gaussian source via two-layer autoencoders. Remarkably, if we restrict to a certain class of \emph{linear encoders} for compression, two-layer autoencoders achieve optimal performance \cite{tulino2013support}, which can be generally obtained via a message passing decoding algorithm \cite{rangan2019vector}. However, for \emph{general encoder/decoder pairs}, shallow autoencoders fail to meet the information-theoretic bound given by the rate-distortion curve, see
Section \ref{sec:discussion}. 

Going beyond the Gaussian assumption on the data, we provide numerical validation to our theoretical predictions on standard datasets, both in the isotropic case and for the general covariance (see Figure \ref{fig:cifar_white}). Additional numerical results -- together with the details of the experimental setting -- are in Appendix \ref{appendix:numerics}.

\paragraph{Proof techniques.} 
The lower bound on the population risk of Theorem \ref{thm:tightlb_lowrate} comes from a sequence of relaxations of the objective function, which eventually allows to apply a trace inequality. For $r\ge 1$, Proposition \ref{prop:highratelb} crucially exploits an inequality for the Hadamard product of PSD matrices \cite{khare2021sharp}, and the asymptotic achievability of Proposition \ref{prop:highrate_min} takes advantage of concentration-of-measure tools for orthogonal matrices. The key quantity in the analysis of gradient methods is the encoder Gram matrix at iteration $t$, i.e., $\B(t)\B(t)^\top$. In particular, for gradient flow (Theorem \ref{thm:wt_gf}), due to the weight-tying condition, tracking $\log \det \B(t)\B(t)^\top$ leads to a quantitative convergence result. However, when the weights are not tied, this quantity does not appear to decrease along the optimization trajectory. Thus, for projected gradient descent (Theorem \ref{thm:gd_min}), the idea is to decompose $\B(t)\B(t)^\top$ into \emph{(i)} its value at the optimum (given by the identity), \emph{(ii)} the contribution due to the spectrum evolution (keeping the eigenbasis fixed), and \emph{(iii)} the 
change in the eigenbasis. Via a sequence of careful approximations, we are able to show that the term \emph{(iii)} vanishes. Having obtained that, we can study explicitly the evolution of the spectrum and obtain the desired convergence.

\section{Related Work}\label{sec:related}

\paragraph{Theory of autoencoders.} A popular line of work has focused on two-layer \emph{linear} autoencoders: \cite{oftadeh2020eliminating} analyzes the loss landscape; 
 \cite{kunin2019loss} shows that the minimizers of the regularized loss recover the principal components of the data and, notably, the corresponding autoencoder is weight-tied; \cite{bao2020regularized} 
proves that
stochastic gradient descent -- after a slight perturbation --  escapes the saddles and eventually converges; \cite{gidel2019implicit} also analyzes the gradient dynamics and characterizes the time-steps at which the  network learns different sets of features. 
 \cite{rangamani2018sparse,nguyen2019dynamics} prove local convergence for weight-tied two-layer ReLU autoencoders. \cite{nguyen2021benefits} focuses on the lazy training regime \cite{chizat2019lazy,jacot2018neural} and bounds the over-parameterization needed for global convergence. \cite{radhakrishnan2020overparameterized} takes a dynamical systems perspective and shows that over-parameterized autoencoders learn solutions that are contractive around the training examples. The latent spaces of autoencoders are studied in \cite{jain2021mechanism}, where it is shown that such latent spaces can be aligned by stretching along the left singular vectors of the data. More closely related to our work, \cite{nguyen2021analysis} and \cite{pmlr-v162-refinetti22a} track the gradient dynamics of \emph{non-linear} two-layer autoencoders via the mean-field PDE and a system of ODEs, respectively. However, these analyses are restricted to diverging and vanishing rates: \cite{nguyen2021analysis}
 considers 
 weight-tied autoencoders with polynomially many neurons in the input dimension (so that $r\to\infty$); \cite{pmlr-v162-refinetti22a} considers the other extreme regime in which the input dimension diverges (so that $r\to 0$).

\paragraph{Neural compression.}
In recent years, compressors based on neural networks have been able to 
outperform traditional schemes 
on real-world data in terms of minimizing distortion and producing visually pleasing reconstructions at reasonable complexity \cite{Balle2017, Theis2017a, SoftToHardVQ, NTC}. These methods typically use an autoencoder architecture with quantization of the latent variables, which is trained over samples drawn from the source. More recently, other architectures such as attention models or diffusion-based models have been incorporated into neural compressors \cite{cheng2020learned, liu2019non, yang2022lossy,theis2022lossy}, and improvements have been observed.  We refer to \cite{yang2022introduction} for a detailed review on this topic.  Given the remarkable success of neural compressors, it is imperative to understand the fundamental limits of compression using neural architectures. In this regard, \cite{Wagner2021NeuralNO} considers a highly-structured and  low-dimensional random process, dubbed the \emph{sawbridge}, and 
 shows numerically that the rate-distortion function 
 is achieved by a compressor based on deep neural networks trained via stochastic gradient descent. In contrast, our work considers Gaussian sources, which are high-dimensional in nature, and provides the fundamental limits of compression when two-layer autoencoders are used. Our results also imply that two-layer autoencoders cannot achieve the rate-distortion limit on Gaussian data, see Section \ref{sec:discussion}.

\paragraph{Rate-distortion formalism.} 
 Lossy compression of stationary sources is a classical problem in information theory, and 
 several approaches have been proposed, including vector quantization \cite{gray1984vector}, or the usage of powerful channel codes 
 \cite{korada2010polar,ciliberti2006message,wainwright2010lossy}.
 The rate-distortion function  characterizes the optimal trade-off between error and size of the representation for the compression of an i.i.d. source \cite{shannon48, shannonRD, CoverThomas}. However,  computing the rate-distortion function is by itself a challenging task. The Blahut-Arimoto scheme \cite{blahut, arimoto} provides a systematic approach, but it suffers from the issue of scalability \cite{lei2022neural}. Consequently, to compute the rate-distortion of empirical datasets, approximate methods based on generative modeling have been proposed \cite{yang2021towards, lei2022neural}. 

\paragraph{Non-linear inverse problems.} The task of estimating a signal $\x$ from non-linear measurements $\by = \sigma(\B \x)$ has appeared in many areas, such as 1-bit compressed sensing where $\sigma(z) = \text{sign}(z)$ \cite{boufounos20081}, or phase retrieval where $\sigma(z) = |z|$ \cite{candes2013phaselift, candes2015phase}.  While the focus of these problems is different from ours (e.g., compressed sensing has often an additional sparsity assumption), the ideas and proof techniques developed in this paper might be beneficial to characterize the fundamental limits and the performance of gradient-based methods for general inverse reconstruction tasks, see e.g.  \cite{ma2021analysis, matsumoto2022binary}.

\section{Preliminaries}

\paragraph{Notations.} We use plain symbols for real numbers (e.g., $a,b$), bold symbols for vectors (e.g., $\boldsymbol{a}, \boldsymbol{b}$), and capitalized bold symbols for matrices (e.g., $\boldsymbol{A}, \boldsymbol{B}$). We define $[n] = \{1, \ldots, n\}$, denote by $\boldsymbol{I}$ the identity matrix and by $\1$ the column vector containing ones.
Given a matrix $\A$, we denote its operator norm by $\opn{\A}$ and its Frobenius norm by $\|\A\|_F$.
Given two matrices $\A$ and $\B$ of the same shape, we denote their element-wise (Hadamard/Schur) product by $\A \circ \B$ and the $k$-th element-wise power by $\A^{\circ k}$. We write $L^2(\mathbb{R}, \mu)$ for the space of $L^2$ integrable functions on $\mathbb{R}$ w.r.t. the standard Gaussian measure $\mu$ and $h_k(x)$ for the $k$-th normalized Hermite polynomial (see e.g. \cite{o2014analysis}).

\paragraph{Setup.} We consider the two-layer autoencoder \eqref{eq:model} and aim at minimizing the population risk \eqref{eq:loss} for a given rate $r = n/d$. In particular, we provide tight lower bounds on the minimum of the population risk computed on Gaussian input data with covariance $\boldsymbol{\Sigma}$, i.e.,
\begin{equation}\label{eq:mmse}
    \widehat{\mathcal{R}}(r) := \min_{\A,\B} \mathcal{R}(\A, \B),
\end{equation}
In the isotropic case ($\boldsymbol{\Sigma}=\boldsymbol{I}$), our results hold for any odd activation $\sigma \in L^2(\mathbb{R}, \mu)$ after restricting the rows of the encoding matrix $\B$ to have unit norm. We remark that, when $\sigma(x)={\rm sign}(x)$, 
the restriction is unnecessary since the activation is homogeneous.\footnote{We say that a function $\sigma$ is homogeneous if there exists an integer $k$ s.t. $\sigma(\alpha x) = \alpha^k \sigma(x)$ for all $\alpha \neq 0$.}
We also note that restricting the norms of the rows of $\B$ prevents the model from entering the ``linear'' regime.
In fact, when $\|\B\|_{F} \approx 0$, by linearizing the activation around zero, \eqref{eq:model} reduces to the linear model
$
\hat{\x}(\x) \approx \A\B \x,
$
which exhibits a PCA-like behaviour.
 For general covariance $\boldsymbol{\Sigma}$, we consider odd homogeneous activations, which includes the sign function and monomials of arbitrary odd degree.

Any function $\sigma \in L^2(\mathbb{R}, \mu)$ admits an expansion in terms of Hermite polynomials. This allows to perform Fourier analysis in the Gaussian space $L^2(\mathbb{R}, \mu)$, and it provides a natural tool because of the Gaussian assumption on the data. In particular, for odd $\sigma$, only odd Hermite polynomials occur, i.e.,
\begin{equation}\label{eq:hermite_expansion}
    \sigma(x) = \sum_{\ell=0}^\infty c_{2\ell+1} h_{2\ell + 1}(x),
\end{equation}
where $\{c_{\ell}\}_{\ell\in \mathbb N}$ denote the Hermite coefficients of $\sigma$.
We also
consider the following auxiliary quantity
\begin{equation}\label{eq:mmse_restricted}
    \widetilde{\mathcal{R}}(r) := \min_{\A,\|(\B\D)_{i,:}\|_2=1} \mathcal{R}(\A, \B),
\end{equation}
that defines a minimum of the population risk for the autoencoder \eqref{eq:model} with a certain norm constraint on the encoder weights $\B$. Here, $\D$ contains the square roots of the eigenvalues of $\boldsymbol{\Sigma}$ (i.e., $\boldsymbol{\Sigma} = \U \D^2 \U^\top$ for an orthogonal matrix $\U$), and $(\B\D)_{i,:}$ stands for the $i$-th row of the matrix $\B\D$.
A few remarks about the restricted population risk \eqref{eq:mmse_restricted} are in order. First of all, if $\sigma$ is homogeneous, the minimum of the restricted population risk \eqref{eq:mmse_restricted} and of the unconstrained one \eqref{eq:mmse} coincide (see Lemma \ref{lemma:closed_from_of_population_risk} and Lemma \ref{lemma:popriskD}). Thus, in this case, the analysis of $\widetilde{R}(r)$ will directly provide results on the quantity of interest, i.e., $\widehat{\mathcal{R}}(r)$. The technical advantage of analysing \eqref{eq:mmse_restricted} over \eqref{eq:mmse} comes from fact that the expectation with respect to the Gaussian inputs, which arises in the constrained objective, can be explicitly computed via the reproducing property of Hermite polynomials (see, e.g., \cite{o2014analysis}). To exploit this reproducing property, it is crucial that the inner products
$\langle\B_{i,:}, \x\rangle$ have the same scale, which is ensured by picking $\|(\B\D)_{i,:}\|_2 = 1$. The sole dependence of the constraint on the spectrum $\D$ (and, not on a particular choice of $\U$) stems from the rotational invariance of the isotropic Gaussian distribution.

\section{Main Results}\label{section:iso}

In this section, we consider isotropic Gaussian data, i.e., $\boldsymbol{\Sigma} =  \D=\boldsymbol{I}$. First, we derive a closed form expression for the population risk in Lemma \ref{lemma:closed_from_of_population_risk}. Then, in Theorem \ref{thm:tightlb_lowrate} we give a lower bound on the population risk for $r \leq 1$ and provide a complete characterization of the autoencoder parameters $(\A,\B)$ achieving it. Surprisingly, the minimizer exhibits a \emph{weight-tying} structure and the corresponding matrices are \emph{rotationally invariant}. Later, in Proposition \ref{prop:highratelb} we derive an analogous lower bound for $r > 1$. While it is hard to characterize the minimizer structure explicitly for a finite input dimension $d$ (and $r>1$), we provide a sequence $\{(\A_d,\B_d)\}_{d\in \mathbb N}$ that meets the lower bound in the high-dimensional limit ($d \rightarrow \infty$), see Proposition \ref{prop:highrate_min}. Notably, the elements of this sequence share the key features (weight-tying, rotational invariance) of the minimizers for $r \leq 1$. 
In Section \ref{subsection:opt_results} we describe gradient methods that provably achieve the optimal value of the population risk. Specifically, we consider gradient flow under a weight-tying constraint and projected (on the sphere) gradient descent with Gaussian initialization. The corresponding results are stated in Theorem \ref{thm:wt_gf} and Theorem \ref{thm:gd_min}.

We start by expanding $\sigma$ in a Hermite series and applying the reproducing property of Hermite polynomials to obtain 
a closed-form expression for the population risk.  This is summarized in the following lemma, which is proved in  Appendix \ref{app:cf}.

\begin{lemma}\label{lemma:closed_from_of_population_risk} Consider any odd $\sigma \in L^2(\mathbb{R},\mu)$ and its Hermite expansion given by \eqref{eq:hermite_expansion}. Then, $\widetilde{\mathcal{R}}(r)$ is equal to
\begin{equation}  \label{eq:poprisklemma}  
\min_{\A,\|\B_{i, :}\|_2=1}\frac{1}{d}\left(\tr{\A^\top \A f(\B\B^\top)} - 2c_1\cdot\tr{\B\A}\right) + 1,
\end{equation}
where $f(x):= \sum_{\ell=0}^\infty (c_{2\ell + 1})^2 x^{2\ell + 1}$ is applied element-wise. In particular, if $\sigma(x) = \mathrm{sign}(x)$, then $f(x) = c_1^2 \cdot \arcsin(x)$ and $c_1 = \sqrt{2/\pi}$. Moreover, for any homogeneous $\sigma$, we have that $\widehat{\mathcal{R}}(r) = \widetilde{\mathcal{R}}(r)$.
\end{lemma}

Note that, if $c_1=0$, it is easy to see that the minimum of $\widetilde{R}(r)$ equals $1$ and it is attained when $\A^\top\A$ is the zero-matrix. Furthermore, if $\sum_{\ell=1}^\infty (c_{2\ell + 1})^2 = 0$, then $\sigma(x) = c_1^2 x$ and we fall back into the simpler case of a linear autoencoder \cite{baldi1989neural,kunin2019loss,gidel2019implicit}. Thus, for the rest of the section, we will assume that $c_1\neq 0$ and $\sum_{\ell=1}^\infty (c_{2\ell + 1})^2 \neq 0$.

\subsection{Fundamental Limits: Lower Bound on Risk}\label{subsec:lb}

We begin by providing a tight lower bound for $r\leq 1$, which is \emph{uniquely} achieved on the set of \emph{weight-tied} orthogonal matrices $\mathcal{H}_{n,d}$ defined as
\begin{equation}\label{haar:min}
   \mathcal{H}_{n,d}:= \left\{\widetilde{\A}, \widetilde{\B}^\top\hspace{-.5em} \in \mathbb{R}^{d\times n}: \widetilde{\A} = \frac{c_1}{f(1)} \cdot \widetilde{\B}^\top,   \widetilde{\B}\widetilde{\B}^\top = \boldsymbol{I}\right\}.
\end{equation}

\begin{theorem}\label{thm:tightlb_lowrate}
    Consider any odd $\sigma \in L^2(\mathbb{R}, \mu)$ and fix $r \leq 1$. Then, the following holds
    $$
    \widetilde{\mathcal{R}}(r) \geq \mathrm{LB}_{r\leq 1}(\I) := 1 - \frac{c_1^2}{f(1)} \cdot r,
    $$
and equality is achieved iff $(\A,\B) \in \mathcal{H}_{n,d}$.
\end{theorem}

We note that the minimizers $\mathcal{H}_{n,d}$ of $\widetilde{\mathcal{R}}(r)$ do not directly correspond to the minimizers of the unconstrained population risk $\widehat{\mathcal{R}}(r)$, since in general $\widetilde{\mathcal{R}}(r) \neq \widehat{\mathcal{R}}(r)$. However, 
if $\sigma$ is homogeneous, the ``inverse'' mapping can be readily obtained. For instance, when $\sigma(x) = \mathrm{sign}(x)$, rescaling the norms of the rows of $\B$ does not affect the compression, i.e.,
$\mathrm{sign}(\B\x) = \mathrm{sign}(\S\B\x)$
for any diagonal $\boldsymbol{S}$ with positive entries. Hence, to obtain a minimizer, it suffices that the rows of $\B$ form any set of orthogonal (not necessarily normalized) vectors. In contrast, note that $\A$ is still defined with respect to the row-normalized version of $\B$. Similar arguments hold for homogeneous activations. 

We now provide a proof sketch for Theorem \ref{thm:tightlb_lowrate} and defer the full argument to Appendix \ref{appendix:lowrate}.

\textit{Proof sketch of Theorem \ref{thm:tightlb_lowrate}.} Using the series expansion of $f(\cdot)$, we can write
\begin{equation}\label{eq:f_taylor}
    \begin{split}
        &\tr{\A^\top \A f(\B\B^\top)} - 2c_1\cdot\tr{\B\A}= \sum_{\ell = 0}^\infty c_{2\ell +1}^2 \Bigg(\tr{\A^\top \A \left(\B\B^\top\right)^{\circ 2\ell+1}} 
        - 2\frac{c_1}{f(1)}\tr{\B\A} \Bigg).
    \end{split}
\end{equation}
Thus,
the minimization problem in Lemma \ref{lemma:closed_from_of_population_risk} can be reduced to analysing each Hadamard power individually:
\begin{equation}\label{eq:mmse_reduced}
    \min_{\A,\|\B_{i, :}\|_2=1}\tr{\A^\top\A (\B\B^\top)^{\circ \ell}} - \frac{2c_1}{f(1)} \cdot \tr{\B\A}.
\end{equation}
The crux of the argument is to provide a suitable sequence of relaxations of \eqref{eq:mmse_reduced}. The first relaxation gives
\begin{equation}\label{eq:mmse_reduced2}
\tr{(\A^\top\A \circ \boldsymbol{Q}) (\B\B^\top \circ \boldsymbol{Q})} - \frac{2c_1}{f(1)} \cdot \tr{\B\A},
\end{equation}
where $\Q$ is \emph{any} PSD matrix with unit diagonal. Using the properties of the SVD of $\Q$, \eqref{eq:mmse_reduced2} can be further relaxed to 
\begin{equation}\label{eq:mmse_reduced3}
\sum_{i,j=1}^n\tr{\A_j\A_j^\top \B_j\B_j^\top} - \frac{2c_1}{f(1)} \cdot \sum_{i=1}^n\tr{\B_i\A_i},
\end{equation}
where now $\A_i, \B^\top_i \in \mathbb{R}^{d\times n}$ are arbitrary matrices.
The key observation is that 
$$
\sum_{i=1}^n\left\|\frac{c_1}{f(1)} \cdot \sqrt{\boldsymbol{X}}^{-1}\A_i^{\top} - \sqrt{\X}\B_i\right\|_F^2 = \eqref{eq:mmse_reduced3} + \frac{c_1^2}{(f(1))^2}\cdot n,
$$
with $\X = \sum_{i=1}^n \A^\top_i\A_i$. As each relaxation lower bounds \eqref{eq:mmse_reduced} and the Frobenius norm is positive, this argument leads to the lower bound on $\widetilde{R}(r)$. The fact that the lower bound is met for any $(\A,\B)\in \mathcal{H}_{n,d}$ can be verified via a direct calculation.
The uniqueness follows by taking the intersection of the minimizers of \eqref{eq:mmse_reduced} for different values of $\ell$. \qed

Next, we move to the case $r > 1$. 
\begin{proposition}\label{prop:highratelb}
      Consider any odd $\sigma \in L^2(\mathbb{R}, \mu)$ and fix $r>1$, then the following holds:
      $$
    \widetilde{R}(r) \geq \mathrm{LB}_{r>1}(\I) := 1 - \frac{r}{r + \left(\frac{f(1)}{c^2_1} - 1\right)}.
      $$
\end{proposition}

The key difference with the proof of the lower bound in Theorem \ref{thm:tightlb_lowrate} is that the term
$
\tr{\A^\top\A \B\B^\top}$
requires a tighter estimate. This is due to the fact that the matrix $\B\B^\top$
is no longer full-rank when $r>1$. We obtain the desired tighter bound by exploiting the following result by \cite{khare2021sharp}: 
\begin{equation}\label{eq:khare_general}
    \A^\top\A \circ \B\B^\top \succeq \frac{1}{d} \cdot \mathrm{Diag}(\B\A)\mathrm{Diag}(\B\A)^\top,
\end{equation}
where $\mathrm{Diag}(\B\A)$ stands for the vector containing the diagonal entries of $\B\A$. The full argument is contained in Appendix \ref{app:pflbr1}.

As for $r\le 1$, the bound is met (here, in the limit $d\to\infty$) by considering weight-tied matrices:
\begin{equation}\label{eq:defoptB}
    \hat{\B}^{\top} = \sqrt{r} \cdot [\I_d, \mathbf{0}_{d,n-d}]  \U^{\top}, \ \b_i = \frac{\hat{\b}_i}{\|\hat{\b}_i\|_2},  \ \A = \beta  \B^\top,
\end{equation}
where $\beta = \frac{r}{r+\left({f(1)}/{c_1^2} - 1\right)}$ and $\U$ is uniformly sampled from the group of rotation matrices. 
The idea behind the choice \eqref{eq:defoptB} is that, as $d\to\infty$, $(\B\B^\top)^{\circ 2\ell}$ for $\ell \geq 2$ is close to the identity matrix, and \eqref{eq:khare_general} is attained exactly.
The formal statement is provided below, and it is proved in Appendix \ref{appendix:highrate}.
\begin{proposition}\label{prop:highrate_min}
Consider any odd $\sigma \in L^2(\mathbb{R}, \mu)$ and fix $r>1$. Let $\A, \B $ be defined as in \eqref{eq:defoptB}. Then, for any $\epsilon >0$ the following holds
$$
\left|\mathcal{R}(\A,\B) - \mathrm{LB}_{r>1}(\I)\right|  \le C d^{-\frac{1}{2}+\epsilon},
$$
with probability $1-c/d^2$.
Here, the constants $c, C$ depend only on $r$ and $\epsilon$.
\end{proposition}

\paragraph{Degenerate isotropic Gaussian data.} All the arguments of this part directly apply for $\bx \sim \mathcal{N}(\0, \sigma^2 \I)$, the only differences being the scaling of the term $\tr{\B\A}$ (which is additionally multiplied by $\sigma$) and the constant variance term $\sigma^2$ (in place of $1$) in \eqref{eq:poprisklemma}.
Our results can be also easily extended to the case of degenerate isotropic Gaussian data, i.e., $\bx\sim \mathcal{N}(\0, \boldsymbol{\Sigma})$ with $\lambda_i(\bSigma) = \sigma^2$ for $i \leq d - k$ and $\lambda_i(\bSigma) = 0$ for $i > d - k$, where $\lambda_i(\bSigma)$ stands for the $i$-th eigenvalue of $\bSigma$ in non-increasing order. In fact, by the rotational invariance of the Gaussian distribution, we can assume without loss of generality that $\x = [x_1,\cdots,x_{d-k},0,\cdots,0]$, where $(x_i)\sim_{\rm i.i.d.} \mathcal{N}(0,\sigma^2)$. Hence, by considering $\A \in \mathbb{R}^{(d-k)\times n}$ and $\B \in \mathbb{R}^{n\times(d-k)}$ and substituting $d$ with $d-k$ where suitable, analogous results follow.

\subsection{Gradient Methods Achieve the Lower Bound}\label{subsection:opt_results}

In this section, we discuss the achievability of the lower bound obtained in the previous section via gradient methods. We study two procedures which find the minimizer of the population risk $\mathcal{R}(\A,\B)$ under the constraint $\|\B_{i,:}\|_2=1$. Namely, we analyse \emph{(i)} \emph{weight-tied gradient flow} on the sphere and \emph{(ii)} its discrete version (with finite step size) \emph{without} weight-tying, i.e., \emph{projected gradient descent}.

The optimization objective in Lemma \ref{lemma:closed_from_of_population_risk} is equivalent (up to a scaling independent of $(\A,\B)$) to
\begin{equation}\label{eq:reduce_opres} 
    \min_{\A, \|\B_{i,:}\|_2=1}\tr{\A^\top\A \cdot f(\B\B^\top)} - 2\cdot \tr{\B\A},
\end{equation}
where we have rescaled the function $f$ by $1/c_1^2$. This follows from the fact that the multiplicative factor $c_1$ can be pushed inside $\A$. Note that such scaling does not affect the properties of gradient-based algorithms (modulo a constant change in their speed). Hence, without loss of generality, we will state and prove all our results for the problem \eqref{eq:reduce_opres}.

\paragraph{Weight-tied gradient flow.} We start with the weight-tied setting, in which 
\begin{equation}\label{eq:weight_tying}
    \A = \beta \B^\top , \quad \beta\in\mathbb R.
\end{equation}
This is motivated by the fact that the lower bounds on the population risk are approached by weight-tied matrices (see Theorem \ref{thm:tightlb_lowrate} and Proposition \ref{prop:highrate_min}).
Under the weight-tying constraint \eqref{eq:weight_tying}, the objective \eqref{eq:reduce_opres} has the following form
\begin{equation}\label{eq:wt_risk_main}
\begin{split}
    \Psi(\beta,\B)&:= 
    \beta^2 \cdot \tr{\B^\top\B \cdot f(\B\B^\top)} - 2\beta n\\
    &= \beta^2 \cdot \sum_{i,j=1}^n {\langle \b_i, \b_j \rangle}  \cdot f\left({\langle \b_i, \b_j \rangle}\right) - 2 \beta n,
\end{split}
\end{equation}
where $\|\b_i\|_2=1$ for all $i$.
Note that the optimal $\beta^{*}$ can be found exactly, since \eqref{eq:wt_risk_main} is a quadratic polynomial in $\beta$. In this view, to optimize \eqref{eq:wt_risk_main}, we perform a gradient flow on $\{\b_i\}_{i=1}^n$, which are regarded as vectors on the unit sphere, and pick the optimal $\beta^{*}$ at each time $t$. 
Formally,
\begin{equation}\label{eq:wt_gf_dynamics}
\begin{split}
    &\beta(t) = \frac{n}{\sum_{i,j=1}^n {\langle \b_i, \b_j \rangle}  \cdot f\left({\langle \b_i, \b_j \rangle}\right)}, \\
    & \frac{\partial \b_i(t)}{\partial t} = -\J_i(t) \nabla_{\b_i} \Psi(\beta(t), \B(t)),
\end{split}
\end{equation}
where $\boldsymbol{J}_i(t) := \boldsymbol{I} - \b_i(t) \b_i(t)^\top$ projects 
the gradient $\nabla_{\b_i} \Psi(\beta(t), \B(t))$ onto the tangent space at the point $\b_i(t)$ (see \eqref{appendix_c:gflow} in Appendix \ref{appendix:wt_glow} for the closed form expression). This ensures that $\|\b_i(t)\|_2 = 1$ along the gradient flow trajectory. The described procedure can be viewed as Riemannian gradient flow, due to the projection of the gradient $\nabla_{\b_i} \Psi(\beta(t), \B(t))$ on the tangent space of the unit sphere.

\begin{theorem}\label{thm:wt_gf} Fix $r\leq 1$. Let $\B(t)$ be obtained via the gradient flow \eqref{eq:wt_gf_dynamics} applied to $\Psi$ defined in  \eqref{eq:wt_risk_main}. Let the initialization $\B(0)$ have unit-norm rows and $\mathrm{rank}(\B(0)) = n$. Then, as $t\to\infty$, $\B(t)\B(t)^\top$ converges to $\I$, which is the unique global optimum of \eqref{eq:wt_risk_main}. Moreover, define the residual
\begin{equation}\label{eq:}
    \phi(t) = \tr{(\B(t)\B(t)^\top - \I) \cdot f(\B(t)^\top\B(t))} \geq 0,
\end{equation}
which vanishes at the minimizer, and let $T$ be the first time such that $\phi(T) = \delta$. Then, 
\begin{equation}\label{eq:quantconv}
\begin{split}
     T \leq -\ind\{\phi(0)>nf(1)\} &\cdot f(1) \cdot {\log \det(\B(0)\B(0)^{\top})}
 - \ind\{\delta \leq nf(1)\} \cdot \frac{2f^2(1)}{\delta} \cdot {\log \det(\B(0)\B(0)^{\top})}.
\end{split}
\end{equation}
\end{theorem}
In words, if the residual at initialization is bigger than $nf(1)$, then it takes at most constant time to reach the regime in which the convergence is linear in the precision $\delta$. 
We also note that by choosing the optimal $\beta^*$, the function $\phi$ can be related to the objective \eqref{eq:wt_risk_main} by
$\Psi(\beta^*, \B(t)) = - \frac{n}{f(1) + \frac{\phi(t)}{n}}$. Hence, \eqref{eq:quantconv} gives a quantitative convergence in terms of the objective function as well. 
We give a sketch of the argument below and defer the complete proof to Appendix \ref{appendix:wt_glow}. 

\textit{Proof sketch of Theorem \ref{thm:wt_gf}.} It can be readily shown that $\B\B^\top = \I$ is a minimizer of \eqref{eq:wt_risk_main} and a stationary point of the gradient flow \eqref{eq:wt_gf_dynamics}. However, if the gradient flow \eqref{eq:wt_gf_dynamics} ends up in points for which $\mathrm{rank}(\B) < n$, such subspaces are never escaped (see Lemma \ref{wt_stationary}) and the procedure fails to converge to the full-rank global minimizer. Thus, our strategy is to show that, if at initialization $\mathrm{rank}(\B)=n$, the gradient flow will never collapse to $\mathrm{rank}(\B)<n$. To do so, the key intuition is to track the quantity $\log\det{(\B(t)\B(t)^\top)}$ during training. In particular, we show in  
Lemma \ref{logdet_lb} 
that
\begin{equation}\label{eq:logdet_main}
    \frac{\partial \log\det{(\B(t)\B(t)^\top)}}{\partial t} \geq \phi(t) \geq 0.
\end{equation}
The inequality \eqref{eq:logdet_main} implies that the determinant is non-decreasing and,  hence, the smallest eigenvalue of 
$\B(t)\B(t)^\top$ is bounded away from $0$ (uniformly in $t$), which gives the desired full-rank property. The convergence speed also follows from \eqref{eq:logdet_main} by a careful integration in time (see Lemma \ref{wt_speed}).
\qed

We remark that Theorem \ref{thm:wt_gf} holds for any $d$ and for all full-rank initializations. 

\paragraph{Projected gradient descent.} We now move to the setting where the encoder and decoder weights are not weight-tied. In this case, we consider the commonly used Gaussian initialization and prove a result for sufficiently large $d$. The Gaussian initialization allows us to relax the requirement on $f$: we only need $c_2=0$, as opposed to the previous assumption that $c_{2\ell} = 0$ for any $\ell\in\mathbb N$ (see the statement of Lemma \ref{lemma:closed_from_of_population_risk}). Specifically, we consider the following algorithm to minimize \eqref{eq:reduce_opres}:
\begin{equation}\label{eq:GDminnew}
    \begin{split}
\A(t) &= \B(t)^\top \left(f(\B(t)\B(t)^{\top})\right)^{-1}\\
\B'(t) &:= \B(t) - \eta \nabla_{\B(t)}, \quad \B(t+1) := \mathrm{proj}(\B'(t)),        
    \end{split}
\end{equation}
where $\A(t)$ is the optimal matrix for a fixed $\B(t)$ and $\nabla_{\B(t)}$ (see \eqref{eq:defnablaBt} in Appendix \ref{appendix:no_wtgauss}) is the projected gradient of the objective \eqref{eq:reduce_opres} with respect to $\B(t)$. Furthermore, $\mathrm{proj}(\B'(t))$ rescales
all the rows to have unit norm. It will become apparent from the proof of Theorem \ref{thm:gd_min} that the inversion in the definition of $\A(t)$ is indeed well defined.
We remark that 
\eqref{eq:GDminnew} can be viewed as the discrete counterpart of the  Riemannian gradient flow \eqref{eq:wt_gf_dynamics} (with the optimal $\A(t)$ in place of the weight-tying), where the application of $\mathrm{proj}(\cdot)$ keeps the rows of $\B(t)$ of unit norm. In the related literature, this procedure 
is often referred to as Riemannian gradient descent (see, e.g., \cite{absil2009optimization}). Alternatively, \eqref{eq:GDminnew} may be viewed as coordinate descent \cite{wright2015coordinate} on the objective \eqref{eq:reduce_opres}, where the step in $\A$ is performed exactly.

Our main result is that 
the projected gradient descent \eqref{eq:GDminnew} converges to the global optimum of \eqref{eq:reduce_opres} for large enough $d$ (with high probability). We give a sketch of the argument and defer the complete proof to Appendix \ref{appendix:no_wtgauss}.

\begin{theorem}\label{thm:gd_min}Consider the projected gradient descent \eqref{eq:GDminnew} applied to the objective \eqref{eq:reduce_opres} for any $f$ of the form $f(x) = x + \sum_{\ell=3} c_{\ell}^2 x^{\ell}$, where $\sum_{\ell=3} c_{\ell}^2 < \infty$. Initialize the algorithm with $\B(0)$ equal to a row-normalized Gaussian, i.e., $\B'_{i,j}(0) \sim \mathcal{N}(0,1/d)$, $\B(0) = \mathrm{proj}(\B'(0))$. Let the step size $\eta$ be $\Theta(1/\sqrt{d})$.
    Then, for any $r<1$ and sufficiently large $d$, with probability at least $1-Ce^{-cd}$, we have that
 $\B(t)\B(t)^\top$ converges to $\I$, which is the unique global optimum  of \eqref{eq:reduce_opres}. 
 Moreover, defining $t = T/\eta$, we have the following bound on the rate of convergence
    $$\opn{\B(t)\B(t)^\top-\I} \leq C(1-c)^T,$$
    where $C>0$ and $c\in (0, 1]$ are universal constants depending only on $r$ and $f$.

\end{theorem}

\textit{Proof sketch of Theorem \ref{thm:gd_min}.} Let $\B(0)\B(0)^\top = \U\Lam(0) \U^\top$ be the singular value decomposition (SVD) of the encoder Gram matrix. Then, the idea is to decompose $\B(t)\B(t)^\top$ 
at each step of the projected gradient descent dynamics as
\begin{equation}\label{eq:decBB_main}
\B(t)\B(t)^\top = \I + \Z(t) + \X(t),     
\end{equation}
where $\Z(t) = \U(\Lam(t)-\I) \U^\top$. Here, $\I$ is the global optimum towards which we want to converge; $\Z(t)$ captures the evolution of the eigenvalues while keeping the eigenbasis fixed, as $\U$ comes from the SVD at initialization; and $\X(t)$ is the remaining error term capturing the change in the eigenbasis.
The update on $\Lam(t)$ is given by $\Lam(t+1) = g(\Lam(t))$, where $g:\mathbb{R}^n\rightarrow\mathbb{R}^n$ admits an explicit expression. Hence, in light of this explicit expression, if we had $\X(t) \equiv 0$, then the desired convergence would follow from the analysis of the recursion for $\Lam(t)$ (see Lemma \ref{lemma:spectrum_updates}). 

The main technical difficulty lies in carefully controlling the error term $\X(t)$. In particular, we will show that $\X(t)$ decays for large enough $d$, which means that dynamics \eqref{eq:decBB_main} is well approximated by  
$\I + \Z(t)$. The proof can be broken down in four steps. In the \emph{first step}, we compute the leading order term of $\nabla_{\B(t)}$ (see Lemma \ref{lemma:finv_bound} and \ref{lemma:AAT_schur}). This simplifies the formula for $\nabla_{\B(t)}$, which can then be expressed as an explicit nonlinear function of $\Z(t)$ and $\X(t)$. In the \emph{second step}, we perform a Taylor expansion of $\nabla_{\B(t)}$, seen as a matrix-valued function in $\Z(t)$ and $\X(t)$ (see Lemma \ref{lemma:gradient_approx}). The intuition for this expansion comes from the fact that $\X(t)$ is a small quantity, and also $\opn{\Z(t)} \to 0$ as $t \to \infty$. In the \emph{third step}, we show that the norm of $\nabla_{\B(t)}$ vanishes sufficiently fast (see Lemma \ref{lemma:nablanablaT}), which implies that the projection step $\B(t+1) := \mathrm{proj}(\B'(t))$ has a negligible effect (see Lemma \ref{proj_Z}). After doing these estimates, we finally obtain an explicit recursion for $\X(t)$. In the \emph{fourth step}, we analyse this recursion (see Lemma \ref{lemma:xtexpconv}): first, we show that the error does not amplify too strongly (as in Gronwall's inequality); then, armed with this worst-case estimate, we can prove an exponential decay for $\X(t)$, which suffices to conclude the argument.
\qed

\paragraph{Scaling of the learning rate.} 
Theorem \ref{thm:gd_min} is stated for $\eta = \Theta(1/\sqrt{d})$, as this corresponds to the biggest learning rate for which our argument works (thus requiring the least amount of steps for convergence). The same result can be proved for $\eta =\Theta(d^{-\kappa})$ with $\kappa \geq 1/2$. The only piece of the proof affected by this change is the third part of Lemma \ref{lem_D_diagonal} (in particular, the chain of inequalities \eqref{eq:neweq}), which continues to hold as long as $\eta$ is polynomial in $d^{-1}$.

\paragraph{Assumptions on compression rate $r$.} We expect an analog of Theorem \ref{thm:wt_gf} to hold for $r>1$, as long as $d$ is \emph{sufficiently large}. In fact, for a fixed $d$, it appears to be difficult to even characterize the global minimizer: the choice \eqref{eq:defoptB} approaches the lower bound $\mathrm{LB}_{r>1}(\I)$ only as $d\to\infty$, see Proposition \ref{prop:highrate_min}. 
We also expect Theorem \ref{thm:gd_min} to hold for $r\ge 1$. Here, an additional challenge is that the minimizer has non-zero off-diagonal entries. In combination with the lack of an exact characterization of the minimizer, this leads to an additional error term that would be difficult to control with the current tools.
At the same time, the restriction $r<1$ is likely to be an artifact of the proof as 
experimentally (see, for instance, Figure \ref{fig:comparison}) the algorithm still converges to the global optimum for $r\ge 1$.

\paragraph{Gaussian initialization in Theorem \ref{thm:gd_min}.} The Gaussian initialization ensures that, with high probability, the off-diagonal entries of $\B(t)\B(t)^\top$ are small. This allows us to approximate higher-order Hadamard powers of $\B(t)\B(t)^\top$ with $\I$. However, in experiments the Gaussian assumption seems to be unnecessary, and we expect the convergence result to hold for all (non-degenerate) initializations.

\section{Extension to General Covariance}\label{section:general_cov}
In this section, we consider a Gaussian source with general covariance structure, i.e., $\bSigma = \U \D^2 \U^\top$. Without loss of generality, the matrix $\D$ can be written as
\begin{equation}\label{eq:defDblock}
    \D = \mathrm{Diag}([\underbrace{D_1, \cdots, D_1}_{\times k_1}|\cdots|\underbrace{D_K, \cdots, D_K}_{\times k_K}]),
\end{equation}

where $\sum_{i=1}^K k_i = d$, $k_i \geq 1$ and $D_i > D_{i+1} \geq 0$. 
We start by deriving a closed-form expression for the population risk, similar to that of Lemma \ref{lemma:closed_from_of_population_risk}. Its proof is given in Appendix \ref{app:cf}. 

\begin{lemma}\label{lemma:popriskD} Let $\sigma \in L^2(\mathbb{R},\mu)$ be an odd homogeneous activation, then $\widetilde{\mathcal{R}}(r)$ is equal to the  minimum of 
\begin{equation}\label{eq:DPR_obj}
    \frac{1}{d}\left(\tr{\A^\top \A f(\B\B^\top)} - 2c_1\cdot\tr{\B\D\A} + \tr{\D^2 }\right)
\end{equation}

under the constraint $\|\B_{i,:}\|_2=1$.
Moreover, $\widehat{\mathcal{R}}(r) = \widetilde{\mathcal{R}}(r)$.
\end{lemma}

The result of Lemma \ref{lemma:popriskD} can be extended to any odd $\sigma\in L^2(\mathbb{R},\mu)$ 
at the cost of losing the equivalence between the objectives $\widehat{\mathcal{R}}(r)$ and $\widetilde{\mathcal{R}}(r)$.

We restrict the theoretical analysis to proving a lower bound on \eqref{eq:DPR_obj} in the weight-tied setting \eqref{eq:weight_tying}. This lower bound can be achieved via a careful choice of the matrices $\bA, \bB$ (see Proposition \ref{prop:Dlb_achievability}), and we provide numerical evidence (see Figure \ref{fig:noniso_exps}) that gradient descent saturates the bound \emph{without} the weight-tying constraint.
Thus, we expect our lower bound to hold also for general (not necessarily weight-tied) matrices.

The lower bound is given by the minimum
\begin{equation}\label{eq:popriskDLB_cv} 
\frac{1}{d}\left(\frac{g(1)}{c_1^2n} \left(\sum_{i=1}^{K} \beta_i \right)^2 + \sum_{i=1}^K \left(c_1^2\frac{\beta_i^2}{s_i} - 2 c_1  D_i \beta_i+ D_i^2\right)\right)
\end{equation}
over all $\beta_i \geq0$ and
\begin{equation}\label{eq:rank_constraints}
\begin{cases}
    0 \leq s_i \leq \min\{k_i, n\},\\
    1 \leq \sum_{i=1}^K s_i \leq \min\{d,n\}.
\end{cases}
\end{equation}
Here $g(x) = f(x)- c_1^2x$, and we use the convention that $\frac{0^2}{0}= 0$ and $\frac{c}{0} = +\infty $ for $c > 0$.
We can also explicitly characterize the optimal $s_i, \beta_i$. 
The optimal $s_i$ are obtained via a \emph{water-filling criterion}: 
\begin{equation}\label{eq:water_filling_main}
    \begin{cases}
    \boldsymbol{s} = [n,0,\cdots,0], & n \leq k_1,\\
    \boldsymbol{s} = [k_1,k_2,\cdots,k_K], & d \leq n,\\
    \boldsymbol{s} = [k_1,\cdots,k_{\mathrm{id}(n) - 1},\mathrm{res}(n),0,\cdots,0] & \text{otherwise},
    \end{cases}
\end{equation}
where $\boldsymbol{s} = [s_1,\cdots,s_k]$, $\mathrm{id}(n)$ denotes the first position at which
$
\min\{n,d\} - \sum_{i=1}^{\mathrm{id}(n)} k_i < 0,
$
and the residual is defined by
$
\mathrm{res}(n) := \min\{n, d\} - \sum_{i=1}^{\mathrm{id}(n) - 1} k_i.
$
The $\beta_i$ can also be expressed explicitly in terms of $f, s_i, D_i$.
This is summarized in the following theorem.

\begin{theorem}\label{thm:wtD_lb}
    Consider the objective \eqref{eq:DPR_obj} under the weight-tied constraint \eqref{eq:weight_tying}.
    Then, 
    \begin{equation}\label{eq:LBD_def}
    \eqref{eq:DPR_obj} \geq  \mathrm{LB}(\D):= \min_{s_i, \beta_i} \eqref{eq:popriskDLB_cv},
    \end{equation}
    where $\beta_i \geq0$ and the $s_i$ satisfy \eqref{eq:rank_constraints}.
    Furthermore, the minimizers of \eqref{eq:popriskDLB_cv} are the $s_i$ obtained via the \emph{water-filling} criterion \eqref{eq:water_filling_main} and
    \begin{equation}\label{eq:betaimain}
    \beta_i =
        \begin{cases}
         \frac{s_i}{c_1} \cdot \left(\frac{\frac{g(1)}{c_1^2n} \sum_{j=1}^{M^{*}}s_j\Delta_j + D_1 }{\frac{g(1)}{c_1^2n}\sum_{j=1}^{M^{*}}s_j + 1} - \Delta_i\right) & \text{if } i \leq M^*, \\
         0 & \text{otherwise,}
    \end{cases}
    \end{equation}
    where $\Delta_j = D_1 - D_j$ and $M^*$ is smallest index such that
    $$\frac{g(1)}{c_1^2n}\sum_{j=1}^{M^*+1} s_j (D_{M^{*} + 1} - D_j) + D_{M^{*} + 1} \leq 0.$$
    If no such index exists, then $M^* = K$.
\end{theorem}

\begin{figure}[t!]
    \centering
    \subfloat{\includegraphics[width=0.5\columnwidth]{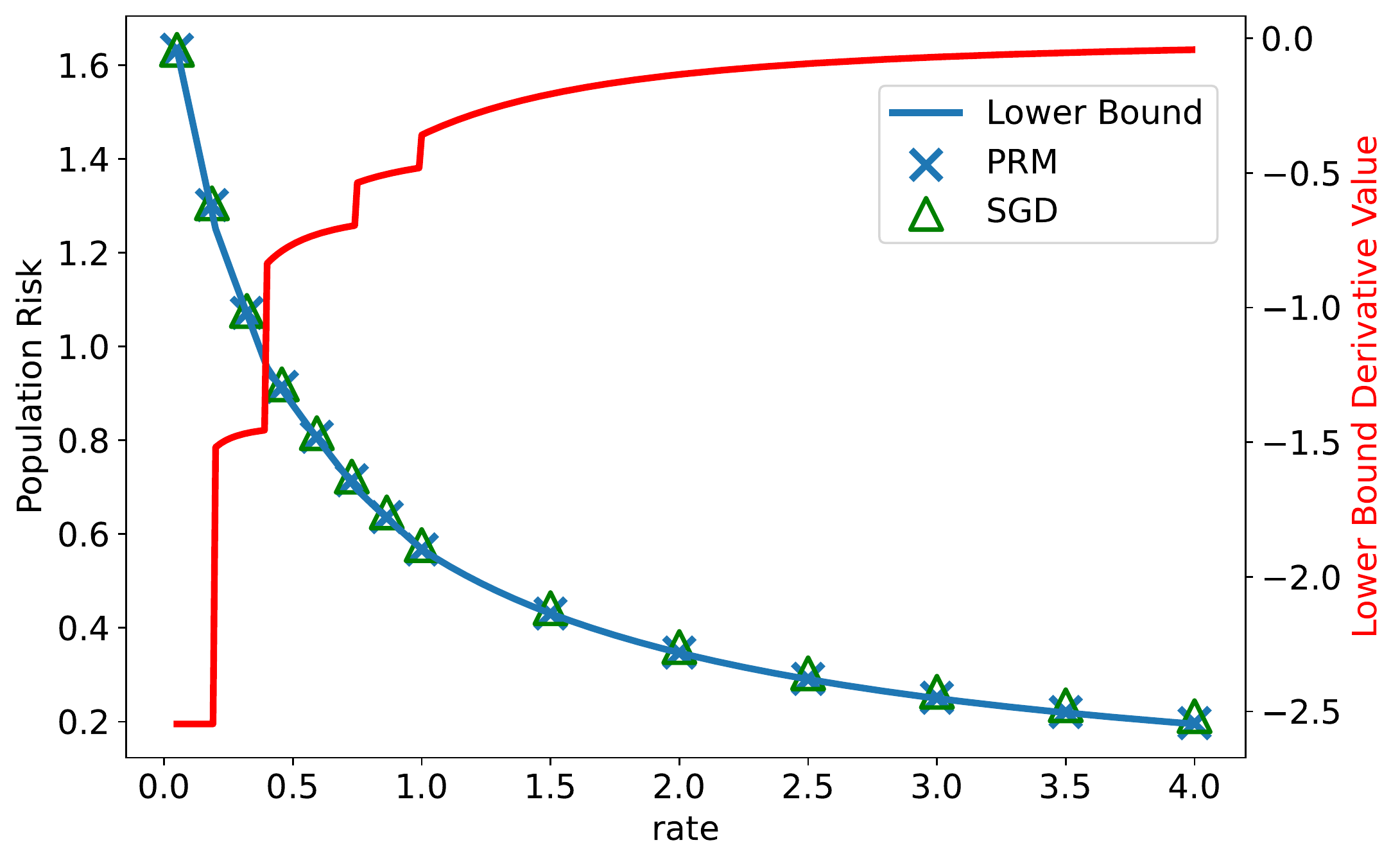}}
    \subfloat{\includegraphics[width=0.5\columnwidth]{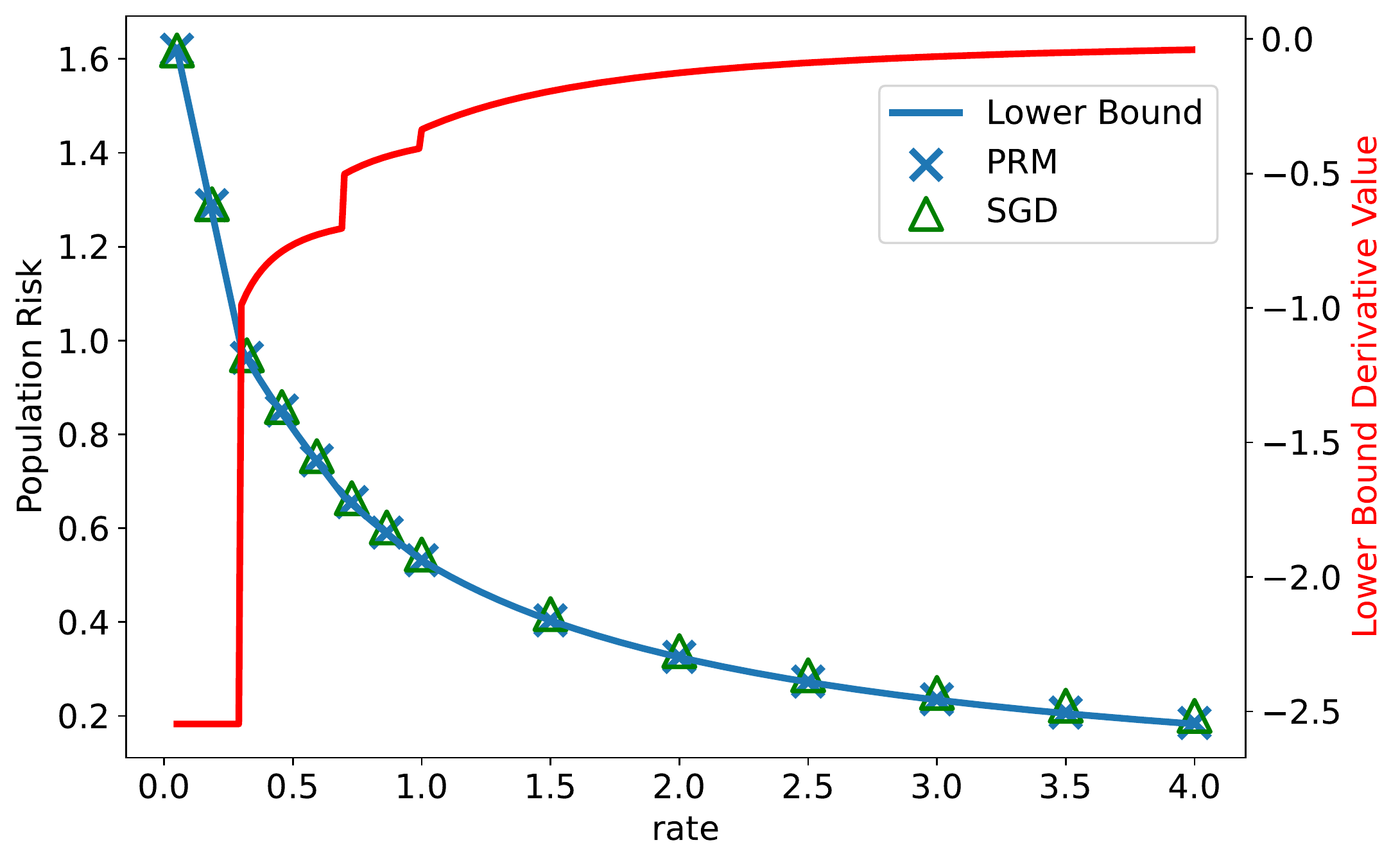}}
\caption{Compression ($\sigma\equiv {\rm sign}$) of a non-isotropic Gaussian source, whose covariance matrix is obtained by taking $\boldsymbol{k}=(20,20,35,25)$ and $(D_1, D_2, D_3, D_4)=(2,1.5,1,0.8)$ for the left plot, and $\boldsymbol{k}=(30, 40, 30)$ and $(D_1, D_2, D_3)=(2, 1, 0.7)$ for the right plot. The blue crosses (Population Risk Minimizer, PRM) are obtained by optimizing \eqref{eq:DPR_obj} via GD. The green triangles are obtained by training an autoencoder via SGD
on Gaussian samples with the given covariance structure. The red solid line plots the derivative of the population risk computed using a finite differences scheme. Note that the derivative jumps when the corresponding blocks are getting filled, although this may not happen in general, see Appendix \ref{appendix:numerics}. A similar behavior can be observed in the isotropic case at $r=1$, as there is only one block to fill (see Figure \ref{fig:comparison}).\vspace{-2mm}}\label{fig:noniso_exps}
\end{figure}

We give a high-level overview of the proof below, and the complete argument is provided in Appendix \ref{appendix:water_fill_proofs}.

\textit{Proof sketch of Theorem \ref{thm:wtD_lb}.}
In the first step, we show that \eqref{eq:LBD_def} 
holds. 
Consider the following block decomposition of $\B$ having the same block structure as $\D$:
\begin{equation}\label{eq:weight_tying_noniso}
    \B = [\bGamma_1\B_1|\cdots|\bGamma_K \B_K],
\end{equation}
where $\B_j \in \mathbb{R}^{n\times k_j}$ with $\|(\B_j)_{i,:}\|_2=1$ and $\{\bGamma_j\}_{j=1}^K$ are diagonal matrices with $\sum_{j=1}^K\bGamma_j^2 = \I$. 
Each $\B_i$ will play a similar role to the $\B$ in the isotropic case.
The crucial bound for this step comes from Theorem {A} in \cite{khare2021sharp}:
$$
(\bGamma_i\B_i\B_i^\top\bGamma_i)^{\circ{2}} \succeq \frac{1}{s_i}  \cdot \mathrm{Diag}(\bGamma_i^2) \mathrm{Diag}(\bGamma_i^2)^\top,
$$
where $s_i = \mathrm{rank}(\B_i\B_i^\top)$.
Now, ignoring the (PSD) cross-terms for $i\neq j$ we can proceed as in the proof of Proposition \ref{prop:highratelb} to arrive at the lower bound
\begin{equation}\label{eq:popriskDLBmain}
   \frac{1}{d}\left( \beta^2 \left(g(1) \cdot n + \sum_{i=1}^K \frac{\gamma_i^2}{s_i}\right) - 2 \beta \cdot \sum_{i=1}^K D_i \gamma_i + \sum_{i=1}^KD_i^2\right),
\end{equation}
where, with an abuse of notation, we have re-defined $g(x):=g(x)/c_1^2$ and $\beta := c_1\beta$.
 Note that for $\D=\I$ one can easily find an expression for the minimum of \eqref{eq:popriskDLBmain} in terms of $r$ and verify that it coincides with the previous bounds in Theorem \ref{thm:tightlb_lowrate} and  Proposition \ref{prop:highratelb}.
 Now by choosing $\beta_i := \beta \gamma_i$ and using that $\sum_{i=1}^K\gamma_i = n$, the objective \eqref{eq:popriskDLBmain} is seen to be equivalent to \eqref{eq:popriskDLB_cv}, hence \eqref{eq:LBD_def} holds. 

 Next, 
the optimal $s_i$ are water-filled as defined in \eqref{eq:water_filling_main}, which follows from the standard convex analysis argument of Lemma \ref{lemma:two_buckets}. 
Finally, given the form of the optimal $s_i$,
it remains to find the optimal $\beta_i$.
This is done by considering a slightly more general problem in Lemma \ref{lem:LBD_KKT_general}. In fact, the problem of minimizing \eqref{eq:popriskDLB_cv} is of the form:
$$
\eqref{eq:popriskDLB_cv} =  \min_{m_i \geq 0 }f\left(\sum_{i=1}^K m_i \right) + \sum_{i=1}^{K} f_i(m_i),
$$
where importantly $f$ and $\{f_i\}_{i=1}^{K}$ are \emph{strictly convex} differentiable functions. 
The proof is based on techniques from convex analysis. The explicit calculations for our case are 
 then carried out in Lemma \ref{lem:LBD_KKT}.
\qed

\paragraph{Asymptotic achievability.} 
We show that the lower bound in Theorem \ref{thm:wtD_lb} can be asymptotically (i.e, as $d\to\infty$) achieved by using the block form \eqref{eq:weight_tying_noniso}, after carefully picking $\B_i$ for each block. Specifically, 
first we generate a matrix $\U \in \R^{n \times n}$ which is sampled uniformly from the group of orthogonal matrices.
Next, we choose each $\B_i$ such that $\hat{\B}_i \hat{\B}_i^\top = \frac{n}{k_i}\U \D_i \U^\top$, where $\D_i$ is a diagonal matrix with
$$(\D_i)_{v,v} = \begin{cases} 1,  \quad \textrm{if} \quad \sum_{j=1}^{i-1}k_j < v \leq \sum_{j=1}^i k_j, \\
0, \quad \mathrm{otherwise},
\end{cases}
$$
and the rows of $\B$ are composed of normalized $\hat{\b}_i$, i.e.,  $\b_i = \frac{\hat{\b}_i}{\|\hat{\b}_i\|_2}$. Furthermore, 
we pick $\bGamma_i^2 = \frac{\gamma_i}{n} \I$ and
 $\A = \beta \B^\top $.
 The scalings $\gamma_i$ and $\beta$ are chosen to be the minimizers of \eqref{eq:popriskDLBmain} for  $s_i$ as in \eqref{eq:rank_constraints}. This is formalized in the following proposition. 
 
 \begin{proposition}\label{prop:Dlb_achievability}
     Assume $\A, \B$ are constructed as described above and fix $r >0$.
     Also assume that, for all $i$, $\frac{k_i}{n}$ converges to a strictly positive number as $d \to \infty$.
     Then, for any $\epsilon >0$ 
with probability $1-\frac{c}{d^2}$, the following holds
$$
\left|\mathcal{R}(\A,\B) - \mathrm{LB}(\D)\right|  \le C d^{-\frac{1}{2}+\epsilon},
$$
where $\mathrm{LB}(\D)$ is defined in \eqref{eq:LBD_def}, and the constants $c, C$ only depend on $r$, $\epsilon$ and $\lim_{d\to \infty}\frac{k_i}{n}$.

 \end{proposition}
The proof of this lemma is similar to that of Proposition \ref{prop:highrate_min}, and it is provided in Appendix \ref{appendix:water_fill_proofs}. We remark that Proposition \ref{prop:Dlb_achievability} can be extended to $D_i$ being sampled from a compactly supported measure, at the price of a worse rate of convergence. This is due to the fact that we can approximate compact measures with discrete measures. We omit the details here. 
 
 Taken together, Proposition \ref{prop:Dlb_achievability} and Theorem \ref{thm:wtD_lb} show that the optimal $\B$ exhibits the block structure \eqref{eq:weight_tying_noniso}, which agrees with the block structure \eqref{eq:defDblock} of the covariance matrix of the data. The individual blocks are orthogonal in the sense that $ \B_i^\top \bGamma_i \bGamma_j\B_j = \0$.
 Furthermore,  each block has the same form as the minimizers in the isotropic case,
 up to some scaling.
 Such a structure is also confirmed by the numerical experiments: for instance, it is observed in the settings considered for Figure \ref{fig:noniso_exps}.

\section{Discussion}\label{sec:discussion}

\paragraph{Population vs. empirical loss.} All our results hold for the optimization of the population loss. Extending them to the empirical loss is an interesting direction for future research. One possible way forward is to exploit recent progress towards relating the landscape of empirical and population losses, see e.g.  \cite{mei2018landscape}. We remark that, in the simulations of gradient descent, we always use the tempered straight-through estimator of the sign activation (see Appendix \ref{appendix:numerics} for details). 
Thus, another promising direction is to show that, in the low-temperature regime (i.e., when the differentiable approximation of the sign becomes almost perfect), the gradient-based scheme converges to the minimizer of the population risk.

\paragraph{Optimality of two-layer autoencoders.} \begin{wrapfigure}{r}{0.48\textwidth}
  \begin{center}
    \vspace{-2.3em}\hspace{-1em}\includegraphics[width=0.46\textwidth]{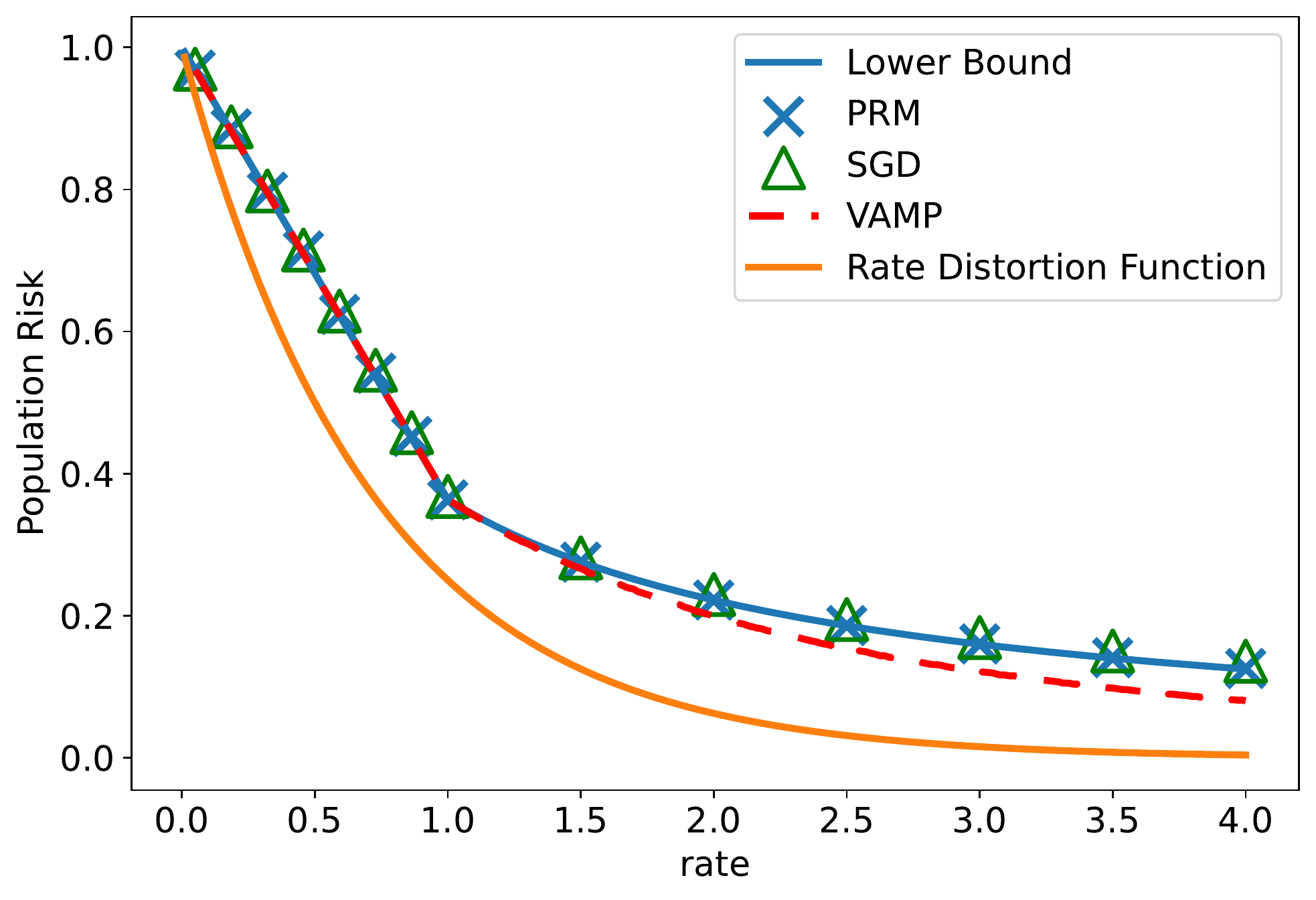}
  \end{center}
 \vspace{-1em}\caption{Performance comparison for the compression ($\sigma\equiv {\rm sign}$) of an isotropic Gaussian source.}\vspace{-1.em}
    \label{fig:comparison}
\end{wrapfigure}This paper characterizes the minimizers of the expected $\ell_2$ error incurred by two-layer autoencoders, and it shows that the minimum error is achieved, under certain conditions, by gradient-based algorithms. Thus, for the special case in which $\sigma\equiv {\rm sign}$, a natural question is to what degree the model \eqref{eq:model} is suitable for data compression.

Let us 
fix the encoder to be a rotationally invariant matrix, i.e., $\B = \U \Lam \V^\top$ with $\U, \V$ independent and distributed according to the Haar measure and $\Lam$ having bounded entries. Then, the information-theoretically optimal reconstruction error can be computed via the replica method from statistical mechanics \cite{tulino2013support} and, in a number of scenarios, it coincides with the error of a Vector Approximate Message Passing (VAMP) algorithm \cite{rangan2019vector,schniter2016vector}. Furthermore, it is also possible to optimize the spectrum $\Lam$ to minimize the error, which leads to the singular values of $\B$ being all $1$ \cite{ma2021analysis}.\footnote{More specifically, \cite{ma2021analysis} consider an expectation propagation (EP) algorithm \cite{minka2001expectation,opper2005expectation,fletcher2016expectation,he2017generalized}, which has been related to  various forms of approximate message passing \cite{ma2017orthogonal,rangan2019vector}.} Surprisingly, for a compression rate $r\le 1$, the optimal error found in \cite{ma2021analysis} \emph{coincides} with the minimizer of the population loss given by Theorem \ref{thm:tightlb_lowrate}. Hence, two-layer autoencoders are optimal compressors under two conditions: \emph{(i)} $r\le 1$, and \emph{(ii)} fixed encoder given by a rotationally invariant matrix. Both conditions are sufficient and also necessary. For $r>1$, VAMP outperforms the two-layer autoencoder. Moreover, for \emph{a general encoder/decoder pair}, the information-theoretically optimal reconstruction error is given by the rate-distortion function, which outperforms two-layer autoencoders for all $r>0$. This picture is summarized in Figure \ref{fig:comparison}: the blue curve represents the lower bound of Theorem \ref{thm:tightlb_lowrate} (for $r\le 1$) and Proposition \ref{prop:highratelb} (for $r>1$), which is met by either running GD on the population risk (blue crosses) or SGD on samples taken from a isotropic Gaussian (green triangles) when $d=100$;\footnote{For further details on the experimental setup, see Appendix \ref{appendix:numerics}.} this lower bound meets the performance of VAMP (red curve) if and only if $r\le 1$; finally, the rate distortion function (orange curve) provides the best performance for all $r>0$.

\paragraph{Universality of Gaussian predictions.} Figures \ref{fig:noniso_exps} and \ref{fig:comparison} show that gradient descent achieves the minimum of the population risk for the compression of Gaussian sources. Going beyond Gaussian inputs, to real-world datasets,  Figure \ref{fig:cifar_white} (as well as those in Appendix \ref{appendix:numerics}) shows  an excellent agreement between our predictions (using the empirical covariance of the data) and the performance of autoencoders trained on standard datasets (CIFAR-10, MNIST). As such, this agreement  provides a clear indication of the universality of our predictions. 
 In this regard, a flurry of recent research (see e.g. \cite{hastie2022surprises,hu2022universality,loureiro2021learning,goldt2022gaussian,dudeja2022spectral, montanari2022universality} and references therein) has proved that the Gaussian predictions actually hold in a much wider range of models.  While none of the existing works exactly fits the setting considered in this paper, this gives yet another indication that our predictions should remain true more generally. 
The rigorous characterization of this universality is left for future work. 

\section*{Acknowledgements}

Alexander Shevchenko, Kevin Kögler and Marco Mondelli are supported by the 2019 Lopez-Loreta Prize.
Hamed Hassani acknowledges the support by the NSF CIF award (1910056) and the NSF Institute for CORE Emerging Methods in Data Science (EnCORE).

\bibliography{refs}
\bibliographystyle{amsalpha}

\newpage
\appendix
\onecolumn
\section{Closed Forms for the Population Risk}\label{app:cf}

\begin{proof}[Proof of Lemma \ref{lemma:closed_from_of_population_risk}] Opening up the two-norm gives
\begin{equation}\label{eq:cfpr1}
    \E\|\x - \A\sigma(\B\x)\|_2^2 = \E\|\x\|_2^2 + \E\|\A\sigma(\B\x)\|_2^2 - 2\E\langle \x, \A\sigma(\B\x) \rangle.
\end{equation}
Since $\x\sim\mathcal{N}(\0,\boldsymbol{I})$, we get 
\begin{equation}\label{eq:rtilde_piece0}
    \E \|\x\|_2^2=d.
\end{equation}
Let $\B^\top = [\b_1,\dots,\b_n]\in \mathbb R^{d\times n}$ and $\A = [\a_1,\dots,\a_n]\in \mathbb R^{d\times n}$, with $\|\b_i\|_2=\|\B_{i, :}\|=1$. Rewriting the second term in  \eqref{eq:cfpr1} gives
\begin{equation}\label{eq:cfpr2}
    \E \|\A\sigma(\B \x)\|_2^2 = \sum_{i,j=1}^n \langle \a_i, \a_j \rangle \cdot \E \left[\sigma(\langle \b_i,\x \rangle) \cdot \sigma(\langle \b_j,\x \rangle)\right].
\end{equation}
Using the reproducing property of Hermite coefficients (see, e.g., Chapter 11 in \cite{o2014analysis}), since the random variables $\langle \b_i,\x \rangle$ and $\langle \b_j,\x \rangle$ are $\langle \b_i,\b_j \rangle$-correlated, we have
$$
\E\left[ h_{2\ell+1}(\langle \b_i,\x \rangle) \cdot h_{2\ell+1}(\langle \b_j,\x \rangle)\right] = \langle \b_i,\b_j \rangle^{2\ell + 1}, \quad \E\left[ h_{2\ell+1}(\langle \b_i,\x \rangle) \cdot h_{2k+1}(\langle \b_j,\x \rangle)\right] = 0,
$$
for $k\neq \ell$. This gives that
$$
\E \left[\sigma(\langle \b_i,\x \rangle) \cdot \sigma(\langle \b_j,\x \rangle)\right] = \sum_{\ell=0}^\infty (c_{2\ell+1})^2 \langle \b_i,\b_j \rangle^{2\ell + 1} = f(\langle \b_i, \b_j \rangle),
$$
and, hence, using \eqref{eq:cfpr2} we arrive to
\begin{equation}\label{eq:rtilde_piece1}
    \E \|\A\sigma(\B \x)\|_2^2 = \sum_{i,j=1}^n \langle \a_i, \a_j \rangle \cdot f(\langle \b_i, \b_j \rangle) = \tr{\A^\top\A\cdot f(\B\B^\top)}.
\end{equation}
Rearranging the last term in \eqref{eq:cfpr1} gives
\begin{equation}\label{eq:cfpr3}
    \E \langle \x , \A\sigma(\B\x)  \rangle = \sum_{i=1}^d\sum_{j=1}^n a_j^i \cdot \E [ x_i  \sigma(\langle \b_j, \x \rangle)],
\end{equation}
where $a_j^i$ stands for the $i$-th coordinate of the vector $\a_j$ and $x_i$ stands for the $i$-th coordinate of the vector $\x$. Let us now compute the inner expected value for each pair $(i,j)$. Notice that the random variables $\langle \b_j , \x \rangle$ and $x_i$ are jointly Gaussian with zero mean and covariance matrix $\widetilde{\boldsymbol{\Sigma}}\in\mathbb R^{2\times 2}$:
$$
\widetilde{\boldsymbol{\Sigma}}_{21} = \widetilde{\boldsymbol{\Sigma}}_{12} = \E x_i\langle \b_j , \x \rangle = \E b_j^i x_i^2 = b_j^i,\quad \ \widetilde{\boldsymbol{\Sigma}}_{11} = \E \langle \b_j , \x \rangle^2 = \|\b_j\|_2^2 = 1, \quad \widetilde{\boldsymbol{\Sigma}}_{22} = \E x_i^2 = 1.
$$
Hence, the random vectors $(\langle \b_j , \x \rangle,x_i)$ and
$$
\left(y_1, b_j^i \cdot y_1 + \sqrt{1-(b_j^i)^2}\cdot y_2\right), \quad \mbox{ with }(y_1, y_2) \sim \mathcal{N}(0, \boldsymbol{I})
$$
are identically distributed. In this view, we obtain
\begin{equation}\label{eq:elaborate_bulbasaur}
\begin{split}
    \E [ x_i \sigma(\langle \b_j, \x \rangle)] &= \E \left[\left(b_j^i \cdot y_1 + \sqrt{1-(b_j^i)^2}\cdot y_2\right)\sigma(y_1)\right]\\ 
    &= b_j^i \cdot \E [y_1 \sigma(y_1)] + \sqrt{1-(b_j^i)^2} \cdot \E [y_2] \cdot \E[\sigma(y_1)] = c_1 \cdot b_j^i,
\end{split}
\end{equation}
where we applied the reproducing property to conclude that $\E [y_1 \sigma(y_1)] = c_1$. Consequently, by combining \eqref{eq:cfpr3} and \eqref{eq:elaborate_bulbasaur}, we get that
\begin{equation}\label{eq:rtilde_piece2}
    \E \langle \x , \A\sigma(\B\x)  \rangle = c_1 \cdot \sum_{i=1}^d\sum_{j=1}^n a_j^i b_j^i = c_1 \cdot \tr{\B\A}. 
\end{equation}
By combining \eqref{eq:cfpr1}, \eqref{eq:rtilde_piece0}, \eqref{eq:rtilde_piece1} and \eqref{eq:rtilde_piece2}, we obtain the desired expression for $\widetilde{R}(r)$.

Assume now that $\sigma$ is homogeneous. Then, in \eqref{eq:cfpr2} and \eqref{eq:cfpr3}, the norm of $\b_i$ can be pushed into the corresponding $\a_i$ and, hence, we obtain
$$
\min_{\A, \B} \E \|\x - \A\sigma(\B\x)\|_2^2 = \min_{\A, \|\B_i\|_2=1} \E \|\x - \A\sigma(\B\x)\|_2^2,
$$
which proves that $
\widehat{\mathcal{R}}(r) = \widetilde{\mathcal{R}}(r)
$.

Finally, consider the case $\sigma(x)= \mathrm{sign}(x)$. Then, Grothendieck’s identity (see, e.g., Lemma 3.6.6 in \cite{vershynin2018high}) gives 
$$
\E \sigma(\langle \b_i,\x \rangle) \sigma(\langle \b_j,\x \rangle) = \frac{2}{\pi} \mathrm{arcsin}(\langle \b_i , \b_j \rangle) \Rightarrow f(x) = \frac{2}{\pi} \mathrm{arcsin}(x).
$$
Recalling that the first Hermite coefficient of $\sigma(x)= \mathrm{sign}(x)$ is equal to $\sqrt{\frac{2}{\pi}}$ finishes the proof.
\end{proof}

\begin{proof}[Proof of Lemma \ref{lemma:popriskD}] The proof of Lemma \ref{lemma:popriskD} follows from similar arguments as that of Lemma \ref{lemma:closed_from_of_population_risk}. Given this, we only explain the key differences. 
We first show that it is enough to consider $\bSigma = \D^2$.
Given the SVD $\bSigma=\U \D^2 \U^\top$, we have $\x = \U \D \tilde{\x}$, where $\tilde{\x} \sim \mathcal{N}(\0, \I)$.
Now, we can push the rotation $\U$ in $\A, \B$:
$$\norm{\x - \A\sigma(\B\x)}_2 = \norm{\D\tilde{\x} - \U^\top\A\sigma(\B\U\D\tilde{\x})}_2.$$
Thus, after replacing $\A$ with $\U^\top \A$ and $\B $ with $\B \U$, we may assume that $\x = \D \tilde{\x}$.

We again open up the two-norm 
\begin{equation}\label{eq:open2normD}
    \E\|\x - \A\sigma(\B\x)\|_2^2 = \E\|\x\|_2^2 + \E\|\A\sigma(\B\x)\|_2^2 - 2\E\langle \x, \A\sigma(\B\x) \rangle.
\end{equation}
For the first term, we clearly have
$$ \E\|\x\|_2^2 = \tr{\D^2 }.$$
Now, for the second term we write
$$\E\|\A\sigma(\B\x)\|_2^2 =\E\|\A\sigma(\B \D \tilde{\x})\|_2^2,$$
where $\tilde{\x} \sim \mathcal{N}(\0,\boldsymbol{\I})$ .
Thus, as in the proof of Lemma \ref{lemma:closed_from_of_population_risk}, we have
$$\E\|\A\sigma(\B \D \tilde{\x})\|_2^2 = \tr{\A^\top\A \cdot f(\B\D^2\B^\top)}. $$
Similarly, for the last term we obtain
$$\E\langle \x, \A\sigma(\B\x) \rangle = \E\langle \tilde{\x}, \D\A\sigma(\B\D\tilde{\x}) \rangle= c_1\tr{\D\A\B\D}.$$
Finally, since $\sigma$ is homogeneous, by abuse of notation we can replace $\B\D$ by any $\B$ with unit-norm rows. This follows from the fact that, similarly to the proof of Lemma \ref{lemma:closed_from_of_population_risk} (namely, equations \eqref{eq:cfpr2} and \eqref{eq:cfpr3}), we have that
\begin{equation*}
\begin{split}
    &\E\|\A\sigma(\B \D \tilde{\x})\|_2^2 = \sum_{i,j=1}^n \langle \a_i, \a_j \rangle \cdot \E \left[\sigma(\langle (\B\D)_{i,:},\tilde{\x} \rangle) \cdot \sigma(\langle (\B\D)_{j,:},\tilde{\x} \rangle)\right], \\
    &\E\langle \tilde{\x}, \D\A\sigma(\B\D\tilde{\x}) \rangle = \sum_{i=1}^d\sum_{j=1}^n a_j^i \cdot \E [ (D_{i,i} \cdot \tilde{x}_i)  \cdot \sigma(\langle (\B\D)_{j,:}, \tilde{\x} \rangle)],
\end{split}
\end{equation*}
which, by homogeneity, readily gives that the norm of $(\B\D)_{i,:}$ can be pushed into the corresponding $\a_i$.

As a result, the statement of Lemma \ref{lemma:popriskD} readily follows by comparing the terms.
\end{proof}

\section{Proofs of Lower Bound on Loss (Section \ref{subsec:lb})}\label{appendix:lb}

\subsection{Case $r \leq 1$}\label{appendix:lowrate}

\subsubsection{Lower bound on $\widetilde{R}(r)$}

\begin{lemma}\label{lemma:allin}
Let $\A = [\a_1,\ldots,\a_n]\in \mathbb R^{d\times n}$ and $\B^\top=[\b_1, \ldots, \b_n]\in \mathbb R^{d\times n}$, with $\|\b_i\|_2=1$ for $i\in [n]$. Let $c_1$ and $f(\cdot)$ be defined as per Lemma \ref{lemma:closed_from_of_population_risk}. Then, the following bound holds:
\begin{equation}\label{eq:lower-bound}
    \mathcal{L}_l(\A,\B):=\tr{\A^{\top}\A \cdot (\B\B^{\top})^{\circ(2\ell+1)}} - \frac{2c_1}{f(1)} \cdot \tr{\B\A} \geq -\frac{c_1^2}{(f(1))^2} \cdot n.
\end{equation}
\end{lemma}

\begin{proof}[Proof of Lemma \ref{lemma:allin}]
For any symmetric $\P, \Q, \T\in \mathbb R^{n\times n}$, a direct computation readily gives that
\begin{equation}\label{eq:dcomp}
 \tr{\P\cdot(\Q\circ \T)}=\tr{(\P\circ \Q)\cdot \T)}.   
\end{equation}
Thus, by taking $\P=\A^{\top}\A$, $\Q=(\B\B^{\top})^{\circ \ell}$ and $\T=(\B\B^{\top})^{\circ(\ell+1)}$, we obtain
$$
\tr{\A^{\top}\A \cdot (\B\B^{\top})^{\circ(2\ell+1)}} = \tr{(\A^{\top}\A \circ (\B\B^{\top})^{\circ \ell} ) \cdot (\B\B^{\top} \circ (\B\B^{\top})^{\circ \ell})}.
$$
Note that $\B\B^{\top}$ is PSD and, therefore, $(\B\B^{\top})^{\circ \ell}$ is also PSD by Schur product theorem. Furthermore, as the rows of $B$ have unit norm, $(\B\B^{\top})^{\circ \ell}$ has unit diagonal. As a result, if we show that, for any PSD matrix $\Q$ with unit diagonal entries,
\begin{equation}\label{relax:1}
  \tr{(\A^{\top}\A \circ \Q ) \cdot (\B\B^{\top} \circ \Q)}-\frac{2c_1}{f(1)}\cdot\tr{\B\A}\geq -\frac{c^2_1}{(f(1))^2} \cdot n,
\end{equation}
then the claim \eqref{eq:lower-bound} immediately follows. 

As $\Q$ is a PSD matrix with unit diagonal, it admits the following decomposition
\begin{equation}\label{Q:SVD}
    \Q = \sum_{i=1}^n \u_i \u_i^{\top}, \quad \D_i = \mathrm{Diag}(\u_i), \quad \sum_{i=1}^n \D_i^2 = \I.
\end{equation}
In this view, defining
$$
\A_i = \A \D_i, \quad \B_i = \D_i \B,
$$
we can rewrite the LHS of \eqref{relax:1} in a more convenient form for further analysis. In particular, for the second term we deduce the following
\begin{equation*}
\begin{split}
\tr{\B\A} = \tr{\A\B} = \tr{\A\cdot \left(\sum_{i=1}^n \D_i^2\right) \cdot \B } = \sum_{i=1}^n \tr{\A \cdot \D_i^2 \cdot \B} &= \sum_{i=1}^n \tr{(\A \D_i) \cdot (\D_i\B)}
\\ &= \sum_{i=1}^n\tr{\A_i \B_i}.
\end{split}
\end{equation*}
Let us now rearrange the first term of \eqref{relax:1}. Notice that
\begin{align*}
    (\A^{\top}\A \circ \Q)_{i,j} = \sum_{k=1}^n \langle \a_i, \a_j \rangle \cdot u_k^i u_k^j = \sum_{k=1}^n \langle \a_i \cdot u_k^i, \a_j \cdot u_k^j \rangle = \sum_{k=1}^n ((\A\D_k)^{\top} \cdot (\A\D_k))_{i,j} = \sum_{k=1}^n (\A_k^{\top} \A_k)_{i,j}.
\end{align*}
In the same fashion we get
$$
(\B\B^{\top} \circ \Q)_{i,j} = \sum_{k = 1}^n (\B_k \B_k^{\top})_{i,j},
$$
from which we deduce that
$$
 \tr{(\A^{\top}\A \circ \Q ) \cdot (\B\B^{\top} \circ \Q )} = \sum_{i,j=1}^n \tr{\A_i^{\top}\A_i\B_j\B_j^{\top}}.
$$
Therefore, the proof of \eqref{relax:1} can be obtained by proving that, for \emph{any} matrices $\A_1, \ldots, \A_n\in \mathbb R^{d\times n}$ and $\B_1, \ldots, \B_n\in \mathbb R^{n\times d}$, 
\begin{equation}\label{relax:2}
    \sum_{i,j=1}^n \tr{\A_i^{\top}\A_i\B_j\B_j^{\top}} - \frac{2c_1}{f(1)} \cdot \sum_{i=1}^n \tr{\A_i \B_i} + \frac{c_1^2}{(f(1))^2} \tr{\I} \geq 0.
\end{equation}
To show the last claim, let us define the following matrices
$$
\X = \sum_{i=1}^n \A_i^{\top} \A_i, \quad \Y = \sum_{i=1}^n \B_i \B_i^{\top}, \quad \Z=\sum_{i=1}^n \B_i \A_i,
$$
which allows us to rewrite the statement of \eqref{relax:2} as 
\begin{equation}\label{relax:3}
    \tr{\X\Y - \frac{2c_1}{f(1)} \cdot \Z + \frac{c_1^2}{(f(1))^2} \cdot \I} \geq 0.
\end{equation}
Note that $\X$ is PSD, hence it has a symmetric square root, which we denote by $\sqrt{\X}$. Using the continuity of the quantities involved in the LHS of \eqref{relax:3}, we can assume without loss of generality that $\X$ is invertible. In fact, the following quantities are continuous: trace, matrix product, matrix transpose. In addition, we can always introduce a small perturbation to $\A_i$'s which makes $\X$ full-rank. Thus, it suffices to show that  \eqref{relax:3} holds for $\A_i$'s such that $\X$ is invertible.

In this view, for any matrix $\T \in \mathbb{R}^{n\times n}$, we have
\begin{align}\label{relax:4}
    0 \leq \sum_{i=1}^n\left\|\frac{c_1}{f(1)} \cdot \T\A_i^{\top} - \sqrt{\X}\B_i\right\|_F^2 &= \sum_{i=1}^n \tr{\left(\frac{c_1}{f(1)} \cdot \T\A_i^{\top} - \sqrt{\X}\B_i\right) \cdot \left(\frac{c_1}{f(1)} \cdot \A_i\T^{\top} - \B_i^{\top}\sqrt{\X}\right)} \nonumber \\ 
    &= \sum_{i=1}^n \tr{\frac{c^2_1}{(f(1))^2} \cdot \T \A_i^{\top} \A_i \T^{\top} - \frac{2c_1}{f(1)}\sqrt{\X}\B_i\A_i \T^{\top} + \X \B_i \B_i^{\top}} \nonumber \\ 
    &= \tr{\frac{c^2_1}{(f(1))^2} \cdot \T\X\T^{\top} - \frac{2c_1}{f(1)} \sqrt{\X} \Z \T^{\top} + \X\Y},
\end{align}
where in the second line we used that  $\tr{\M}=\tr{\M^\top}$ for any $\M$, and $\tr{\M\N}=\tr{\N\M}$ for any $\M, \N$. 

As $\X$ is invertible, its square root $\sqrt{\X}$ is invertible. As $\X$ is also PSD, its inverse, i.e., $\X^{-1}$, is PSD and, hence, it has a symmetric square root, i.e., $\sqrt{\X^{-1}}$. In this view, we get that
$$
\sqrt{\X^{-1}} = (\sqrt{\X})^{-1}.
$$
Thus, by picking $\T = (\sqrt{\X})^{-1}$, we obtain
$$
\T^{\top}\T = \T^2 = \X^{-1}, \quad  \T^{\top} \sqrt{\X} = \T \sqrt{\X} = \I.
$$
Using these observations, we deduce that the RHS of \eqref{relax:4} is equal to the LHS
of \eqref{relax:3}, which concludes the proof.
\end{proof}

\subsubsection{Matrices in $\mathcal{H}_{n,d}$ Are the Only Minimizers}

\begin{lemma}\label{uniqueunder} Let $\A\in \mathbb R^{d\times n}$ and $\B^\top=[\b_1, \ldots, \b_n]\in \mathbb R^{n\times d}$, with $\|\b_i\|_2=1$ for $i\in [n]$. Let $c_1$ and $f(\cdot)$ be defined as per Lemma \ref{lemma:closed_from_of_population_risk}. Then, we have that the set of minimizers of
\begin{equation}\label{eq:minarcsin}
    \tr{\A^{\top}\A \cdot f(\B\B^{\top})} - 2 c_1 \cdot \tr{\B\A}
\end{equation}
coincides with the set $\mathcal{H}_{n,d}$ of weight-tied orthogonal matrices .
\end{lemma}

\begin{proof}[Proof of Lemma \ref{uniqueunder}]
A direct computation immediately shows that the lower bound \eqref{eq:lower-bound} is achieved for all $\ell \in \mathbb{N}$ by matrices $(\A,\B)$ that belong to the set $\mathcal{H}_{d, n}$. Define the sets of minimizers of \eqref{eq:lower-bound} as follows 
$$
\mathcal{M}_\ell := \argmin_{\A,\B:\|\b_i\|_2=1} \mathcal{L}_\ell(\A,\B) = \left\{(\A_\B,\B): \A_\B \in \argmin_\A \mathcal{L}_{\ell}(\A,\B),\  \B \in \argmin_{\B:\|\b_i\|_2=1} \mathcal{L}_{\ell}(\A_\B,\B)\right\}.
$$
We will now show that
\begin{equation}\label{eq:Ml}
\bigcap\limits_{l=0}^{\infty} \  \mathcal{M}_\ell = \mathcal{H}_{n, d}.
\end{equation}
As the Taylor coefficients of $f(\cdot)$ are non-negative, \eqref{eq:Ml} readily gives that the set of minimizers of \eqref{eq:minarcsin} coincides with $\mathcal{H}_{n, d}$. Futher, recall that $c_1 \neq 0$ and $\sum_{l=1}^{\infty} (c_{2l+1})^2 \neq 0$ and, hence, \eqref{eq:Ml} is the union of the linear term ($l=0$) and at least one non-linear ($l > 0$) term.

We first prove that it is enough to consider the case $r=1$. Thus, assume that the result holds for $n=d$ and consider now $ n< d$. We have that, for any orthogonal matrix $\O \in \mathbb{R}^{d \times d}$,
\begin{equation}\label{eq:inv}
\begin{split}
    \E_{\x}\norm{\x - \A\sigma(\B\x)}_2^2 &= 
        \E_{\x}\norm{\O\x - \A\sigma(\B\O\x)}_2^2 \\
    &=    \E_{\x}\norm{\x - \O^\top \A\sigma(\B\O\x)}_2^2,
\end{split}    
\end{equation}
where in the first step we have used the rotational invariance of $\x$, and in the second step we have multiplied the argument of the norm by the orthogonal matrix $\O^\top$. Thus, \eqref{eq:inv} gives that $(\A, \B)\in \mathcal{H}_{n, d}$ if and only if $(\O^\top \A, \B\O)\in \mathcal{H}_{n, d}$. 

Let us write the SVD of $\B$ as  $\U\D\V^\top$, where $\U \in \mathbb{R}^{n \times n}, \V \in \mathbb{R}^{d \times d}$ are orthogonal matrices and $\D \in \mathbb{R}^{n \times d}$ is a (rectangular) diagonal matrix. Thus, by taking $\O=\V$, one can assume that $\B$ has the form $(\B_{1:n,1:n}, \0_{1:n, 1:d-n})$, where $\B_{1:n,1:n}$ denotes the left $n\times n$ sub-matrix of $\B$ and $\0_{1:n, 1:d-n}$ denotes a $n\times (d-n)$ matrix of 0's.
We also write the decompositions $\A = ((\A_{1:n, 1:n})^\top, (\A_{n+1:d, 1:n})^\top)^\top $ and $\x = (\x_{1:n}, \x_{n+1:d})$, where $\A_{1:n, 1:n}$ (resp. $\A_{n+1:d, 1:n}$) denotes the top $n\times n$ (resp. bottom $(d-n)\times n$) sub-matrix of $\A$, and $\x_{1:n}$ (resp. $\x_{n+1:d}$) denotes the first $n$ (resp. last $d-n$) components of $\x$. Hence, the objective \eqref{eq:loss} can be expressed (up to the constant multiplicative factor $d^{-1}$) as the sum of 
$$\mathcal{R}_1(\A, \B) = \E\left[ \norm{\x_{1:n} - \A_{1:n,1:n} \sigma(\B_{1:n,1:n} \x_{1:n})}^2\right]$$
and
$$\mathcal{R}_2(\A, \B) = \E\left[ \norm{\x_{n+1:d} - \A_{n+1:d,1:n} \sigma(\B_{1:n,1: n} \x_{1:n})}^2\right].$$
As $\x_{n+1:d}$ has zero mean and it is independent from $\x_{1:n}$, we have that
$$
\mathcal{R}_2(\A, \B) = d - n + \E\left[ \norm{\A_{n+1:d,1:n} \sigma(\B_{1:n,1: n} \x_{1:n})}^2\right],
$$
which is minimized by setting $\A_{n+1:d,1:n}$ to $\0$. Note that $\mathcal{R}_1$ depends only on $\A_{1:n,1:n}, \B_{1:n,1:n}$ (and not on  $\A_{n+1:d,1:n}$), hence its minimizers are $(\A_{1:n,1:n}, \B_{1:n,1:n}) \in \mathcal{H}_{n, n}$ by our assumption on the $r=1$ case. As a result, by using that $(\A, \B)\in \mathcal{H}_{d, n}$ if and only if $(\O^\top \A, \B\O)\in \mathcal{H}_{d, n}$, we conclude that all the minimizers of the desired objective have the form $\O ((\A_{1:n, 1:n})^\top, (\0_{1:n-d, 1:n})^\top)^\top$ and $(\B_{1:n,1:n}, \0_{1:n, 1:d-n})\O^\top$, i.e., they form the set $\mathcal{H}_{n,d}$ defined in \eqref{haar:min}.


It remains to prove the result for $r=1$.
First, consider $\ell=0$. In this case, we have 
\begin{align}\label{loloss}
    \L_0(\A, \B) &= \tr{\A^\top \A\B\B^\top} - \frac{2c_1}{f(1)} \cdot \tr{\B\A}\nonumber\\
    & = \tr{\B^\top \A^\top \A\B} - \frac{2c_1}{f(1)} \cdot  \tr{\A\B}\nonumber\\
    & = \norm{\A\B}_F^2 - \frac{2c_1}{f(1)} \cdot \tr{\A\B},
\end{align}
where we have used that the trace is invariant under cyclic permutation. 
Notice that the minimizer of \eqref{loloss} is clearly $\A\B=\frac{c_1}{f(1)}\I_d$.


Consider some $\ell \geq 1$. As $\A\B = \frac{c_1}{f(1)}\I_d$ and $\A, \B$ are square matrices, $\B$ is invertible and $\A^\top \A = \frac{c_1^2}{(f(1))^2}\cdot (\B\B^\top)^{-1}$. Thus, 
\begin{equation}\label{eq:Llrisk}
\begin{split}
    \L_\ell(\A, \B) &= \tr{\A^\top \A (\B\B^{\top})^{\circ (2\ell+1)}} - \frac{2c_1}{f(1)}\cdot\tr{\B\A} \\
    &= \frac{c_1^2}{(f(1))^2} \cdot \tr{(\B\B^\top)^{-1} (\B\B^{\top})^{\circ (2\ell+1)}} - \frac{2c_1^2}{(f(1))^2} \cdot n.
\end{split}    
\end{equation}
Let $\P = \B\B^\top$. Note that $\P$ is symmetric and, hence, also its inverse is symmetric. Then, by using \eqref{eq:dcomp}, we have that 
\begin{equation}\label{eq:shursim}
   \tr{\P^{-1} \P^{\circ (2\ell +1)}}
= \tr{(\P^{-1} \circ \P) \P^{\circ 2l}}.
\end{equation}
An application of Theorem {5} in \cite{visick2000quantitative} gives that
\begin{equation}\label{schurinv}
   \P \circ \P^{-1} \succeq \I, 
\end{equation}
where $\succeq$ denotes majorization in the PSD sense. We now show that $ \P \circ \P^{-1} = \I$. 
To do so, suppose by contradiction that 
$$
\P \circ \P^{-1} = \I + \bR,
$$
for some $\bR \succeq \0$ such that $\bR \neq \0$. 
Hence, 
\begin{equation}\label{eq:cdiction}
\tr{(\P^{-1} \circ \P) \P^{\circ 2\ell}} = \tr{\P^{\circ 2\ell}} + \tr{\bR \P^{\circ 2\ell}} = n +  \tr{\bR \P^{\circ 2\ell}},
\end{equation}
where in the last equality we use that $\P$ (and, consequently, $\P^{\circ 2\ell}$) has unit diagonal. By the Schur product theorem, $\P^{\circ 2\ell}\succ \0$ and, hence, it admits a square root. Thus, we get
$$
\tr{\bR \P^{\circ 2\ell}} = \tr{\sqrt{ \P^{\circ 2\ell}}\cdot \bR \cdot \sqrt{ \P^{\circ 2\ell}}}.
$$
It is easy to see that the matrix $\sqrt{\P^{\circ 2\ell}}\cdot \bR \cdot \sqrt{ \P^{\circ 2\ell}}$ is PSD and, thus, 
$$
\tr{\sqrt{ \P^{\circ 2\ell}}\cdot \bR \cdot \sqrt{ \P^{\circ 2\ell}}} \geq 0,
$$
where the inequality is strict if and only if the corresponding matrix has only zero eigenvalues.
However, for any non-zero $\v\in\mathbb{R}^n$, we have that
$$
\u_\v := \sqrt{ \P^{\circ 2\ell}} \cdot \v \neq 0,
$$
since $\sqrt{\P^{\circ 2\ell}}$ is strictly positive definite (as $\P^{\circ 2\ell} \succ \0$) and, thus, it does not have $0$ eigenvalues. Hence, if
$$
\v^\top \cdot \sqrt{ \P^{\circ 2\ell}} \cdot \bR \cdot \sqrt{ \P^{\circ 2\ell}} \cdot \v = \u_\v^\top \bR \u_\v = 0,
$$
then $\u_\v \neq \0$ is an eigenvector of $\bR$ corresponding to a zero eigenvalue. In this view, if $\sqrt{ \P^{\circ 2\ell}}\cdot \bR \cdot \sqrt{ \P^{\circ 2\ell}}$ has all zero eigenvalues, then all eigenvalues of $\bR$ are zero. As $\bR$ cannot be the zero matrix, by using \eqref{eq:cdiction}, we conclude that 
\begin{equation}\label{eq:cdiction2}
\tr{(\P^{-1} \circ \P) \P^{\circ 2\ell}} > n .
\end{equation}

By combining \eqref{eq:Llrisk}, \eqref{eq:shursim} and \eqref{eq:cdiction2}, we have that $\L_\ell(\A, \B)>-c_1^2n/(f(1))^2$, which contradicts with the fact that $(\A,\B)$ is a minimizer (since any $(\A',\B') \in \mathcal{H}_{n, d}$ achieves the value of $-c_1^2n/(f(1))^2$). Therefore, we conclude that $ \P \circ \P^{-1} = \I$.

At this point, we show that $ \P \circ \P^{-1} = \I$ implies that $ \P = \I$.  Note that $\P$ is a Gram matrix, and let its basis be $\{\b_1,\cdots,\b_n\}$. Define
$$
\b'_i = \b_i - \tilde{\b}_i,
$$
where $\tilde{\b}_i$ is orthogonal projection of $\b_i$ onto the space spanned by $\{\b_j\}^n_{j\neq i}$. From a well-known result (see, for instance, Theorem {2.1} in \cite{del1995statistical}) we have that
\begin{equation}\label{dualprop}
    \P^{-1}_{ii} = \frac{1}{\|\b_i'\|^2_2}.
\end{equation}
Hence, we obtain that
\begin{equation}\label{dualnorm}
    \|\b'_i\|_2 \leq \|\b_i\|_2 = 1,
\end{equation}
where the inequality is sharp only if $\b_i$ is orthogonal to all $\{\b_j\}^n_{j\neq i}$. Then, from \eqref{dualprop}, we deduce
\begin{equation}\label{eq:sharp}
n = \tr{\I} = \tr{\P \circ \P^{-1}} = \sum_{i=1}^{n} \|\b_i\|_2^2 \cdot \frac{1}{\|\b_i'\|^2_2} = \sum_{i=1}^{n} \frac{1}{\|\b_i'\|^2_2}.
\end{equation}
 By combining \eqref{dualnorm} and \eqref{eq:sharp}, we conclude that $\{\b_i\}_{i\in [n]}$ form an orthonormal basis, and, hence, $\P=\I$. This means that \eqref{eq:Ml} holds for $r=1$ since
$$
\eqref{eq:minarcsin} = \sum_{\ell=1}^{\infty} (c_{2\ell + 1})^2 \cdot \mathcal{L}_{\ell} (\A,\B),
$$
which concludes the proof.
\end{proof}

\begin{proof}[Proof of Theorem \ref{thm:tightlb_lowrate}] It follows by combining the results of Lemma \ref{lemma:allin} and \ref{uniqueunder}.
\end{proof}

\subsection{Case $r > 1$}\label{appendix:highrate}

\subsubsection{Lower bound on $\widetilde{R}(r)$}\label{app:pflbr1}

\begin{proof}[Proof of Proposition \ref{prop:highratelb}] An application of Theorem {A} in \cite{khare2021sharp} gives that
$$
\tr{\A^{\top}\A\B\B^{\top}} = \langle \mathbf{1}, (\A^{\top}\A \circ \B\B^{\top}) \mathbf{1} \rangle \geq \frac{1}{d}\langle \mathbf{1}, (\mathrm{Diag}(\B\A)\mathrm{Diag}(\B\A)^{\top}) \mathbf{1} \rangle = \frac{1}{d}\left(\tr{\B\A}\right)^2,
$$
where $\mathrm{Diag}(\B\A) \in \mathbb{R}^n$ stands for the vector with entries corresponding to the diagonal of the matrix $\B\A$. Hence, we have
\begin{equation}\label{2.4.1}
    \tr{\A^{\top}\A \cdot f(\B\B^{\top})} - 2c_1\cdot\tr{\B\A} \geq \frac{c_1^2}{d}\left(\tr{\B\A}\right)^2 + \sum_{\ell=1}^{\infty}(c_{2\ell+1})^2\cdot\tr{\A^{\top}\A\cdot(\B\B^{\top})^{\circ 2\ell +1}} - 2c_1\cdot\tr{\B\A}.
\end{equation}
Define $\alpha := f(1)-c^2_1$. Then, for any $\beta\in[0,1]$, we can rewrite the RHS of \eqref{2.4.1} as
\begin{equation}\label{2.4.2}
    \left[\frac{c_1^2}{d}\left(\tr{\B\A}\right)^2 - 2(1-\beta)c_1\cdot\tr{\B\A}\right] + \sum_{\ell=1}^{\infty}(c_{2\ell+1})^2 \cdot \left( \tr{\A^{\top}\A\cdot(\B\B^{\top})^{\circ 2\ell +1}} - \frac{2\beta c_1}{\alpha}\cdot\tr{\B\A}\right).
\end{equation}
The first term in \eqref{2.4.2} is a quadratic polynomial in $\tr{\B\A}$. Hence, we have that
\begin{equation}\label{2.4.3}
\left[\frac{c_1^2}{d}\left(\tr{\B\A}\right)^2 - 2(1-\beta)c_1\cdot\tr{\B\A}\right] \geq -d(1-\beta)^2.
\end{equation}
Define $\B_e := \left[\B, \0_{1:n,1:n-d}\right]$ and $\A_e^\top := [\A^\top, \0_{1:n,1:n-d}]$. One can readily verify that 
the traces in the second term of \eqref{2.4.2} remain unchanged if we replace $\A$ and $\B$ with $\A_e$ and $\B_e$, respectively. Note that $\A_e, \B_e$ are square matrices, hence we can apply Lemma \ref{lemma:allin} (which readily generalizes to a different scaling in front of the second trace) to get
\begin{align}\label{2.4.4}
\sum_{\ell=1}^\infty (c_{2\ell+1})^2 \cdot \left( \tr{\A^{\top}\A\cdot(\B\B^{\top})^{\circ 2\ell +1}} - \frac{2\beta c_1}{\alpha}\cdot\tr{\B\A} \right) \geq -\sum_{\ell=1}^{\infty} (c_{2\ell+1})^2 \cdot \frac{\beta^2c_1^2}{\alpha^2}n = - \frac{\beta^2c_1^2}{\alpha}n.
\end{align}
By combining \eqref{2.4.1}, \eqref{2.4.2}, \eqref{2.4.3} and \eqref{2.4.4}, we obtain that
\begin{equation}\label{2.4.5}
\frac{1}{d}\left(\tr{\A^{\top}\A \cdot f(\B\B^{\top})} - 2\cdot\tr{\A\B}\right) +1\ge    1-(1-\beta)^2 - \frac{\beta^2c_1^2}{\alpha}r.
\end{equation}
By taking $\beta = \alpha/(c_1^2 r+\alpha)$ and re-arranging the RHS of \eqref{2.4.5}, the desired result readily follows. 
\end{proof}


\subsubsection{Asymptotic Achievability of the Lower Bound}

\begin{lemma}\label{lem:optBresult}
    Let $\A, \B $ be defined as in \eqref{eq:defoptB}. Then, for any $\epsilon >0$, we have that, with probability at least $1-c/d^2$,
    \begin{equation*}
        \left|\left(\tr{\A^{\top}\A f(\B\B^{\top})} - 2c_1\tr{\A\B}\right) - \left(\beta^2c_1^2rn + \beta^2 \alpha n - 2c_1\beta n\right)\right|\le  Cn^{\frac{1}{2}+\epsilon}.
    \end{equation*}
    Thus, choosing $\beta = \frac{c_1}{c_1^2 r+\alpha}$ the loss approaches $1-\frac{r}{r+\frac{\alpha}{c_1^2}}$, i.e., with the same probability,
    $$ \left|\left( 1 + \frac{1}{d}\left(\tr{\A^{\top}\A f(\B\B^{\top})} - 2c_1\tr{\A\B}\right)\right)- \left(1-\frac{r}{r+\frac{\alpha}{c_1^2}}\right)\right| \le C d^{-\frac{1}{2}+\epsilon}.$$
    Here, the constants $c, C$ depend only on $r$ and $\epsilon$.
\end{lemma}
We start by proving the following.
\begin{lemma}\label{lem:optBconc}
    Let $\hat{\B}, \B$ be defined as in \eqref{eq:defoptB}. Then, for any  $\epsilon > 0$, we have that, with probability at least $1-c/d^2$,
    $$\max_{i, j}\abs{\frac{(\B\B^{\top})_{i,j}}{(\hat{\B}\hat{\B}^{\top})_{i,j }}-1} \leq C n^{-\frac{1}{2}+\epsilon}.$$
   Here, the constants $c, C$ depend only on $r$ and $\epsilon$.

\end{lemma}
\begin{proof}
    If $\U \in \mathbb{R}^{n\times n}$ is sampled uniformly from $\mathbb{SO}(n)$, then it follows from rotational invariance that any fixed row or column is uniformly distributed on the $n$-dimensional sphere $\mathbb{S}^{n-1}$.
    Thus, any fixed row of $\U$ is distributed as $\g/\norm{\g}_2$, where $\g  \sim \mathcal{N}\left(0, \I/n \right)$.
    Now, it follows from the concentration of $\norm{\g}_2$ (see e.g. Theorem 3.1.1 in \cite{vershynin2018high}) that $\psin{\norm{\g}_2 - 1 } \leq C n^{-\frac{1}{2}}$, where $\psin{\cdot}$ denotes the sub-Gaussian norm.
    Denote by $\g_d \in \R^d$ the first $d$ components of $\g_d$.
    Then, by the same reasoning, it holds that $\psin{\sqrt{r}\norm{\g_d}_2 - 1 } \leq cd^{-\frac{1}{2}}$.
    Looking at the definition of $\hat{\B}$, we have that, for any fixed $i$, the distribution of its rows is given by $\hat{\b}_i \sim \sqrt{r} \g_d/\norm{\g}_2$. Furthermore, for any pair of indices $i, j$, we have that
    \begin{equation*}
        \frac{(\B\B^{\top})_{i,j}}{(\hat{\B}\hat{\B}^{\top})_{i,j }} = \frac{1}{\lVert\hat{\b}_i\rVert_2\cdot \lVert\hat{\b}_j\rVert_2}.
    \end{equation*}
    Hence,
    \begin{equation*}
        \mathbb{P}\left(\abs{\frac{(\B\B^{\top})_{i,j}}{(\hat{\B}\hat{\B}^{\top})_{i,j }} - 1} \leq n^{-\frac{1}{2}+\epsilon} \right) =\mathbb{P}\left(\abs{\frac{1}{\lVert\hat{\b}_i\rVert_2\cdot \lVert\hat{\b}_j\rVert_2} - 1} \leq n^{-\frac{1}{2}+\epsilon} \right) 
        \leq C \exp\left(-\frac{d^\epsilon}{C}\right).
    \end{equation*} 
    Now a simple union bound over all rows gives us
    \begin{equation*}
        \mathbb{P}\left(\max_{i, j}\abs{\frac{1}{\lVert\hat{\b}_i\rVert_2\cdot \lVert\hat{\b}_j\rVert_2} - 1} \leq n^{-\frac{1}{2}+\epsilon} \right)
        \leq Cn \exp\left(-\frac{d^\epsilon}{C}\right) \leq  \frac{C}{d^2},
    \end{equation*}
    which implies the desired result.
\end{proof}

Next, we bound the traces of the terms $\B\B^{\top}(\B\B^{\top})^{\circ (2l +1)}$. We start with the case $l=0$.

\begin{lemma}\label{lem:optBl0}
    Let $\B$ be defined as in \eqref{eq:defoptB}. Then,
    for any $\epsilon > 0$, with probability at least $1-c/d^2$,
    $$\left|\tr{\B\B^{\top}(\B\B^{\top})} -rn \right|\le  Cd^{\frac{1}{2}+\epsilon}.$$
   Here, the constants $c, C$ depend only on $r$ and $\epsilon$.
\end{lemma}
\begin{proof}
   Note that 
   \begin{equation*}
       \begin{split}
      \tr{\B\B^{\top}(\B\B^{\top})} = \sum_{i,j} \left((\B\B^{\top})_{i,j}\right)^2 = \sum_{i,j} \left(\frac{\left((\B\B^{\top})_{i,j}\right)^2}{\left((\hat{\B}\hat{\B}^{\top})_{i,j}\right)^2}-1\right) \left((\hat{\B}\hat{\B}^{\top})_{i,j}\right)^2 +   \tr{\hat \B\hat \B^{\top}(\hat \B\hat \B^{\top})} . 
       \end{split}
   \end{equation*}
   Thus, an application of Lemma \ref{lem:optBconc} gives that, with probability at least $1-c/d^2$,
\begin{equation}\label{eq:bd1}
\left|\tr{\B\B^{\top}(\B\B^{\top})} - \tr{\hat{\B}\hat{\B}^{\top}(\hat{\B}\hat{\B}^{\top})}\right|\le \tr{\hat{\B}\hat{\B}^{\top}(\hat{\B}\hat{\B}^{\top})}\cdot  C d^{-\frac{1}{2}+\epsilon}.
\end{equation}
Since the trace is invariant under cyclic permutation, we readily have that 
\begin{equation}\label{eq:bd2}
    \tr{\hat{\B}\hat{\B}^{\top}(\hat{\B}\hat{\B}^{\top})}=rn.
\end{equation}
By combining \eqref{eq:bd1} and \eqref{eq:bd2}, the desired result follows. 
\end{proof}

Finally, we consider the higher order terms for $l \geq 1$.
\begin{lemma}\label{lem:optBhigh}
    Let $\B$ be defined as in \eqref{eq:defoptB}. Then, for any $\epsilon >0$, we  have that, with probability at least $1 - c/d^2$,
    $$\sup_{l\ge 1}\left|\tr{\B\B^{\top}(\B\B^{\top})^{\circ (2l +1)}} - n\right| \le C \log^2 n.$$
   Here, the constants $c, C$ depend only on $r$ and $\epsilon$.
\end{lemma}
\begin{proof}
    We first observe that 
    $$\tr{\B\B^{\top}(\B\B^{\top})^{\circ (2l +1)}} =\sum_{i,j} \left((\B\B^{\top})_{i, j}\right)^{2l+2} = n + \sum_{i\neq j} \left((\B\B^{\top})_{i, j}\right)^{2l+2}.$$
    An application of Lemma \ref{lem:optBconc} gives that, with probability $1 - c/d^2$,
    \begin{equation}\label{eq:intcc}
    \sup_{l\ge 1}\sum_{i\neq j} \left((\B\B^{\top})_{i, j}\right)^{2l+2} \leq \sup_{l\ge 1} \sum_{i\neq j} \left((1 + Cd^{-1/2 + \epsilon})\cdot (\hat{\B}\hat{\B}^{\top})_{i,j}\right)^{2l+2}.
        \end{equation}
    Furthermore, by using the first part of Lemma \ref{lem_D_diagonal} with $\A=\hat{\B}\hat{\B}^{\top}$, we have that, with probability at least $1 - 1/n^2$, the RHS of \eqref{eq:intcc} is lower bounded by
    $$
    \sup_{l\ge 1} \sum_{i\neq j} \left((1 + C d^{-1/2 + \epsilon})\cdot C\sqrt{\frac{\log n}{n}}\right)^{2l+2}\le C \log^2 n,
    $$
    which implies the desired result.
\end{proof}

At this point, we are ready to give the proof of Lemma \ref{lem:optBresult}.

\begin{proof}[Proof of Lemma \ref{lem:optBresult}]

Recall that $\{c_{2l + 1}\}_{l=0}^\infty$ denote the Taylor coefficients of $f(x)$, which by construction are non-negative.
By using that $\A = \beta \B^\top$, our objective becomes
\begin{align*}
   & \tr{\A^{\top}\A f(\B\B^{\top})} - 2c_1\tr{\A\B}  = \beta^2 \tr{\B\B^{\top}f(\B\B^{\top})} - 2c_1 \beta n \\
    &\hspace{2em}=\beta^2\sum_{\ell=0}^{\infty}(c_{2\ell+1})^2 \tr{\B\B^{\top}(\B\B^{\top})^{\circ (2\ell +1)}} - 2\beta n\\
    &\hspace{2em}= \beta^2c_1^2 rn +\beta^2\sum_{\ell=1}^{\infty}c_{2\ell+1} n- 2\beta n\\
    &\hspace{8em}+\beta^2c_1^2 \left(\tr{\B\B^{\top}(\B\B^{\top})}-rn\right)+\beta^2\sum_{\ell=1}^{\infty}c_{2\ell+1} \left(\tr{\B\B^{\top}(\B\B^{\top})^{\circ (2\ell +1)}}-n\right) .
\end{align*}
Then, by bounding the last two terms with Lemma \ref{lem:optBl0} and Lemma \ref{lem:optBhigh}, the desired result follows.
\end{proof}

\begin{proof}[Proof of Proposition \ref{prop:highrate_min}] The proof is a direct application of Lemma \ref{lem:optBresult}.
\end{proof}

\section{Global Convergence of Weight-tied Gradient Flow (Theorem \ref{thm:wt_gf})}\label{appendix:wt_glow}

Let $\B^\top = [\b_1,\cdots,\b_n]$. Recall that, under the weight-tying \eqref{eq:weight_tying}, the objective in \eqref{eq:reduce_opres} can be re-written 
as 
\begin{equation}\label{appendix_c:loss}
 \beta^2 \cdot \sum_{i,j=1}^n {\langle \b_i, \b_j \rangle}  \cdot f\left({\langle \b_i, \b_j \rangle}\right) - 2 \beta n.
\end{equation}
By definition in Theorem \ref{thm:wt_gf}, the residual $\phi(t)$ is given by
\begin{equation}\label{eq:defphiq}
\phi(t) := \sum_{i\neq j}^n{\langle \b_i, \b_j \rangle}  \cdot f\left({\langle \b_i, \b_j \rangle}\right).
\end{equation}
In this view, in accordance with \eqref{eq:wt_gf_dynamics}, we study the following gradient flow:
\begin{equation}\label{appendix_c:gflow}
    \begin{cases}
    \b_k(t) = -\beta^2(t) \cdot \left[\J_k(t) \sum_{i\neq j} \b_j(t) \cdot g(\langle \b_k(t), \b_j(t) \rangle)\right], \\
    \displaystyle\beta(t) = \frac{n}{nf(1) + \phi(t)},\\
    \|\b_k(0)\|_2 = 1,
    \end{cases}
\end{equation}
where $g(x):= x \cdot f'(x) + f(x)$, and we have rescaled the time of the dynamics by a factor $2$ to omit the factor $2$ in front of $\beta^2(t)$. From here on, we will suppress the time notation when it is clear from the context, for the sake of simplicity.  Note that one of the terms is absent in the summation, due to the fact that by definition of operator $\J_k$:
$$
\J_k \b_k = \0.
$$
In addition, since $\J_k$ defines the projection of the gradient on the tangent space at the point $\b_k$ of the unit sphere, along the trajectory of the gradient flow  \eqref{appendix_c:gflow} we have that $\|\b_k\|_2 = 1$.

The gradient flow \eqref{appendix_c:gflow} is well-defined (i.e., its solution exists and it is unique) when its RHS is Lipschitz continuous (see, for instance, \cite{santambrogio2017euclidean}). It suffices to check the Lipschitz continuity of $g(\cdot)$. Note that both $xf'(x)$ and $f(x)$ are Lipschitz continuous on any interval $[-1+\delta, 1-\delta]$ for some $\delta > 0$. Hence, the RHS of \eqref{appendix_c:gflow} is Lipschitz continuous, if
\begin{equation}\label{eq:bibj}
\max_{i\neq j}|\langle \b_i, \b_j \rangle| \leq 1 - \delta,
\end{equation}
where $\delta$ is bounded away from $0$ uniformly in $t$.

Recall that, by the assumption of Theorem \ref{thm:wt_gf}, we have that $\mathrm{rank}(\B(0)\B(0)^\top)=n$, hence $ \det (\B(0)\B(0)^\top) \geq \varepsilon_1$ for some $\varepsilon_1>0$. Thus, from the result in Lemma \ref{logdet_lb}, we obtain that 
\begin{equation}\label{eq:detB}
\det (\B(t)\B(t)^\top) \geq \varepsilon_1.
\end{equation}
Let $0<\lambda_1< \lambda_2< \ldots<\lambda_n $ denote the eigenvalues of $\B(t)\B(t)^\top$ in increasing order. Then, \eqref{eq:detB} directly gives that 
$$
\lambda_1 \prod_{i=2}^n \lambda_i \geq \varepsilon_1 > 0.
$$
Since $\B(t)\B(t)^\top$ has unit diagonal, we have that $\sum_{i=1}^n \lambda_i = n$. Hence, the smallest possible value of $\lambda_1$ during the gradient flow dynamics can be inferred from
$$
\lambda_1 \geq \frac{\varepsilon_1}{\prod_{i=2}^n \lambda_i}, 
$$
by picking the largest possible $\prod_{i=2}^n \lambda_i$ given the constraint $\sum_{i=2}^n \lambda_i \leq n$. This is achieved by taking
$$
\lambda_i = \frac{n}{n-1}, \quad \forall i \in \{2,\cdots,n\},
$$
which gives 
$$
\prod_{i=2}^n \lambda_i = \left(\frac{n}{n-1}\right)^{n-1} = \left(1+\frac{1}{n-1}\right)^{n-1} \leq C,
$$
where $C$ is a universal constant, since the RHS converges from below to Euler's number as $n$ increases. This proves that $\lambda_1$ is bounded away from zero uniformly in $t$. As a result, we can readily conclude that \eqref{eq:bibj} holds. To see this last claim, consider a vector $\v$ which has $1$ on position $i$ and $-\mathrm{sign}\langle \b_i, \b_j \rangle$ on position $j$. Hence, we have that
$$
2 \lambda_1 = \lambda_1 \cdot \|\v\|_2^2 \leq \v^{\top} (\B(t)\B(t)^{\top}) \v = 2 - 2 \cdot |\langle \b_i, \b_j \rangle| \Rightarrow |\langle \b_i, \b_j \rangle| \leq 1 - \lambda_1.
$$

Notice that
$$
\phi(t) \leq (n^2-n)f(1),
$$
since $x f(x) \leq f(1)$ for $|x| \leq 1$. Hence, we have that $\beta(t) \geq \frac{1}{nf(1)} > 0$. In this view, along the trajectory of the gradient flow \eqref{appendix_c:gflow}, the quantity $\phi(t)$
is \emph{strictly} decreasing until convergence, by the property of gradient flow.

\begin{lemma}[Characterization of stationary points]\label{wt_stationary} 
Consider the gradient flow \eqref{appendix_c:gflow}. Then, the following holds:
\begin{enumerate}
\setlength\itemsep{0mm}
    \item[(A)] \label{wtst:1} Any orthogonal set of $b_i$ is a stationary point and a global minimizer.
    \item[(B)] \label{wtst:2} The gradient flow \eqref{appendix_c:gflow} never escapes any subspace spanned by a set of linearly dependent $b_i$. However, for each such subspace there exists a direction in which \eqref{appendix_c:loss} can be improved.
\end{enumerate}
\end{lemma}

\begin{proof}[Proof of Lemma \ref{wt_stationary}] Recall that $\beta(t) > 0$ and $\{\b_i\}_{i=1}^n$. Then, the stationary point condition can be expressed as
\begin{equation}\label{st_cond}
    \J_k \sum_{j\neq k}  \b_j \cdot g\left(\langle \b_k,\b_j\rangle\right) = 0, \quad \forall k \in [n].
\end{equation}
Thus, any orthogonal set of vectors is clearly a stationary point by definition of $g(\cdot)$. Moreover, \eqref{appendix_c:loss} is minimized \emph{iff} $\B\B^\top = \I$ as $xf(x)$ is an even function since $f(\cdot)$ is odd.

Note that the kernel of the operator $\J_k$ is spanned by the vector $\b_k$. Thus, the condition \eqref{st_cond} is equivalent to
$$
\sum_{j\neq k}  \b_j \cdot g\left(\langle \b_k,\b_j\rangle\right) = \gamma_k \cdot \b_k,
$$
for some $\gamma_k \in \mathbb{R}$. One can readily verify that $g(x)=0$ if and only if $x=0$. Thus, 
either \emph{(i)} $\b_k$ is orthogonal to $\b_j$ for all $j\neq k$ and $\gamma_k = 0$, or \emph{(ii)} $\b_k$ lies in the span of $\{\b_j\}_{j\neq k}$. If condition \emph{(i)} holds for all $k\in [n]$, then $\{\b_i\}_{i=1}^n$ form an orthogonal set of vectors and we fall in category (A). If condition \emph{(ii)} holds for some $k\in [n]$, then we fall in category (B).

Now, let us show that, if $\{\b_i\}_{i=1}^n$ spans a sub-space of dimension smaller than $n$, then there is a direction along which the value of \eqref{appendix_c:loss} can be improved. Since the $\{\b_i\}_{i=1}^n$ are linearly dependent, there exists $\u$ of unit norm such that
\begin{equation}\label{eq:uip}
\langle \u, \b_j \rangle = 0,\quad \forall j \in [n].
\end{equation}
For some $k\in[n]$, consider the perturbation 
$$
\hat{\b}_k = \frac{1}{\sqrt{1+\lambda^2}} \cdot (\b_k + \lambda \cdot \u),
$$
which has unit norm as $\langle \b_k, \u \rangle = 0$. Recall that \eqref{appendix_c:loss} can be expressed as
\begin{equation}\label{rewrite:popwt}
    \beta^2 \left(2 \cdot \sum_{j\neq k}^n \left\langle \hat{\b}_k, \b_j \right\rangle f\left(\left\langle \hat{\b}_k, \b_j \right\rangle\right) + \frac{\pi}{2} + \sum_{i,j\neq k}^n \left\langle \b_i, \b_j \right\rangle f\left(\left\langle \b_i, \b_j \right\rangle\right)\right) - 2\beta n.
\end{equation}
Here, $\beta$ is chosen to be the minimizer of the quantity \eqref{rewrite:popwt} having fixed $\{\b_j\}_{j\neq k}$ and $\hat{\b}_k$. Thus, in order to prove that the population risk gets smaller by replacing $\b_k$ with $\hat{\b}_k$ for any $\lambda>0$, it suffices to show that 
the following quantity 
\begin{equation}\label{wt_diff_pert}
    \sum_{j\neq k}^n \left\langle \hat{\b}_k, \b_j \right\rangle f\left(\left\langle \hat{\b}_k, \b_j \right\rangle\right),
\end{equation}
is decreasing with $\lambda$. This last claim follows from the chain of inequalities below:
\begin{align}
    \eqref{wt_diff_pert} &= \frac{1}{\sqrt{1+\lambda^2}} \sum_{j\neq k} \langle \b_k, \b_j \rangle \cdot f\left(\frac{1}{\sqrt{1+\lambda^2}}\left\langle \b_k, \b_j \right\rangle\right) \\
    &= \frac{1}{\sqrt{1+\lambda^2}} \sum_{j\neq k} \langle \b_k, \b_j \rangle \cdot \sum_{\ell=0}^{\infty}\left(\frac{c_{2\ell+1}}{c_1}\right)^2\cdot\left(\frac{1}{\sqrt{1+\lambda^2}}\right)^{2l+1}\cdot\left\langle \b_k, \b_j \right\rangle^{2l+1}\\
    &\leq \left(\frac{1}{\sqrt{1+\lambda^2}}\right)^2 \sum_{j\neq k} \langle \b_k, \b_j \rangle \cdot \sum_{\ell=0}^{\infty}\left(\frac{c_{2\ell+1}}{c_1}\right)^2\cdot\left\langle \b_k, \b_j \right\rangle^{2l+1}\\
    &= \frac{1}{1+\lambda^2} \sum_{j\neq k} \langle \b_k, \b_j \rangle \cdot f\left(\left\langle \b_k, \b_j \right\rangle\right) < \sum_{j\neq k} \langle \b_k, \b_j \rangle \cdot f\left(\left\langle \b_k, \b_j \right\rangle\right),
\end{align}
where in the second line we substitute the Taylor expansion of $f(\cdot$), the inequality in the third line uses that the coefficients $\{c^2_{2l+1}\}_{l=0}^\infty$ are all non-negative, and the last inequality follows from the fact that $\lambda > 0$.

Finally, we show that the gradient flow \eqref{appendix_c:gflow} does not escape the degenerate sub-space. If $$\mathrm{dim}(\mathrm{span}(\{\b_i\}_{i=1}^n)) < n,$$ then there exists $\u\in\mathbb{R}^d$ such that \eqref{eq:uip} holds. 
By projecting the gradient expression \eqref{st_cond} onto $\u$, we have
$$
\left\langle \u, \J_k \sum_{j\neq k} \b_j \cdot g\left(\langle \b_k, \b_j \rangle\right) \right\rangle = 0.
$$
Hence, for any $k\in[n]$, the directional derivative of $\b_k$ in the direction of $\u$ is equal to zero, and the gradient flow does not escape the low-rank sub-space, which concludes the proof.
\end{proof}

In next lemma we show that, if at initialization $\{\b_i\}_{i=1}^n$ spans a sub-space of dimension $n$, then it will never get stuck in a low-rank sub-space.

\begin{lemma}[Linearly independent $\{\b_i\}_{i=1}^n$ stay linearly independent]\label{logdet_lb} Consider the gradient flow \eqref{appendix_c:gflow} with full rank initialization, i.e., $\mathrm{rank}(\B(0)\B(0)^\top) = n$. Then, the following holds
$$
\frac{\partial }{\partial t} \log\det(\B(t)\B(t)^\top) \geq 2\beta(t)^2 \cdot \phi(t) \geq 0,
$$
where $\B(t)^\top = [\b_1(t),\cdots,\b_n(t)]$ and $\phi(t)$ is defined in \eqref{eq:defphiq}. In particular, this implies that $\{\b_i\}_{i=1}^n$ stay full-rank along the gradient flow trajectory. 

\end{lemma}

\begin{proof}[Proof of Lemma \ref{logdet_lb}] Applying the chain rule and using that the time derivative of $\B$ is given by the gradient flow \eqref{appendix_c:gflow} implies that
$$
\frac{\partial}{\partial t} \log \det (\B\B^{\top}) = \mathrm{Tr}\left[ (\B\B^{\top})^{-1} \cdot \left( \frac{\partial \B}{\partial t} \cdot \B^\top +  \B \cdot \frac{\partial \B^{\top}}{\partial t} \right)\right],
$$
where 
$$
\frac{\partial \b_k}{\partial t} = - \beta(t)^2 \cdot \left(\J_k \sum_{j\neq k} \b_j \cdot g\left(\langle \b_k, \b_j \rangle\right)\right).
$$
Let us compute the quantity
$$
\left\langle \frac{\partial \b_k}{\partial t}, \b_l \right\rangle = \left(\frac{\partial \B}{\partial t} \cdot \B^\top\right)_{k,l}.
$$
By definition of $\J_k$, we have that
$$
\J_k \sum_{j\neq k} \b_j \cdot g\left(\langle \b_k, \b_j \rangle\right) = \sum_{j\neq k} \b_j \cdot g\left(\langle \b_k, \b_j \rangle\right) - \sum_{j\neq k} \b_k \cdot \langle \b_k, \b_j \rangle \cdot g\left(\langle \b_k, \b_j \rangle\right).
$$
Note that
$$
\left\langle \sum_{j\neq k} \b_k \cdot \langle \b_k, \b_j \rangle \cdot g\left(\langle \b_k, \b_j \rangle\right), \b_l\right\rangle = \left[\mathrm{Diag}\left[\1((\B\B^{\top} - \I)\circ g(\B\B^{\top}))\right]\cdot \B\B^{\top}\right]_{kl},
$$
and that
\begin{align*}
    \left\langle \sum_{j\neq k} \b_j \cdot g\left(\langle \b_k, \b_j \rangle\right), \b_l  \right\rangle
    = \left[ g(\B\B^{\top}) \cdot \B\B^{\top}\right]_{k,l} - g(1) \cdot [\B\B^{\top}]_{k,l}.
\end{align*}
By combining these last four equations, we conclude that
$$
\frac{\partial \B}{\partial t} \cdot \B^\top=- \beta(t)^2\left(g(\B\B^{\top}) \cdot \B\B^{\top}-g(1) \cdot \B\B^{\top}+\mathrm{Diag}\left[\1((\B\B^{\top} - \I)\circ g(\B\B^{\top}))\right]\cdot \B\B^{\top}\right).
$$
Furthermore,
\begin{align*}
    \B \cdot \frac{\partial \B^{\top}}{\partial t}=\left(\frac{\partial \B}{\partial t} \cdot \B^\top\right)^\top=- \beta(t)^2\Big(\B\B^{\top}\cdot g(\B\B^{\top}) &-g(1) \cdot \B\B^{\top}\\ &+\B\B^{\top}\cdot \mathrm{Diag}\left[\1((\B\B^{\top} - \I)\circ g(\B\B^{\top}))\right]\Big).
\end{align*}
Hence, by using the cyclic property of the trace, we get that
\begin{align*}
    \frac{\partial}{\partial t} \log \det (\B\B^{\top}) &=  2\beta(t)^2\cdot\tr{\mathrm{Diag}\left[\1((\B\B^{\top} - \I)\circ g(\B\B^{\top}))\right] } -2\beta(t)^2 \cdot \mathrm{Tr}\left[g(\B\B^{\top}) - g(1) \cdot \I\right] \\
    &= 0 + 2\beta(t)^2\cdot\sum_{i\neq j}^n \langle \b_i, \b_j \rangle \cdot g\left(\langle \b_i, \b_j \rangle\right),
\end{align*}
Now, note that
$$
x g(x) = x^2 f'(x) + xf(x) \geq 0,
$$
since $x^2 f'(x)$ and $xf(x)$ are non-negative functions, which concludes the proof.
\end{proof}
The result of Lemma \ref{logdet_lb} gives that $\mathrm{det}(\B\B^{\top})$ is non-decreasing. Hence, if $\lambda_{\rm min}(\B\B^{\top}) > \delta > 0$ at initialization, then this quantity will be bounded away from zero during the gradient flow dynamics and the gradient flow will not get stuck in a low-rank solution. Therefore, by Lemma \ref{wt_stationary}, the gradient flow converges to a global minimum, in which the rows of $\B$ are orthogonal vectors with unit norm. The speed at which this happens is characterized by the next lemma.

\begin{lemma}[Rate of convergence]\label{wt_speed} 
Consider the gradient flow \eqref{appendix_c:gflow} with full rank initialization, i.e., $\mathrm{rank}(\B(0)\B(0)^\top) = n$. Let  $T$ be the time at which $\phi(T)$ hits the value $\delta > 0$. Then, the following holds
\begin{equation}\label{eq:Tub}
T \leq - \det(\B(0)\B(0)^\top) \cdot \left(f(1) \cdot \ind\{\phi(0)>n\cdot f(1)\} + \frac{2f^2(1)}{\delta} \cdot \ind\{\delta \leq n\cdot f(1)\} \right).
\end{equation}
\end{lemma}

\begin{proof}[Proof of Lemma \ref{wt_speed}] For all $t$, we have that $\tr{\B(t)\B(t)^{\top}} = n$, which implies that $\det(\B(t)\B(t)^{\top}) \leq 1$ and, as a consequence, that $\log \det(\B(t)\B(t)^{\top}) \leq 0$.
From Lemma \ref{logdet_lb}, we know that 
$$
\frac{\partial }{\partial t} \log\det(\B(t)\B(t)^{\top}) \geq 2\beta(t)^2 \cdot \phi(t).
$$
In this view, using the exact expression \eqref{appendix_c:gflow} for $\beta(t)$, we get
\begin{equation}\label{wt_speed_master}
    -\log \det(\B(0)\B(0)^{\top}) \geq \log \det(\B(t)\B(t)^{\top})-\log \det(\B(t)\B(t)^{\top})\geq  \int_0^t  \frac{2}{\left(f(1) + \frac{\phi(s)}{n}\right)^2} \cdot \phi(s) \mathrm{d}s.
\end{equation}
\textbf{Stage 1.} Assume that $\phi(0) > n \cdot f(1)$, and let $T_1$ be such that $\phi(T_1)=n \cdot f(1)$. Recall that the function $\phi(t)$ is decreasing and note that $x/(1+x)^2$ is decreasing for $x\in[1,+\infty)$. In this view, we can lower bound the integrand in the RHS of \eqref{wt_speed_master} for all $t\le T_1$ by 
\begin{equation}\label{eq:lbint}
\frac{2\cdot\phi(0)}{\left(f(1) + \frac{\phi(0)}{n}\right)^2} \geq \frac{2(n-1)}{nf(1)} \geq \frac{1}{f(1)},
\end{equation}
where the first inequality follows from the definition \eqref{eq:defphiq} of $\phi(\cdot)$, which readily implies that $\phi(0) \leq f(1) \cdot n(n-1)$. 
Hence, by combining \eqref{wt_speed_master} with the lower bound \eqref{eq:lbint}, we get
$$
T_1 \leq -f(1) \cdot \log \det(\B(0)\B(0)^{\top}).
$$

\noindent\textbf{Stage 2.} Assume that $\phi(0) \leq n \cdot f(1)$. Let $\delta\in (0, n \cdot f(1)]$ be the desired precision which should be reached during the gradient flow, and let $T_2$ be such that $\phi(T_2) = \delta$. As $\phi(t)$ is decreasing, we have that
\begin{equation}\label{eq:lbint2}
\frac{1}{\left(f(1) + \frac{\phi(t)}{n}\right)^2} \geq \frac{1}{\left(f(1) + \frac{\phi(0)}{n}\right)^2} \geq \frac{1}{4f^2(1)},
\end{equation}
where in the last step we use that $\phi(0) \leq n \cdot f(1)$.
Hence, by combining \eqref{wt_speed_master} with the lower bound \eqref{eq:lbint2}, we get
$$
-\log \det(\B(0)\B(0)^{\top}) \geq \frac{1}{2f^2(1)} \cdot T_2 \delta,
$$
which implies that
$$
T_2 \leq  -\frac{2f^2(1)\cdot \log \det(\B(0)\B(0)^{\top})}{\delta}.
$$
By combining the results of both stages, the desired result \eqref{eq:Tub} readily follows.
\end{proof}

\begin{proof}[Proof of Theorem \ref{thm:wt_gf}] Theorem \ref{thm:wt_gf} is a compilation of the results presented in current section.
\end{proof}

\section{Global Convergence of Projected Gradient Descent (Theorem \ref{thm:gd_min})}\label{appendix:no_wtgauss}

Recall from statement of Theorem \ref{thm:gd_min} that
$$
f(x)= x + \sum_{\ell=3}^\infty c_\ell^2 x^{\ell} ,
$$
with  $\sum_{\ell=3}^\infty c_\ell^2 < \infty$. 
We also define $\alpha = \sum_{\ell=3}^\infty c_{\ell}^2$, and we assume that $\alpha >0$. In fact, if $\alpha = 0$, then the algorithm trivially converges after one step.
We denote by $C, c$ uniform positive constants (depending only on $r$ and $\alpha$) the value of which might change from term to term.
To make the notation lighter we will also but the time $t$ as a subscript (for example $\B(t)$ becomes $\B_t$).

We analyze the following projected gradient descent procedure for minimizing the population risk
\begin{equation}\label{eq:popriskobj}
\sum_{i,j=1}^n \langle \a_i , \a_j \rangle \cdot f\left(\left\langle \frac{\b_i}{\|\b_i\|_2} , \frac{\b_j}{\|\b_j\|_2}  \right\rangle\right) - 2\sum_{i=1}^n \left\langle \a_i , \frac{\b_i}{\|\b_i\|_2}  \right\rangle.
\end{equation}
Given unit-norm initial $\{\b_i\}_{i\in [n]}$, at each step we pick the optimal value of $\A$ given $\B$
\begin{equation}\label{eq:optA}
\A_t = \B_t^\top \left(f(\B_t\B_t^{\top})\right)^{-1}.
\end{equation}
Then, we update $\B_t$ with a gradient step and a projection on the sphere to keep the unit norm:
$$
\B'_t := \B_t - \eta \nabla_{\B_t}, \quad \B_{t+1} := \mathrm{proj}(\B'_t).
$$
Here, the operator $\mathrm{proj}(\M)$ normalizes the rows of $\M$ to be of unit norm and each row of $\nabla_{\B_t}$ is defined as the corresponding row of the gradient of $\B_t$, i.e.,
\begin{equation}\label{eq:defnablaBt}
(\nabla_{\B_t})_{k, :} = \underbrace{- 2 \J_k \a_k + 2\sum_{j \neq k} \langle \a_k, \a_j \rangle \J_k \b_j}_{:=\nabla_{\B_t}^1 \quad \text{(part 1)}} + \sum_{l=3}^\infty \underbrace{\ell c_\ell^2 \sum_{j \neq k} \langle \a_k, \a_j \rangle \langle \b_k, \b_j \rangle^{l-1} \J_k \b_j}_{:=\nabla_{\B_t}^2 \quad\text{(part 2)}},
\end{equation}
where $\J_k:= \I - \b_k\b_k^\top$ and we have omitted the iteration number $t$ on $\{\a_j, \b_j\}_{j\in [n]}$ to keep notation light.
Note that in \eqref{eq:defnablaBt} the norms $\norm{\b_i}_2$, $\norm{\b_j}_2$ no longer appear as the projection step enforces $\norm{\b_i}_2=1$.
At each step of the projected gradient descent dynamics, we decompose $\B_t\B_t^\top$ as follows:
\begin{equation}\label{eq:decBB}
\B_t\B_t^\top = \I + \Z_t + \X_t,     
\end{equation}
where $\B_0\B_0^\top = \U\Lam_0 \U^\top$, $\Z_t = \U(\Lam_t-\I) \U^\top$ and $\Lam_{t+1} = g(\Lam_t)$ for some function $g:\mathbb{R}^n\rightarrow\mathbb{R}^n$ which defines the spectrum evolution. Here, $\U$ is an orthogonal matrix that importantly does not depend on $t$ and $\Lam_t$ is the diagonal matrix containing the eigenvalues (i.e., $\U \Lam_t \U^\top$ is the SVD).
We also define $\X_t^D:=\mathrm{Diag}(\X_t)$ and $\X^O_t:=\X_t-\X_t^D$.

For now we will make the following assumptions, which will be proved later in the argument. There exist universal constants $C,C_X>0$ and $\delta\in (0, 1)$ (depending only on $r$) such that, with probability at least $1-Ce^{-cd}$,
\begin{align}\label{eq:assumptionZX}
    \begin{split}
    &\inf_{t\geq 0}\lambda_{\textrm{min}}(\Z_t) \geq -1 + \delta_r, \\
    &\sup_{t\geq 0}\opn{\Z_t} \leq C, \\
    & \sup_{t\geq 0}\opn{\X_t} \leq C_X\frac{\mathrm{poly}(\log d)}{\sqrt{d}},\\
    &\|\Lam_t-\I\|_{op} \leq C\,e^{-c \eta t}.
    \end{split}
\end{align}
Here, $\mathrm{poly}(\log d)$ is used to denote polynomial powers of $\log d$, i.e., $(\log d )^C$ for some universal constant $C$. 
In the assumptions \eqref{eq:assumptionZX}, we specifically distinguish the constant $C_X$ in the bound on $\|\X_t\|_{op}$ from the others. This important distinction between $C$ and $C_X$ will be apparent later to show that assumptions \eqref{eq:assumptionZX} indeed hold. 
Note also that, for sufficiently large $d$, \eqref{eq:assumptionZX} implies that
\begin{equation}\label{eq:assumptionZXbis}
\sup_{t\geq 0} \opn{\X_t} \leq 1.
\end{equation}


We are now ready to give the proof Theorem \ref{thm:gd_min}. For the convenience of the reader we restate it here.
\begin{theorem}\label{thm:GD-min_appendix}
Consider the projected gradient descent algorithm as described above applied to the objective \eqref{eq:reduce_opres} for any $f$ of the form $f(x) = x + \sum_{\ell=3} c_{\ell}^2 x^{\ell}$, where $\sum_{\ell=3} c_{\ell}^2 < \infty$. Initialize the algorithm with $\B_0$ equal to a row-normalized Gaussian, i.e., $(\B'_0)_{i,j} \sim \mathcal{N}(0,1/d)$, $(\B_0){i,:} = \mathbf{Proj}_{\mathbb{S}^{d-1}}\left(
(\B_0')_{i,:}\right)$. Let the step size $\eta$ be $\Theta(1/\sqrt{d})$.
    Then, for any $r<1$, we have that at any time $t = T/\eta$, with probability at least $1-Ce^{-cd}$,  
    $$\opn{\B_t\B_t^\top-\I} \leq C(1-c)^T,$$
    where $C>0$ and $c\in (0, 1]$ are universal constants depending only on $r$ and $f$.

\end{theorem}

Let $\bE^t  :=\bE(\X_t,\Z_t)\in\mathbb{R}^{n\times n}$ be a generic matrix whose operator norm is upper bounded by 
\begin{equation}\label{eq:defopnE}
    \opn{\bE^t  }\le C \left(\frac{\mathrm{poly}(\log d)}{\sqrt{d}} \cdot \|\Z_t\|_{op}^{1/2}+\|\X_t\|_{op}^2+\|\X_t\|_{op} \|\Z_t\|_{op}^{1/2}\right).
\end{equation}
We highlight that the constant in front of the upper-bound on the error term $\bE^t  $ is \emph{independent} of $C_X$ and $t$.


\begin{lemma}[Bound for the matrix inverse]\label{lemma:finv_bound} Assume that \eqref{eq:assumptionZX} holds. Then, for all sufficiently large $n$, with probability at least $1-1/d^2$, jointly for all $t\ge 0$ and $\ell\ge 3$, the following bounds hold 
\begin{equation}
    \label{eq:claimBB}
\|(\B_t\B_t^\top - \I)^{\circ \ell}\|_{op} \le \|\bE^t  \|_{op},
\end{equation}
\begin{equation}\label{eq:claimBB2}
    \|\big(f(\B_t\B_t^{\top})\big)^{-1}-(\alpha \I + \B_t\B_t^{\top})^{-1}\|_{op}  \le   \|\bE^t  \|_{op},
\end{equation}
where $\alpha$ was defined as $\alpha = \sum_{\ell=3}^\infty c_{\ell}^2$ .
\end{lemma}  

\begin{proof}[Proof of Lemma \ref{lemma:finv_bound}]
Note that, for any square matrices $\bR, \S\in \mathbb R^{n\times n}$, 
\begin{equation}\label{hadbound}
    \|\bR \circ \S \|_{op} \leq \sqrt{n} \|\S\|_{op}\max_{i,j}|\bR_{i,j}|.
\end{equation}
Thus, for $\ell\ge 3$, 
\begin{equation}\label{eq:intm1}
    \begin{split}
\opn{(\B_t\B_t^\top-\I)^{\circ \ell}}&\le \sqrt{n}\opn{(\B_t\B_t^\top-\I)^{\circ (\ell-3)}}\max_{i,j}|((\B_t\B_t^\top-\I)^{\circ 3})_{i,j}|\\
&= \sqrt{n}\opn{(\B_t\B_t^\top-\I)^{\circ (\ell-3)}}\max_{i\neq j}|((\B_t\B_t^\top-\I)^{\circ 3})_{i,j}|\\
&= \sqrt{n}\opn{(\B_t\B_t^\top-\I)^{\circ (\ell-3)}}\max_{i\neq j}|((\Z_t+\X_t)^{\circ 3})_{i,j}|,        
    \end{split}
\end{equation}
where in the first line we use \eqref{hadbound}, in the second line we use that $((\B_t\B_t^\top-\I)^{\circ 3})_{i,i}=0$ for $i\in [n]$ and in the third line we use the decomposition \eqref{eq:decBB}. 

Let us bound the off-diagonal entries of $\X_t$ via \eqref{eq:assumptionZX} and the off-diagonal entries of $\Z_t$ via Lemma \ref{lem_D_diagonal}. This gives that, with probability at least $1-1/d^2$, jointly for all $t\ge 0 $,  
\begin{equation}\label{eq:intm2}
    \max_{i\neq j}|((\Z_t+\X_t)^{\circ 3})_{i,j}| \le (C + C_X)^3\bigg(\frac{\mathrm{poly}(\log d)}{d}\bigg)^{3/2}.
\end{equation}
We will condition on this event (without explicitly mentioning it every time) for the reminder of the argument.
By combining \eqref{eq:intm1} and \eqref{eq:intm2}, we have that 
\begin{align}
\begin{split}\label{eq:intm3}
    \opn{(\B_t\B_t^\top-\I)^{\circ \ell}}&\le \sqrt{n}\left[(C + C_X)^3\bigg(\frac{\mathrm{poly}(\log d)}{d}\bigg)^{3/2}\right]\opn{(\B_t\B_t^\top-\I)^{\circ (\ell-3)}}\\ &\le \opn{(\B_t\B_t^\top-\I)^{\circ (\ell-3)}}
    \end{split}
\end{align}
where the last inequality holds for all sufficiently large $n$. 
Note that, for any square matrices $R, S$, an application of Theorem 1 in \cite{visick2000quantitative} gives that 
\begin{equation}\label{eq:bdRS}
 \|\bR\circ \S\|_{op} \leq \|\bR\|_{op}\|\S\|_{op}.   
\end{equation}
Hence,
\begin{equation}\label{eq:intm5}
\|(\B_t\B_t^\top - \I)^{\circ \ell}\|_{op} \leq \opn{(\B_t\B_t^\top-\I)^{\circ (\ell-3)}} \|(\B_t\B_t^\top - \I)^{\circ 3}\|_{op}.    
\end{equation}
Now, by using again \eqref{eq:bdRS} and the assumptions \eqref{eq:assumptionZX}, we have that, for $\ell\in [3]$, 
\begin{equation}\label{eq:intm4}
    \|(\B_t\B_t^\top - \I)^{\circ \ell}\|_{op} \leq C.
\end{equation}
Thus, by combining \eqref{eq:intm3} and \eqref{eq:intm4}, we obtain that $\opn{(\B_t\B_t^\top-\I)^{\circ (\ell-3)}}$ is uniformly bounded in $\ell$, which together with \eqref{eq:intm5} gives that
\begin{equation}
    \|(\B_t\B_t^\top - \I)^{\circ \ell}\|_{op} \leq C\|(\B_t\B_t^\top - \I)^{\circ 3}\|_{op}.
\end{equation}
We remark here that $C$ is independent of $l$ and $C_X$. This means that it suffices to prove the claim \eqref{eq:claimBB} for $l=3$.

To do so, define $\H:=\1\1^\top - \I$, hence, since $\B_t\B_t^\top$ has unit diagonal, we have that
\begin{align*}
    (\B_t\B_t^\top - \I)^{\circ 3} &= (\B_t\B_t^\top - \I)^{\circ 3} \circ \H  = (\U(\Lam_t - \I) \U^\top + \X_t^O + \X_t^D)^{\circ 3} \circ \H  \\ &= (\Z_t \circ \H  +\X_t^O \circ \H  +\X_t^D\circ \H )^{\circ 3} = (\Z_t \circ \H  + \X_t^{O})^{\circ 3} \\ &= (\Z_t \circ \H )^{\circ 3} + 3 (\Z_t \circ \H )^{\circ 2} \circ \X_t^{O} + 3 (\Z_t \circ \H ) \circ (\X_t^{O})^{\circ 2} + (\X_t^{O})^{\circ 3}.
\end{align*}
Using again \eqref{eq:bdRS} and that, for any $\bR \in \mathbb R^{n\times n}$,
$$
\|\bR  \circ \H \|_{op} =\|\bR -\mathrm{diag}(\bR )\|_{op} \leq C \|\bR \|_{op},
$$
we get
\begin{equation}\label{eq:f1}
    \begin{split}
\|(\B_t\B_t^\top - \I)^{\circ 3}\|_{op} &\leq C \left( \|(\Z_t\circ \H )^{\circ 3}\|_{op} +  \|\Z_t\|_{op}^2 \|\X_t^{O}\|_{op} + \|\Z_t\|_{op} \|\X_t^{O}\|^2_{op} + \|\X_t^{O}\|^3_{op}\right )\\
& \leq C \left( \|(\Z_t\circ \H )^{\circ 3}\|_{op} +  \|\Z_t\|_{op}^{1/2} \|\X_t^{O}\|_{op} + \|\X_t^{O}\|^2_{op}\right ),
\end{split}
\end{equation}
where the second step holds since $\opn{\X_t^O} \le 1$ and $\|\Z_t\|_{op} \leq C$ by \eqref{eq:assumptionZX}-\eqref{eq:assumptionZXbis}. Another application of \eqref{hadbound}
gives that
\begin{equation}\label{eq:f2}
\begin{split}
\|(\Z_t\circ \H )^{\circ 3}\|_{op} = \|(\Z_t\circ \H )^{\circ 2} \circ \Z_t\|_{op} &\leq \sqrt{n} \cdot \max_{i\neq j}\abs{(\Z_t)_{i,j}}^2 \cdot \|\Z_t\|_{op}\\
&\leq C\frac{\log d}{\sqrt{d}} \cdot \|\Z_t\|_{op} \leq C\frac{\log d}{\sqrt{d}} \cdot \|\Z_t\|^{1/2}_{op},
\end{split}
\end{equation}
where the second passage follows from Lemma \ref{lem_D_diagonal} and the last from $\|\Z_t\|_{op} \leq C$. 
By combining \eqref{eq:f1} and \eqref{eq:f2}, the proof of \eqref{eq:claimBB} for $\ell=3$ is complete.

To prove \eqref{eq:claimBB2}, define the following quantity
$$\Y := \sum_{l=3}^\infty c_\ell^2 (\B_t\B_t^\top - \I)^{\circ \ell}.$$ By definition of $f(\cdot)$ we have that 
$$
f(\B_t\B_t^\top) = \alpha \I + \B_t\B_t^{\top} +  \Y,
$$
which implies that
\begin{equation}\label{eq:clemma0}
\begin{split}
\big(f(\B_t\B_t^\top)\big)^{-1} &= (\alpha \I + \B_t\B_t^\top + \Y)^{-1} \\
&= (\I + \Y(\alpha \I + \B_t\B_t^\top)^{-1})^{-1}(\alpha \I + \B_t\B_t^\top)^{-1}\\
& = \bigg(\I +\sum_{k=1}^\infty(-1)^k (\Y(\alpha \I + \B_t\B_t^\top)^{-1})^k \bigg) (\alpha \I + \B_t\B_t^\top)^{-1}.
\end{split}    
\end{equation}
By definition \eqref{eq:defopnE}, we have that $\opn{\bE^t  }\le 1/2$ under assumptions \eqref{eq:assumptionZX} for sufficiently large $d$. Hence, by the result \eqref{eq:claimBB} we have just proved, $\|(\B_t\B_t^\top - \I)^{\circ \ell}\|_{op} \leq 1/2$,
which implies that $\sum_{\ell=3}^\infty c_\ell^2 \|(\B_t\B_t^\top - \I)^{\circ \ell}\|_{op}\leq \alpha/2$. Thus, we have
\begin{equation}\label{eq:0dot5bd}
\| \Y (\B_t\B_t^{\top} + \alpha \I)^{-1}\|_{op} \leq \| \Y\|_{op} \|(\B_t\B_t^{\top} + \alpha \I)^{-1}\|_{op} \leq \frac{\alpha}{2} \cdot \frac{1}{\alpha} \leq \frac{1}{2}.
\end{equation}
Therefore, we can conclude that
\begin{equation}\label{eq:clemma1}
    \begin{split}
    \|\big(f(\B_t\B_t^\top)\big)^{-1} - (\alpha \I + \B_t\B_t^\top)^{-1}\|_{op} &\leq \| (\alpha \I + \B_t\B_t^\top)^{-1}\|_{op} \cdot \sum\limits_{k=1}^{\infty}\| \Y(\alpha \I + \B_t\B_t^\top)^{-1}\|^k_{op}
    \\&\leq \frac{1}{\alpha} \cdot \frac{\| \Y(\alpha \I + \B_t\B_t^\top)^{-1}\|_{op}}{1 - \| \Y(\alpha \I + \B_t\B_t^\top)^{-1}\|_{op}} \\
    &\leq \frac{2}{\alpha} \cdot \| \Y\|_{op} \| (\alpha \I + \B_t\B_t^{\top})^{-1}\|_{op} \\ &\leq \frac{2}{\alpha^2} \cdot \|\Y\|_{op},
    \end{split}
\end{equation}
where the third inequality uses \eqref{eq:0dot5bd}. By bounding $\|\Y\|_{op}$ via \eqref{eq:claimBB}, the proof of \eqref{eq:claimBB2} is complete. 
\end{proof}

\begin{lemma}[Bound for the Schur product with $\A^\top \A$]\label{lemma:AAT_schur}
Assume that \eqref{eq:assumptionZX} holds, and let $\A_t$ be given by \eqref{eq:optA}. Then, we have that, with probability at least $1-1/d^2$, jointly for all $t\ge 0$ and $\ell\ge 2$, 
\begin{align}\label{eq:error}
    &\opn{\A_t^\top \A_t \circ (\B_t\B_t^\top - \I)^{\circ \ell}} \le  \|\bE^t  \|_{op}.
\end{align}
\end{lemma}

\begin{proof}[Proof of Lemma \ref{lemma:AAT_schur}] 
We have that
\begin{equation}\label{eq:Schurl2}
    \begin{split}
\|\A_t^\top \A_t \circ (\B_t\B_t^\top - \I)^{\circ \ell}\|_{op} &\leq \|\A_t^\top \A_t \circ (\B_t\B_t^\top - \I)^{\circ 2}\|_{op} \opn{(\B_t\B_t^\top - \I)^{\circ (\ell-2)}}\\
&\leq C ||\A_t^\top \A_t \circ (\B_t\B_t^\top - \I)^{\circ 2}\|_{op},        
    \end{split}
\end{equation}
where the first inequality uses \eqref{eq:bdRS} and the second inequality uses that $\opn{(\B_t\B_t^\top - \I)^{\circ (\ell-2)}}$ is uniformly bounded in $l$, which follows from \eqref{eq:intm3} and \eqref{eq:intm4}.

Let us now focus on bounding the RHS of \eqref{eq:Schurl2}.
An application of Lemma \ref{lemma:finv_bound} gives that
$$
\big(f(\B_t\B_t^{\top})\big)^{-1} = (\alpha \I + \B_t\B_t^{\top})^{-1} + \bE_1,
$$
where
$$
\|\bE\|_{op} \le  \opn{\bE^t  }.
$$
Hence, by using \eqref{eq:optA}, we get that
\begin{align}
    \A_t^\top \A_t &= ((\alpha \I + \B_t\B_t^\top)^{-1}\B_t + \bE_1^\top \B_t)(\B_t^\top(\alpha \I + \B_t\B_t^\top)^{-1} + \B_t^\top \bE_1) \notag\\ &= \B_t\B_t^\top (\alpha \I + \B_t\B_t^\top)^{-2} + \bE_1^\top \B_t\B_t^\top (\alpha \I + \B_t\B_t^\top)^{-1} + (\alpha \I + \B_t\B_t^\top)^{-1} \B_t\B_t^{\top} \bE_1 + \bE_1^\top \B_t\B_t^{\top} \bE_1, \label{2.18.1}
\end{align}
where we rearranged the first term in \eqref{2.18.1} using that $\B_t\B_t^\top$ and $(\alpha \I + \B_t\B_t^\top)^{-1}$ commute. By using the assumptions \eqref{eq:assumptionZX}, we have that
$$
\|\B_t\B_t^\top\|_{op} \leq C, \qquad \|\bE_1\|_{op}  \leq 1/2, \qquad \|(\alpha \I + \B_t\B_t^\top)^{-1}\|_{op} \leq \frac{1}{\alpha}.
$$
Hence, we can upper bound the operator norm of the last three terms in \eqref{2.18.1} as
\begin{equation}\label{2.18.2}
  \opn{\bE_1^\top \B_t\B_t^\top (\alpha \I + \B_t\B_t^\top)^{-1} + (\alpha \I + \B_t\B_t^\top)^{-1} \B_t\B_t^{\top} \bE_1 + \bE_1^\top \B_t\B_t^{\top} \bE_1} \leq C\|\bE_1\|_{op} .
\end{equation}
Let us now take a closer look at the first term in \eqref{2.18.1}. Recall that
\begin{align*}
    \B_t\B_t^\top = \U \Lam_t \U^\top +  \X_t.
\end{align*}
As the operator norm is sub-multiplicative, we have that
\begin{equation}\label{eq:A1}
    \| \X_t \cdot (\alpha \I + \B_t\B_t^\top)^{-2}\|_{op} \leq C \|\X_t\|_{op}.
\end{equation}
Furthermore, 
\begin{equation}
    \begin{split}
    \U \Lam_t \U ^\top (\alpha \I + \U \Lam_t \U ^\top + \X_t)^{-2} &=  \U \Lam_t \U ^\top \big((\I + \X_t(\alpha \I + \U \Lam_t \U ^\top)^{-1})(\alpha \I + \U \Lam_t \U ^\top)\big)^{-2} \\
&=  \U \Lam_t \U ^\top \T_1^{-1}\T_2^{-1}\T_1^{-1}\T_2^{-1}    , \label{2.18.3}
    \end{split}
\end{equation}
where we have defined
\begin{equation*}
\T_1 =    \alpha \I + \U \Lam_t \U ^\top, \qquad \T_2=  \I + \X_t(\alpha \I + \U \Lam_t \U ^\top)^{-1}.
\end{equation*}
By expanding $\T_2^{-1}$ as in \eqref{eq:clemma0}-\eqref{eq:clemma1}, we get
\begin{equation*}
    \|\T_2^{-1} - \I\|_{op} \leq C \|\X_t\|_{op},
\end{equation*}
or equivalently
\begin{equation*}
\T_2^{-1} = \I + \bE_2,  
\end{equation*}
with $\|\bE_2\|_{op} \leq C\|\X_t\|_{op}$. 
In this view, looking at \eqref{2.18.3} we have
\begin{align*}
    \U\Lam_t \U^\top \T_1^{-1}\T_2^{-1}\T_1^{-1}\T_2^{-1} &=  \U\Lam_t \bB{}\U^\top \T_1^{-1} (\I + \bE_2) \T_1^{-1} (\I + \bE_2).
\end{align*}
All the terms which involve $\bE_2$ can be controlled. We provide the analysis for two terms of different nature, the rest follows from similar arguments. As $\|\T_1^{-1}\|_{op} \leq 1/\alpha$ and $\|\Lam_t\|_{op} \leq C$, we have that
\begin{align*}
    &\|\U\Lam_t \U^\top \T_1^{-1}\bE_2\T_1^{-1}\bE_2\|_{op} \leq \|\T_1^{-1}\|^2_{op}\|\bE_2\|^2_{op} \leq \frac{C}{\alpha^2}  \|\X_t\|_{op}^2\leq \frac{C}{\alpha^2}  \|\X_t\|_{op}, \\
    &\|\U\Lam_t \U^\top \T_1^{-1}\I \T_1^{-1}\bE_2\|_{op} \leq \|\T_1^{-1}\|^2_{op}\|\bE_2\|_{op} \leq \frac{C}{\alpha^2}  \|\X_t\|_{op},
\end{align*}
where we have also used that $\opn{\X_t}$ is bounded via assumptions \eqref{eq:assumptionZX}.
Furthermore, a simple manipulation gives
$$
\U\Lam_t \U^\top \T_1^{-2} = \U\Lam_t \U^\top (\alpha \I + \U \Lam_t \U^\top)^{-2} = \U \Lam_t (\alpha \I + \Lam_t)^{-2} \U^\top = \U \phi(\Lam_t) \U^{\top},
$$
where $\phi(x) = \frac{x}{(\alpha + x)^2}$. As a result,
\begin{equation*}
    \opn{\U\Lam_t \U^\top \T_1^{-1}\T_2^{-1}\T_1^{-1}\T_2^{-1}-\U \phi(\Lam_t) \U^{\top}}\le C\opn{\X_t},
\end{equation*}
which implies that
\begin{align}\label{eq:A2}
    \|\B_t\B_t^\top(\alpha \I + \B_t\B_t^{\top})^{-2} - \U\phi(\Lam_t)\U^\top\|_{op} \leq C\|\X_t\|_{op}.
\end{align}
By combining \eqref{2.18.1}, \eqref{2.18.2} and \eqref{eq:A2}, we have that
\begin{equation}
    \|\A_t^\top \A_t - \U\phi(\Lam_t)\U^\top\|_{op} \leq C\big(\|\X_t\|_{op}+\opn{\bE _1}\big).
\end{equation}
At this point, we are ready to analyze the operator norm of
$\|\A_t^\top \A_t \circ (\B_t\B_t^\top-\I)^{\circ 2}\|_{op}$:
\begin{align}
    \A_t^\top \A_t \circ (\B_t\B_t^\top-\I)^{\circ 2} &= (\U\phi(\Lam_t)\U^\top + \bE _3) \circ (\U(\Lam_t - \I)\U^\top + \X_t)^{\circ 2} \circ \H\notag \\
    & = (\U\phi(\Lam_t )\U^\top + \bE _3) \circ ((\U(\Lam_t - \I)\U^\top)^{\circ 2} + \X_t^{\circ 2} + 2(\U(\Lam_t - \I)\U^\top) \circ \X_t ) \circ \H, \label{2.18.4}
\end{align}
where we have defined $\H:=\1\1^\top - \I$ and $\|\bE _3\|_{op}\leq C\big(\|\X_t\|_{op}+\opn{\bE_1}\big)$. We now decompose the quantity into three terms:
\begin{equation*}
    \A_t^\top \A_t \circ (\B_t\B_t^\top-\I)^{\circ 2} = \S_1 + \S_2 + \S_3,
\end{equation*}
where
\begin{equation*}
    \begin{split}
        \S_1 & = (\U\phi(\Lam_t)\U^\top \circ \U(\Lam_t - \I)\U^\top \circ H) \circ \U(\Lam_t - \I)\U^\top,\\
        \S_2 & = H \circ \bE _3 \circ ((\U(\Lam_t - \I)\U^\top)^{\circ 2} + \X_t^{\circ 2} + 2(\U(\Lam_t - \I)\U^\top) \circ \X_t),\\
        \S_3 & = H \circ \U\phi(\Lam_t)\U^\top \circ (\X_t^{\circ 2} + 2(\U(\Lam_t - \I)\U^\top) \circ \X_t).
    \end{split}
\end{equation*}
We proceed to bound each of these terms separately.

We start with $\S_1$. As $\phi(x)$ is differentiable for $x\geq 0$, the derivative of $\phi(x)$ is bounded for any compact interval $I\subseteq \mathbb{R}_{+}$. Hence, $\phi(x)$ is locally Lipschitz on $I$  with Lipschitz constant $C_I > 0$, which implies that
$$
|\phi(x)-\phi(1)|=\left|\phi(x) - \frac{1}{(1+\alpha)^2}\right| \leq C_I|x-1|.
$$
By assumption \eqref{eq:assumptionZX}, we have that $\Lam_t \succ 0$ and $\|\Lam_t\|_{op} \leq C$, hence
\begin{equation}\label{controlphi}
    \left\|\U\phi(\Lam_t)\U^\top - \frac{1}{(1+\alpha)^2}\I\right\|_{op} \leq C_I \cdot \|\Z_t\|_{op}.
\end{equation}
Hence, an application of Lemma \ref{lem_D_diagonal} gives that, with probability at least $1-1/d^2$,
\begin{equation}\label{controlphi2}
\sup_{t\geq 0}m\left(\U\phi(\Lam_t)\U^\top - \frac{1}{(1+\alpha)^2}\I\right) \leq c\sqrt{\frac{\log d}{d}},
\end{equation}
where $c>0$ is a universal constant. 
Another application of Lemma \ref{lem_D_diagonal} also gives that, with the same probability,
\begin{equation}\label{controlphi3}
\sup_{t\geq 0}m\left(\U(\Lam_t-\I)\U^\top\right) \leq c\sqrt{\frac{\log d}{d}}.
\end{equation}
As a result, we obtain the bound 
\begin{align}\label{eq:bdS1}
    \|\S_1\|_{op}&= \|([\U\phi(\Lam_t)\U^\top - 1/(1+\alpha)^2\I] \circ \U(\Lam_t - \I)\U^\top \circ \H) \circ \U(\Lam_t - \I)\U^\top\|_{op} \leq C \frac{\log d}{\sqrt{d}} \|\Z_t\|_{op}.
\end{align}
Here, the first equality is due to the fact that we are taking the Hadamard product with the matrix $\H$ which has $0$ on the diagonal, hence we can add multiples of the identity to $\U\phi(\Lam_t)\U^\top$; and the second inequality uses \eqref{hadbound} with $\bR=[\U\phi(\Lam_t)\U^\top-1/(1+\alpha)^2\I] \circ \U(\Lam_t - \I)\U^\top \circ \H$ and $\S=\U(\Lam_t - \I)\U^\top$ in combination with \eqref{controlphi2}-\eqref{controlphi3}. 

Next, we bound $\opn{\S_2}$. We inspect the terms appearing in the expression for $\S_2$ one by one. First note that we can omit $\H$ in the expression since, for any matrix $\bR$
 \begin{equation}\label{substractdiagbound}
     \|\bR\circ \H\|_{op} \leq C\|\bR\|_{op}.
 \end{equation}
 Hence, by using \eqref{eq:bdRS}, we get
\begin{align*}
    &\|\H \circ  \bE _3 \circ ((\U(\Lam_t - I)\U^\top)^{\circ 2}\|_{op} \leq C\|\bE _3\|_{op}\|\Z_t\|_{op}^2\\
    & \|\H\circ  \bE _3 \circ \X_t^{\circ 2}\|_{op}  \leq C\|\bE _3\|_{op}\|\X_t\|_{op}^2 \\
    & \|\H\circ \bE _3 \circ 2(\U(\Lam_t - I)\U^\top) \circ X)\|_{op} \leq C\|\bE _3\|_{op}\|\X_t\|_{op}\|\Z_t\|_{op},
\end{align*}
which leads to the bound 
\begin{equation}\label{eq:bdS2}
    \opn{\S_2}\le C\|\bE _3\|_{op}\left(\|\X_t\|_{op}^2 + \|\Z_t\|_{op}^2 + \|\X_t\|_{op}\|\Z_t\|_{op}\right).
\end{equation}

Finally, we bound $\opn{\S_3}$.
Consider the term
$$
\|[\H \circ \U\phi(\Lam_t)\U^\top \circ 2(\U(\Lam_t - I)\U^\top] \circ \X_t\|_{op}.
$$
Then, by using \eqref{substractdiagbound} and \eqref{controlphi}, we have
\begin{equation}\label{eq:HUUT}
\|\H \circ \U\phi(\Lam_t)\U^\top\|_{op} = \left\|\H \circ [\U\phi(\Lam_t)\U^\top-\frac{1}{(1+\alpha)^2}I]\right\|_{op} \leq C\left\| \U\phi(\Lam_t)\U^\top-\frac{1}{(1+\alpha)^2}I\right\|_{op} \leq C \|\Z_t\|_{op}.
\end{equation}
Hence, in conjunction with \eqref{eq:bdRS}, we get
$$
\|\H \circ \U\phi(\Lam_t)\U^\top \circ 2\U(\Lam_t - I)\U^\top\|_{op} \leq C\cdot \|\Z_t\|^2_{op},
$$
which invoking \eqref{eq:bdRS} one more time gives
$$
\|[\H \circ \U\phi(\Lam_t)\U^\top \circ 2(\U(\Lam_t - I)\U^\top] \circ \X_t\|_{op} \leq C \|\Z_t\|^2_{op} \|\X_t\|_{op}.
$$
Furthermore, by combining \eqref{eq:bdRS} and \eqref{eq:HUUT},  we get
$$
\|[\H \circ \Lam\phi(\Lam_t)\Lam^\top] \circ \X_t^{\circ 2}\|_{op} \leq C\|\Z_t\|_{op} \|\X_t\|_{op}^2.
$$
Thus,
\begin{equation}\label{eq:bdS3}
    \opn{\S_3}\le C(\|\Z_t\|_{op}^2\|\X_t\|_{op}+\|\Z_t\|_{op}\|\X_t\|_{op}^2).
\end{equation}

Recall that, from assumptions \eqref{eq:assumptionZX}-\eqref{eq:assumptionZXbis}, $\opn{\X_t}, \opn{\Z_t}\le C$. Then, by combining the bounds in \eqref{eq:bdS1}, \eqref{eq:bdS2} and \eqref{eq:bdS3}, the desired result readily follows. 
\end{proof}


By exploiting the above lemmas, we are able to make the following approximation for the gradient.

\begin{lemma}[Gradient approximation]\label{lemma:gradient_approx} 
Assume that \eqref{eq:assumptionZX} holds, and let $\nabla_{\B_t}$ be given by \eqref{eq:defnablaBt}. Further define $\gamma = 1 + \alpha$ and $F(x)= \frac{1+x}{(\gamma+x)^{2}}$. Then, for all sufficiently large $n$, with probability $1-1/d^2$, jointly for all $t\ge 0$, 

\begin{equation} \label{eq:grad_app_final_0_statement}
\begin{split}
      &\opn{  \frac{1}{2}\nabla_{\B_t}\B_t^\top   +\alpha F(\Z_t)-\alpha\mathrm{Diag}\left(F(\Z_t)\right)(\I+\Z_t) -\frac{2\alpha}{\gamma^3}\X_t^O - \frac{\alpha}{\gamma^2}\X_t^D} \le \|\bE^t  \|_{op}.
\end{split}
\end{equation}
\end{lemma}

\begin{proof}[Proof of Lemma \ref{lemma:gradient_approx}]
We start by showing that, with probability $1-1/d^2$, jointly for all $t\ge 0$, 
\begin{equation}\label{eq:grad_app_final_0}
\begin{split}
      &\opn{  \frac{1}{2}\nabla_{\B_t} + \alpha (\alpha \I + \B_t\B_t^\top)^{-2}\B_t - 
        \alpha \mathrm{Diag} \left(( \alpha \I + \B_t \B_t^\top)^{-2}  (\B_t\B_t^\top)  \right)\B_t }\le \|\bE^t  \|_{op}.
\end{split}
\end{equation}
Let us first consider the term $\nabla^1_{\B_t}$, which can be equivalently expressed as 
    $$ \nabla^1_{\B_t} = 2\big( -\A_t^\top + \mathrm{Diag}(\B_t \A_t) \B_t  + \T \B_t  - \mathrm{Diag}(\T(\B_t\B_t^\top)) \B_t\big), $$
where $\T = \A_t^\top \A_t - \mathrm{Diag}(\A_t^\top \A_t)$.
It is then easy to verify that
\begin{equation}\label{gradre1}
    \frac{1}{2}\nabla^1_{\B_t}
    =  -\A_t^\top + \A_t^\top \A_t\B_t +  \textrm{Diag}(\B_t \A_t) \B_t   - \mathrm{Diag}(\A_t^\top \A_t\B_t\B_t^\top) \B_t.
\end{equation}
Using Lemma \ref{lemma:finv_bound}, we get
\begin{equation}\label{ATAre1}
    \begin{split}
        \A_t^\top \A_t = ((\alpha \I + \B_t\B_t^\top)^{-1} + \bE_1) \B_t\B_t^\top ((\alpha \I + \B_t\B_t^\top)^{-1} + \bE_1) ,
    \end{split}
\end{equation}
where $\|\bE_1\|_{op} \le \|\bE^t  \|_{op}$. It follows from \eqref{eq:assumptionZX} that $\|\B_t\B_t^\top\|_{op}\le C$. Hence, using that $\B_t\B_t^\top$ and $(\alpha \I + \B_t\B_t^\top)$ commute in conjunction with $\|(\alpha \I + \B_t\B_t^\top)^{-1}\|_{op} \le 1/\alpha$ we get
\begin{equation}\label{ATAre2}
    \begin{split}
        \A_t^\top \A_t = \B_t\B_t^\top (\alpha \I + \B_t\B_t^\top)^{-2} + \bE _2,
    \end{split}
\end{equation}
where $\|\bE _2\|_{op} \le  \|\bE^t  \|_{op}$. Noting that
$
\frac{1}{\alpha + x} - \frac{\alpha}{(\alpha + x)^2} = \frac{x}{(\alpha + x)^2}
$ and using the spectral theorem for the symmetric matrix $ \B_t\B_t^\top$, we can further rewrite \eqref{ATAre2} as
\begin{equation}\label{ATAre3}
    \begin{split}
        \A_t^\top \A_t = (\alpha \I +  \B_t\B_t^\top)^{-1} - \alpha (\alpha \I +  \B_t\B_t^\top)^{-2} + \bE _2.
    \end{split}
\end{equation}
With similar arguments, by Lemma \ref{lemma:finv_bound}, we can write
\begin{equation}\label{BAre1}
    \begin{split}
        \B_t\A_t = \B_t\B_t^\top (\alpha \I + \B_t\B_t^\top)^{-1} + \bE _3,
    \end{split}
\end{equation}
where $\|\bE _3\|_{op}\leq \|\bE^t  \|_{op}$. Noting that $1-\frac{\alpha}{\alpha + x} = \frac{x}{\alpha + x}$, again by the spectral theorem for $ \B_t\B_t^\top$, we get
\begin{equation}\label{BAre2}
    \begin{split}
        \B_t\A_t = \I - \alpha(\alpha \I +  \B_t\B_t^\top)^{-1} + \bE_3,
    \end{split}
\end{equation}
and, consequently, we obtain
\begin{equation}\label{diagBABre1}
    \begin{split}
        \mathrm{Diag}(\B_t \A_t)\B_t = \B_t - \alpha \mathrm{Diag}((\alpha \I + \B_t\B_t^{\top})^{-1})\B_t + \bE_4,
    \end{split}
\end{equation}
where $\|\bE_4\|_{op}\leq \|\bE^t  \|_{op}$. Using \eqref{ATAre3} and  $1-\frac{\alpha}{\alpha+x} = \frac{1}{x+\alpha}$, we get
\begin{equation}\label{diagATABBTre2}
    \begin{split}
        \mathrm{Diag}(\A_t^\top \A_t \B_t \B_t^\top)\B_t &= \mathrm{Diag}((\alpha \I + \B_t\B_t^\top)^{-1}\B_t\B_t^\top)\B_t - \alpha \mathrm{Diag}((\alpha \I + \B_t\B_t^\top)^{-2}\B_t\B_t^\top)\B_t  + \bE_5\\
        &= \B_t - \alpha \mathrm{Diag}((\alpha \I + \B_t\B_t^\top)^{-1})\B_t - \alpha \mathrm{Diag}((\alpha \I + \B_t\B_t^\top)^{-2}\B_t\B_t^\top)\B_t  + \bE_5,
    \end{split}
\end{equation}
where $\|\bE_5\|_{op} \le \|\bE^t  \|_{op}$.

With this in mind, we get back to \eqref{gradre1}. Combining the results of \eqref{ATAre3}, \eqref{diagBABre1} and \eqref{diagATABBTre2} we get
\begin{equation}\label{gradapproxre1}
\begin{split}
    \nabla_{\B_t}^1 &= \underbrace{-(\alpha \I + \B_t\B_t^\top)^{-1}\B_t}_{-\A_t^\top} + \underbrace{(\alpha \I + \B_t\B_t^\top)^{-1}\B_t - \alpha (\alpha + \B_t \B_t^\top)^{-2}\B_t}_{\A_t^\top \A_t \B_t} 
+ \underbrace{\B_t - \alpha \mathrm{Diag}((\alpha \I + \B_t\B_t^{\top})^{-1})\B_t}_{\mathrm{Diag}(\B_t\A_t)\B_t} \\ & \underbrace{-\B_t + \alpha \mathrm{Diag}((\alpha \I + \B_t\B_t^\top)^{-1})\B_t + \alpha \mathrm{Diag}((\alpha \I + \B_t\B_t^\top)^{-2}\B_t\B_t^\top)\B_t}_{-\mathrm{Diag}(\A_t^\top \A_t \B_t \B_t^\top)\B_t} + \bE_6 \\
&=- \alpha(\alpha \I + \B_t\B_t^\top)^{-2}\B_t  + \alpha \mathrm{Diag}((\alpha \I + \B_t\B_t^\top)^{-2}\B_t\B_t^\top)\B_t + \bE_6,
\end{split}
\end{equation}
where $\|\bE_6\|_{op} \leq \|\bE^t  \|_{op}$.

Let us now analyze the second part of the gradient which involves terms of the form below for $\ell \ge 3$:
\begin{equation*} \label{gradient_quadratic_part}
      \nabla_{\B_t}^{2,k,\ell} := c_\ell^2 \cdot \ell \cdot \sum_{j \neq k} \langle \a_k, \a_j \rangle \langle \b_k, \b_j \rangle^{(\ell-1)} \J_k \b_j.
\end{equation*}
Now, from the fact that
    $$ \J_k = \I - \b_k \b_k^\top,$$
we can write
\begin{equation} \label{gradient_formula_approxs}
     c_\ell^2 \cdot \ell \cdot \sum_{j \neq k} \langle \a_k, \a_j \rangle \langle \b_k, \b_j \rangle^{(\ell-1)} \J_k \b_j =   c_\ell^2 \cdot \ell \cdot \sum_{j \neq k} \langle \a_k, \a_j \rangle \langle \b_k, \b_j \rangle^{(\ell-1)} (\b_j - \langle \b_k, \b_j \rangle \b_k).
\end{equation}
The second term of the RHS gives the following contribution to the $\B_t$ update
    $$\text{Diag}(\A_t^\top \A_t(\B_t\B_t^\top-\I)^{\circ \ell})\B_t.$$
By recalling that $\|\A_t^\top \A_t\|_{op} \le C$ and $
\|\B_t\|_{op} \le C$, we have
\begin{equation}\label{eq:diagdc}
    \begin{split}
        \|\text{Diag}(\A_t^\top \A_t(\B_t\B_t^\top -\I)^{\circ \ell})\B_t\|_{op} &\leq C \|\A_t^\top \A_t(\B_t\B_t^\top -\I)^{\circ \ell}\|_{op} \|\B_t\|_{op} \leq C \|(\B_t\B_t^\top -\I)^{\circ \ell}\|_{op}.
    \end{split}
\end{equation}
Now, for $\ell < 5$, we upper bound the RHS of \eqref{eq:diagdc} via Lemma \ref{lemma:finv_bound}, which gives that 
\begin{equation}
    \begin{split}
        \|\text{Diag}(\A_t^\top \A_t(\B_t\B_t^\top -\I)^{\circ \ell})\B_t\|_{op} &\leq C \|\bE^t  \|_{op}.
    \end{split}
\end{equation}
Furthermore, if we follow passages analogous to \eqref{eq:intm1}-\eqref{eq:intm2} (the only difference being that we exchange the roles of the Hadamard powers $3$ and $\ell-3$), we have that, with probability at least $1-1/d^2$, jointly for all $t\ge 0$ and $\ell\ge 5$,
\begin{equation}\label{eq:dcm} 
    \begin{split}
        \|\text{Diag}(\A_t^\top \A_t(\B_t\B_t^\top -\I)^{\circ \ell})\B_t\|_{op} &\leq C \sqrt{n}\|\bE^t  \|_{op} \left(\frac{\mathrm{poly}(\log d)}{d}\right)^{(\ell-3)/2}\leq C \|\bE^t  \|_{op} \left(\frac{\mathrm{poly}(\log d)}{d}\right)^{(\ell-4)/2},
    \end{split}
\end{equation}
for sufficiently large $d$.

Define the following quantity:
\begin{align} \label{U_def_app}
    \Y = (\A_t^\top \A_t)\circ (\B_t\B_t^\top - \I)^{\circ (\ell-1)}.
\end{align}
In this view, the first term in \eqref{gradient_formula_approxs} can be written as $\Y\B_t$.
For $l < 5$, by Lemma \ref{lemma:AAT_schur} we have that $\|\Y\|_{op} \leq  \|\bE^t  \|_{op}$, hence $\|\Y\B_t\|_{op}\leq C\|\bE^t  \|_{op}$ as $\|\B_t\|_{op} \le C$. Furthermore, with probability at least $1-1/d^2$, jointly for all $t\ge 0$ and $\ell \geq 5$, we have
\begin{equation}\label{eq:pt2}
    \begin{split}
        \|\Y\B_t\|_{op} \leq C\|\Y\|_{op} 
        &= C\sqrt{n}\|(\A_t^\top \A_t)\circ (\B_t\B_t^\top - \I)^{\circ     2}\|_{op} \max_{i,j}|(\B_t\B_t^\top - \I)_{i, j}|^{\ell-3} \\ 
        &\leq  \sqrt{n}\|\bE^t  \|_{op} \max_{i,j}|(\B_t\B_t^\top -         \I)_{i, j}|^{\ell-3}\\ 
        &\leq  \sqrt{n}\|\bE^t  \|_{op}             \left[(C + C_X)^{\ell-3}\left(\frac{\mathrm{poly}(\log d)}{d}\right)^{(\ell-3)/2} \right]\\
        &\leq (C+C_X)^{\ell-3} \|\bE^t  \|_{op}  \left(\frac{\mathrm{poly}(\log d)}{d}\right)^{(\ell-4)/2}. 
    \end{split}
\end{equation}
Here, in the second line we use Lemma \ref{lemma:AAT_schur}; and in the third line we bound the off-diagonal entries of $\X_t$ via \eqref{eq:assumptionZX} and the off-diagonal entries of $\Z_t$ via Lemma \ref{lem_D_diagonal}. 
Hence, by combining \eqref{eq:dcm} and \eqref{eq:pt2}, we conclude that
\begin{equation}
   \opn{ \nabla_{\B_t}^2} \leq C \|\bE^t  \|_{op} +  \|\bE^t  \|_{op} \sum_{\ell=5}^{\infty} (C + C_X)^{\ell-3}c_\ell^2 \,\ell\, \left(\frac{\mathrm{poly}(\log d)}{\sqrt{d}}\right)^{\ell-4} \leq C \|\bE^t  \|_{op},
\end{equation}
where we used that the series $\sum_{\ell=5}^{\infty} (C + C_X)^{\ell-3}c_\ell^2\, \ell \, \left(\frac{(\mathrm{poly}(\log d)}{\sqrt{d}}\right)^{\ell-4}$ converges to a finite value for all sufficiently large $d$, since $(C+C_X) \frac{\mathrm{poly}(\log d)}{\sqrt{d}}<1$. This finishes the proof of \eqref{eq:grad_app_final_0}.

We now further analyse the gradient in  \eqref{eq:grad_app_final_0}.
Defining  $F(x)= \frac{1+x}{(\gamma+x)^{2}}$, with $\gamma=1+\alpha$, we can write 
\begin{equation}\label{eq:F1}
    \frac{1}{2} \nabla_{\B_t} \B_t^\top =  - \alpha F(\Z_t+\X_t)+\alpha\mathrm{Diag}\left(F(\Z_t+\X_t))\right)+\alpha\mathrm{Diag}\left(F(\Z_t+\X_t))\right)(\Z_t+\X_t) + \bE^t  .
\end{equation}
By a slight abuse of notation, we will denote by $F^{(l)}(0)$ the $l$-th derivative of the unidimensional function $F(x)=\frac{1+x}{(\gamma+x)^2}$ computed at $x=0$. 
Here, $F(\Z_t+\X_t)$ is defined by the spectral theorem (note that indeed $\Z_t+\X_t=\B_t\B_t^\top-\I$ is symmetric). 

We will now compute the error we incur if in \eqref{eq:F1} we replace $F(\X_t+\Z_t)$ by $F(\Z_t)$.
We first consider the case when $\opn{\Z_t} >\frac{\gamma}{3}$.
In this case, we have that
\begin{equation}\label{eq:Fbound1}
\left\lVert F(\Z_t+\X_t)-F(\Z_t) - F^{(1)}(0) \X_t \right\rVert_{op}\leq C \left\lVert \X_t \right\rVert_{op} \leq C \left\lVert \Z_t \right\rVert_{op}\left\lVert \X_t \right\rVert_{op}.
\end{equation}
Here, the second inequality trivially holds since $\opn{\Z_t} >\frac{\gamma}{3}$.
To prove the first inequality, let $DF$ be the derivative of the matrix-valued function $F(\M) = (\I+\M)(\gamma \I+\M)^{-2}$. Then, by evaluating this derivative for $\M=\Z_t$ in the direction of $\X_t$, we obtain
\begin{equation}\label{eq:DFZX}
    DF(\Z_t)\,\X_t = -(\I+\Z_t)(\gamma \I+\Z_t)^{-1}\X_t(\gamma \I+\Z_t)^{-2}  -(\I+\Z_t)(\gamma \I+\Z_t)^{-2}\X_t(\gamma \I+\Z_t)^{-1}  + \X_t(\gamma \I+\Z_t)^{-2}.
\end{equation}
To verify this expression we first note that the derivative of the function $G(\M) = \M^{-1}$ in the direction of $\X$ is given by $DG(\M)\X = -\M^{-1}\X\M^{-1}$. Now, \eqref{eq:DFZX} easily follows from the product rule applied to $F(\Z) = (\I+\Z)(\gamma \I+\Z)^{-1}(\gamma \I+\Z)^{-1}$.
By the assumptions in \eqref{eq:assumptionZX}, we have that $\Z_t, (\gamma \I+\Z_t)^{-1}$ are uniformly bounded, hence the map $DF$ is uniformly bounded as well. This implies that
\begin{equation*}
    \left\lVert F(\Z_t+\X_t)-F(\Z_t) \right\rVert_{op}\leq C \left\lVert \X_t \right\rVert_{op}.
\end{equation*}
As $\opn{F^{(1)}(0)\X_t} \leq C \opn {\X_t}$, we readily obtain \eqref{eq:Fbound1}.

Now we consider the case where  $\left\lVert \Z_t \right\rVert_{op}\leq \frac{\gamma}{3}$. First note that, by \eqref{eq:assumptionZX}, $\left\lVert \X_t \right\rVert_{op}\leq \frac{\gamma}{3}$. Hence, 
$$F(\Z_t+\X_t)=\sum\limits_{\ell=0}^{\infty}F^{(\ell)}(0) \frac{(\Z_t+\X_t)^{\ell}}{{\ell!}}.$$
The series above converges absolutely since $F^{(l)}(0)$ scales as $\frac{\ell!}{\gamma^{\ell}}\textrm{poly}(\ell)$. To see this, first we note that, if $h(x) = \frac{1}{(\gamma + x)^2}$, then $h^{(\ell)}(0) =(-1)^\ell (l+1)! \frac{1}{\gamma^{\ell+2}}$.
Thus, by the product rule, $F^{(\ell)}(0) = (-1)^\ell (\ell+1)! \frac{1}{\gamma^{\ell+2}} + (-1)^{\ell-1} \ell! \frac{1}{\gamma^{\ell+1}}$ which has the desired asymptotic behaviour.
Expanding the brackets and applying the triangle inequality  yields
\begin{align*}
    \opn{F(\Z_t+\X_t)- \sum\limits_{\ell=0}^{\infty}F^{(\ell)}(0) \frac{\Z_t^{\ell}}{\ell!} - F^{(1)}(0) \X_t } \leq \sum\limits_{\ell=2}^{\infty}F^{(\ell)}(0) \frac{\opn {\X_t}^{\ell}}{\ell!}
     + \sum\limits_{\ell=2}^{\infty}F^{(\ell)}(0) \frac{1}{{\ell!}} \sum_{i=1}^{\ell-1} {\binom{\ell}{i}} \opn{\Z_t}^i \opn{\X_t}^{\ell-i}.
\end{align*}
As $\opn{\Z_t},\opn{\X_t} \leq \frac{\gamma}{3}$, we have
$$\sum\limits_{\ell=2}^{\infty}F^{(\ell)}(0) \frac{\opn {\X_t}^{\ell}}{\ell!}\leq \opn{\X_t}^2\sum\limits_{\ell=2}^{\infty}F^{(\ell)}(0)  \Big(\frac{\gamma}{3}\Big)^{\ell-2}\frac{1}{{\ell!}} \leq C \opn{\X_t}^2,$$
and
\begin{align*}
   \sum\limits_{\ell=2}^{\infty}F^{(\ell)}(0) \frac{1}{{\ell!}} \sum_{i=1}^{\ell-1} {\binom{\ell}{i}} \opn{\Z_t}^i \opn{\X_t}^{\ell-i} &\leq \sum\limits_{\ell=2}^{\infty}F^{(\ell)}(0) \frac{1}{{\ell!}}2^l\Big(\frac{\gamma}{3}\Big)^{\ell-2} \opn{\Z_t}\opn{\X_t} 
    \leq C \opn{\Z_t}\opn{\X_t}.
    \end{align*}
By combining the last three expressions and using that 
$$
F(\Z_t) = \sum\limits_{\ell=0}^{\infty}F^{(\ell)}(0) \frac{\Z_t^{\ell}}{\ell!},
$$
we obtain
\begin{equation}\label{eq:Ftaylor}
    \norm{F(\X_t+\Z_t)-F(\Z_t)-F^{(1)}(0)\X_t}_{op}  \leq C\left( \opn{\X_t}\opn{\Z_t} +\opn{\X_t}^2 \right).
\end{equation}
As the map $DF$ is uniformly bounded, we have
\begin{equation}\label{eq:Ftaylor2}
    \norm{F(\Z_t)-F(0)\I}_{op}\leq C\opn{\Z_t}.
\end{equation}
By combining \eqref{eq:Ftaylor}, \eqref{eq:Ftaylor2} and \eqref{eq:F1}, we obtain 
\begin{equation}\label{eq:F2}
  \frac{1}{2}  \nabla_{\B_t}\B_t^\top = -\alpha F(\Z_t)+\alpha\mathrm{Diag}\left(F(\Z_t)\right)(\I+\Z_t) -\alpha F^{(1)}(0)\X_t+ \alpha \mathrm{Diag}\left(\X_tF^{(1)}(0)\right)+ \alpha \X_tF(0) + \bE^t  .
\end{equation}
Using that $F(0) = \frac{1}{\gamma^2}$ and $F^{(1)}(0) = \frac{1}{\gamma^2}(1-\frac{2}{\gamma})$,
we finally obtain
\begin{equation}\label{eq:F3}
    \frac{1}{2} \nabla_{\B_t}\B_t^\top = -\alpha F(\Z_t)+\alpha\mathrm{Diag}\left(F(\Z_t)\right)(\I+\Z_t) +\frac{2\alpha}{\gamma^3}\X_t^O + \frac{\alpha}{\gamma^2}\X_t^D +\bE^t  ,
\end{equation}
which concludes the proof.
\end{proof}

\noindent Now let us return to the update equation of $\B_t\B_t^\top$ during the gradient step
\begin{equation}\label{eq:BBT0}
    \B_{t}'\B_{t}'^\top = (\B_t - \eta \nabla_{\B_t})(\B_t - \eta \nabla_{\B_t})^\top = \B_t\B_t^\top - \eta \cdot \nabla_{\B_t} \B_t^\top -\eta \cdot \B_t (\nabla_{\B_t})^\top + \eta^2 \cdot \nabla_{\B_t} (\nabla_{\B_t})^\top.
\end{equation}
Note that we can control the terms $\B_t (\nabla_{\B_t})^\top$ and $\nabla_{\B_t} \B_t^\top$ via Lemma \ref{lemma:gradient_approx}. In this view, it remains to argue that the contribution of the term $\eta^2 \cdot \nabla_{\B_t} (\nabla_{\B_t})^\top$ and of the projection step are of order $\eta \opn{\bE^t  }$. For convenience of the upcoming lemmas we define the following quantity:
 \begin{equation}\label{eq:deftilde}
 \widetilde{\nabla}_{\B_t}:= 2\left(-\alpha (\alpha \I + \B_t\B_t^\top)^{-2}\B_t +
        \alpha \mathrm{Diag} \left(( \alpha \I + \B_t \B_t^\top)^{-2}  (\B_t\B_t^\top)\right)\B_t\right).
        \end{equation}
 \begin{lemma}\label{lemma:nablanablaT} Assume that \eqref{eq:assumptionZX} holds, and let $\nabla_{\B_t}$ be given by \eqref{eq:defnablaBt} with $\eta\leq C/\sqrt{d}$. Then, for all sufficiently large $n$, with probability $1-1/d^2$, jointly for all $t\ge 0$:
 $$
 \eta^2 \opn{ \nabla_{\B_t} (\nabla_{\B_t})^\top} \le \eta \opn{\bE^t  }.
 $$
 \end{lemma}
 
\begin{proof}[Proof of Lemma \ref{lemma:nablanablaT}]  We start by showing that 
\begin{equation}\label{lemma:widetildenabla_bound}
    \|\widetilde{\nabla}_{\B_t}\|_{op} \leq C (\|\X_t\|_{op} + \|\Z_t\|_{op}).
\end{equation}
Recall that $\|\B_t\|_{op}, \|(\alpha \I + \B_t \B_t^\top)^{-2}\|_{op} \leq C$. Hence, the following chain of inequalities holds
\begin{equation}
    \begin{split}
        \|\widetilde{\nabla}_{\B_t}\|_{op} &\leq \|\B_t\|_{op} \cdot \left\|-\alpha (\alpha \I + \B_t\B_t^\top)^{-2} +
        \alpha \mathrm{Diag} \left(( \alpha \I + \B_t \B_t^\top)^{-2}  (\B_t\B_t^\top)\right)\right\|_{op} \\
        &\leq C \left\|-\alpha (\alpha \I + \B_t\B_t^\top)^{-2}(\I - \B_t\B_t^\top + \B_t\B_t^\top) +
        \alpha \mathrm{Diag} \left(( \alpha \I + \B_t \B_t^\top)^{-2} (\B_t\B_t^\top)\right)\right\|_{op} \\
        &\leq C \Big(\left\|(\alpha \I + \B_t\B_t^\top)^{-2}(\Z_t + \X_t)\right\|_{op} \\ &\hspace{3em} + \left\|(\alpha \I + \B_t\B_t^\top)^{-2}\B_t\B_t^\top -
        \mathrm{Diag} \left(( \alpha \I + \B_t \B_t^\top)^{-2} (\B_t\B_t^\top)\right)\right\|_{op}\Big) \\
        &\leq C \left(\|\X_t\|_{op} +\|\Z_t\|_{op} + \|F(\X_t+\Z_t) - \mathrm{Diag}(F(\X_t + \Z_t))\|_{op}\right),
    \end{split}
\end{equation}
where we recall the definition $F(x)= \frac{1+x}{(\gamma+x)^{2}}$, with $\gamma=1+\alpha$. By combining \eqref{eq:Ftaylor} and \eqref{eq:Ftaylor2} (in the proof of Lemma \ref{lemma:gradient_approx}), we have
$$
\|F(\X_t+\Z_t) - F(0)\I\|_{op} \leq C (\|\X_t\|_{op} +\|\Z_t\|_{op}),
$$
As $\|\mathrm{Diag}(\M)\|_{op} \leq C\|\M\|_{op}$ for any matrix $\M$, we also have that
$$
\|\mathrm{Diag}(F(\X_t+\Z_t))-F(0)\I\|_{op} \leq C (\|\X_t\|_{op} +\|\Z_t\|_{op}).
$$
Hence,
\begin{align*}
    \|F(\X_t+\Z_t) - \mathrm{Diag}(F(\X_t + \Z_t))\|_{op}  
    &\leq C (\|\X_t\|_{op} +\|\Z_t\|_{op} ),
\end{align*}
which finishes the proof of \eqref{lemma:widetildenabla_bound}.

At this point, recall from \eqref{eq:grad_app_final_0} and \eqref{eq:deftilde} that
\begin{equation}\label{eq:gapp}
\opn{\nabla_{\B_t}-\widetilde{\nabla}_{\B_t}}\le \opn{\bE^t  }.
\end{equation}
Thus,
$$
\opn{\nabla_{\B_t} \nabla_{\B_t}^\top} \le 2\opn{\widetilde{\nabla}_{\B_t}  \bE^t  } + \opn{\widetilde{\nabla}_{\B_t} (\widetilde{\nabla}_{\B_t})^\top} + \opn{(\bE^t  )^2}.
$$
Recalling the previous bound on $\|\widetilde{\nabla}_{\B_t}\|_{op}$ in \eqref{lemma:widetildenabla_bound} and using the assumptions in \eqref{eq:assumptionZX}, we get that
$$
\left\|\widetilde{\nabla}_{\B_t} \bE^t  \right\|_{op},\ \opn{\bE^t  }^2 \leq C \|\bE^t  \|_{op},
$$
and
\begin{equation}\label{eq:etalast}
\begin{split}
\eta^2 \|\widetilde{\nabla}_{\B_t}\|_{op}^2 &\leq C\eta  (\opn{\X_t}^2 + \opn{\X_t}\opn{\Z_t}) + C\eta^2\opn{\Z_t}^2 \\
&\le C\eta  \left(\frac{1}{\sqrt{d}} \|\Z_t\|_{op}+\opn{\X_t}^2 + \opn{\X_t}\opn{\Z_t}\right)\le C\eta\opn{\bE^t  },
\end{split}
\end{equation}
where we have also used that $\eta\leq C/\sqrt{d}$. This concludes the proof.
\end{proof}

The next lemma controls the contribution of the projection step.

\begin{lemma}[Projection step]\label{proj_Z} Assume that \eqref{eq:assumptionZX} holds and $\eta\leq C/\sqrt{d}$. Then, for all sufficiently large $n$, with probability $1-1/d^2$, jointly for all $t\ge 0$:
$$
\|\mathrm{proj}(\B'_t) - \B'_t\|_{op} \leq \eta\opn{\bE^t  },
$$
which implies that, by differentiability of the bilinear form,
$$\|\mathrm{proj}(\B'_t)\mathrm{proj}(\B'_t)^{\top} - \B'_t(\B'_t)^{\top}\|_{op} \leq \eta\opn{\bE^t  }.
$$
\end{lemma}

\begin{proof}[Proof of Lemma \ref{proj_Z}] Recall that the objective \eqref{eq:popriskobj} does not depend on the norm of $\{\b_i\}_{i=1}^n$, hence $(\nabla_{\B_t})_{i,:}$ is orthogonal to $(\B_t)_{i,:}$, which implies that 
$$
\mathrm{proj}_i(\B'_t) = \frac{(\B_t)_{i,:} - \eta  (\nabla_{\B_t})_{i,:}}{\sqrt{1+\eta^2\|(\nabla_{\B_t})_{i,:}\|^2}}.
$$
Let us define 
$$
\D_t:=\mathrm{Diag}\left(\frac{1}{\sqrt{1+\eta^2\|(\nabla_{\B_t})_{1,:}\|^2}}, \dots, \frac{1}{\sqrt{1+\eta^2\|(\nabla_{\B_t})_{n,:}\|^2}}\right).
$$
Then, we obtain the following compact form:
$$
\mathrm{proj}(\B'_t) = \D_t (\B_t-\eta\nabla_{\B_t})=\D_t \B_t'.
$$
In this view, it remains to bound $\|\D_t-\I\|_{op}$. 
In more details, by \eqref{lemma:widetildenabla_bound} and \eqref{eq:gapp}, we have
$$
\|\nabla_{\B_t}\|_{op} \le \|\widetilde{\nabla}_{\B_t}\|_{op} + \|\bE^t  \|_{op} \leq C(\|\X_t\|_{op} + \|\Z_t\|_{op} + \|\bE^t  \|_{op}) \le C',
$$
where $C'>0$ is a universal constant (independent of $C_X, n, d$). Hence, by recalling that $\|\B_t\|_{op} \le C$ by assumption \eqref{eq:assumptionZX}, we have
$$
\|\mathrm{proj}(\B'_t) - \B'_t\|_{op}=\opn{(\D_t-\I)(\B_t-\eta\nabla_{\B_t})} \leq C \opn{\D_t - \I}.
$$
Note that function $1/\sqrt{1+x}$ is differentiable at $0$, hence, we have that
for small enough $\eta$ (which follows from $\eta \le C/\sqrt{d}$):
$$
\left|\frac{1}{\sqrt{1+\eta^2\|(\nabla_{\B_t})_{i,:}\|^2}} - 1\right| \le C\eta^2 \|(\nabla_{\B_t})_{i,:}\|^2.
$$
In this view, we have 
$$
\|\D_t-\I\|_{op} \leq C\eta^2 \|\nabla_{\B_t}\|_{op}^2\leq C\eta^2 \|\widetilde{\nabla}_{\B_t}\|_{op}^2 + C\eta^2 \|\widetilde{\nabla}_{\B_t}\| \|\bE^t  \|_{op} + C\eta^2 \|\bE^t  \|^2.
$$
Inspecting each term one by one and applying \eqref{lemma:widetildenabla_bound} in conjunction with $\eta \le C/\sqrt{d}$ gives that
\begin{equation*}
    \begin{split}
        &\eta^2 \opn{\bE^t  }^2 \le C\eta \opn{\bE^t  }, \\ 
        &\eta^2 \|\widetilde{\nabla}_{\B_t}\| \|\bE^t  \|_{op} \leq C\eta \|\bE^t  \|_{op}, \\
        &\eta^2 \|\widetilde{\nabla}_{\B_t}\|_{op}^2 \leq C\eta\opn{\bE^t  },
    \end{split}
\end{equation*}
where in the last step we have used \eqref{eq:etalast}. This concludes the proof.
\end{proof}

\noindent In this view, using \eqref{eq:BBT0} and Lemmas \ref{lemma:gradient_approx}, \ref{lemma:nablanablaT} and \ref{proj_Z}, we obtain
\begin{equation}\label{eq:BBT1}
\begin{split}
    \I + \Z_{t+1} + \X_{t+1} = \B_{t+1}\B_{t+1}^\top &= \I + \Z_t + \X_t + 4\eta \alpha F(\Z_t) - 2\eta\alpha  \mathrm{Diag}(F(\Z_t))(\I+\Z_t) \\
    &- 2\eta\alpha (\I+\Z_t)\mathrm{Diag}(F(\Z_t)) - \frac{8\alpha\eta}{\gamma^3} \X^O_{t} - \frac{4\alpha\eta}{\gamma^2} \X^{D}_{t} + \eta \bE^t  .
\end{split}
\end{equation}
Furthermore, we have that 
\begin{align}\label{eq:diagZ0}
    \begin{split}
    \mathrm{Diag}(F(\Z_t))(\I+\Z_t) &= \left(\mathrm{Diag}(F(\Z_t) - F(0)\I) + F(0)\I\right)(\I+\Z_t) \\
    &= \frac{1}{\gamma^2}(\I+\Z_t) + \left(\mathrm{Diag}(F(\Z_t) - F(0)\I)\right)(\I+\Z_t)\\
    &= \frac{1}{\gamma^2}(\I+\Z_t) + \left(\frac{1}{n}\tr{F(\Z_t) - F(0)\I} + \D'_t\right) (\I+\Z_t),
    \end{split}
\end{align}
where $\D'_t$ is a diagonal matrix such that, with probability at least $1-1/d^2$, its entries are upper bounded in modulus by $\frac{C\log d}{\sqrt{d}}\|\Z_t\|_{op}^{1/2}$. The last passage follows from Lemma \ref{lem_D_diagonal}.
Note that $\frac{1}{\gamma^2}(\I+\Z_t) =\frac{1}{n}\tr{F(0)\I}$ and recall that $\opn{\Z_t} \leq C$. Hence, \eqref{eq:diagZ0} implies that
\begin{align}
\begin{split}\label{eq:diagZ1}
    \mathrm{Diag}(F(\Z_t))(\I+\Z_t) &=  \frac{1}{n}\tr{F(\Z_t)}(\I+\Z_t) + \bE^t  .
\end{split}
\end{align}
Similarly, we have that
\begin{align}
\begin{split}\label{eq:diagZ12}
    (\I+\Z_t)\mathrm{Diag}(F(\Z_t)) &=  \frac{1}{n}\tr{F(\Z_t)}(\I+\Z_t) + \bE^t  .
\end{split}
\end{align}
By combining \eqref{eq:diagZ1}-\eqref{eq:diagZ12} with \eqref{eq:BBT1} and using that $\X_t=\X_t^O+\X_t^D$, we get
\begin{equation}\label{eq:BBT3}
\begin{split}
    \Z_{t+1} + \X_{t+1} &= \left(1-\frac{8\alpha}{\gamma^3}\eta\right) \X_t^O + \left(1-\frac{4\alpha}{\gamma^2}\eta\right) \X_t^D + \Z_t + 4\eta \alpha F(\Z_t) \\ &- 4\eta\alpha  \frac{1}{n}\tr{F(\Z_t)}(\I+\Z_t) + \eta \bE^t  .
\end{split}
\end{equation}
Hence, we can write the following system capturing the dynamics of the spectrum $\Z_t$ and of the errors $(\X_t^O, \X_t^D)$
\begin{align}
    \Z_{t+1} &= \Z_t + 4\eta \alpha F(\Z_t) - 4\eta\alpha \frac{1}{n}\tr{F(\Z_t)}(\I+\Z_t), \label{eq:BBT4:1} \\
    \X_{t+1}^D &= \left(1-\frac{4\alpha}{\gamma^2}\eta\right) \X_t^D + \eta \bE^t  ,\label{eq:BBT4:2} \\
    \X_{t+1}^O &= \left(1-\frac{8\alpha}{\gamma^3}\eta\right) \X_t^O + \eta \bE^t  .\label{eq:BBT4:3}
\end{align}
Here, the operator norm of $\bE^t  $ is upper bounded as in \eqref{eq:defopnE}, where we recall that the constant $C$ is uniformly bounded in $t$. 

In the view of \eqref{eq:BBT4:1}, one can readily see that the updates on the spectrum of $\Z_t$ follow the one described in Lemma \ref{lemma:spectrum_updates} and, thus, converges exponentially. This means that the set of assumptions on $\Z_t$ in \eqref{eq:assumptionZX} is satisfied by suitably picking $C$.

Now it only remains to take care of $\X_t$.
If we write $x_t^D = \opn{\X_t^D}, x_t^O = \opn{\X_t^O}, z_t = \opn{\Z_t}^{1/2}$, then recalling the definition of $\bE_t$ in \eqref{eq:defopnE}, \eqref{eq:BBT4:2}, \eqref{eq:BBT4:3} we have that 
\begin{align}
    x_{t+1}^D &\leq
    \left(1-\frac{4\alpha}{\gamma^2}\eta\right) x_t^D + \eta C_D \left(\frac{\mathrm{poly}(\log d)}{\sqrt{d}} \cdot z_t + (x_t^D + x_t ^O)^2+ (x_t^D + x_t^O) z_t \right)\label{eq:XD} \\
    x_{t+1}^O &\leq \left(1-\frac{8\alpha}{\gamma^3}\eta\right) x_t^O + \eta C_O \left(\frac{\mathrm{poly}(\log d)}{\sqrt{d}} \cdot z_t + (x_t^D + x_t ^O)^2+ (x_t^D + x_t^O) z_t \right)\label{eq:XD_} .
\end{align}
Since both of these recursive bounds are monotone in $x_t^D, x_t^O$, we can dominate them as follows.
If we recursively define $x_t$ by 
\begin{align}
    x_{t+1} = \left(1-\eta\min \left\{\frac{4\alpha}{\gamma^2},\frac{8\alpha}{\gamma^3}\right\}\right)x_t + \eta \max\{C_D,C_O\} \left(\frac{\mathrm{poly}(\log d)}{\sqrt{d}} \cdot z_t + (x_t + x_t )^2+ (x_t + x_t) z_t \right)\label{eq:xt} ,
\end{align}
then by monotonicity $\max\{x_t^D, x_t^O\}\leq x_t$.
Thus, we only need to analyse the recursion \eqref{eq:xt}, which we do in the following lemma. Note that the condition $z_t \leq C e^{-ct\eta}$ required by Lemma \ref{lemma:xtexpconv} holds by \eqref{eq:assumptionZX}.

\begin{lemma}[Error decay]\label{lemma:xtexpconv} Let $\{z_t\}_{t=0}^{\infty}$  be a non-negative exponentially decaying sequence, i.e., $z_t \leq C_z e^{-\eta c_zt}$, and consider a non-negative sequence $\{x_t\}_{t=0}^{\infty}$ such that at each time-step $t$ the following condition holds for $\eta=\Theta(1/\sqrt{d})$ and sufficiently large $d$:
\begin{equation}\label{eq:orrec}
x_{t+1} = (1-\eta c_1)x_{t} + \eta C_2 \cdot z_t \cdot x_t + \eta C_3 x_t^2 + \eta C_4 \cdot \frac{\mathrm{poly}(\log d)}{\sqrt{d}} \cdot z_t,
\end{equation}
with $x_0 = 0$.
Then, the following holds
\begin{equation}\label{eq:xtbound}
x_{t} \leq C \frac{\mathrm{poly}(\log d)}{\sqrt{d}} \cdot T e^{-c T},
\end{equation}
where $T= t\eta$.
\end{lemma}

\begin{proof}[Proof of Lemma \ref{lemma:xtexpconv}] We proceed in two parts. In the first part, we show that our recursion does not blow up in $t  = K/\eta$ steps. In the second part, $z_t \leq C_z \exp(-c_zK)$ will be small, which allows us to deduce \eqref{eq:xtbound}. 

\paragraph{Error does not blow up in finite time.}
Let $t=K/\eta$ where $K$ is such that $K/\eta \in \mathbb{N}$.
We start by analysing the simpler recursion
$$
x_{t+1} = (1-\eta c_1)x_{t} +\eta C_2 \cdot z_t \cdot x_t + \eta C_4 \cdot \frac{\mathrm{poly}(\log d)}{\sqrt{d}} \cdot z_t.
$$
By hypothesis, $z_t \leq C_z$. Hence, we arrive to
$$
x_{t+1} = (1-\eta c_1)x_{t} + \eta C_2C_z  \cdot  x_t +  \eta C_4 C_z \frac{\mathrm{poly}(\log d)}{\sqrt{d}}.
$$
Writing $C_5 =C_2C_z-c_1$, unrolling the recursion on the RHS and using $x_0 = 0$ gives 
\begin{align*}
    x_{t+1} &= \eta C_4 C_z\frac{\mathrm{poly}(\log d)}{\sqrt{d}} \sum_{j=0}^{t} (1+\eta C_5)^j \\
    &\leq 
    \eta C_4 C_z\frac{\mathrm{poly}(\log d)}{\sqrt{d}} \sum_{j=0}^{K/\eta} e^{\eta C_5 j}\\
    &=  \eta C_4 C_z\frac{\mathrm{poly}(\log d)}{\sqrt{d}}  \cdot e^{C_5K} \sum_{j=0}^{K/\eta} e^{-C_5\eta (t-j)} \\
    &\leq \eta  C_4 C_z\frac{\mathrm{poly}(\log d)}{\sqrt{d}}\cdot \frac{e^{C_5K}}{1-e^{-\eta C_5}},
\end{align*}
where the inequality holds for $t \leq K/\eta$ and we have used $1 + x \leq e^x$.
For small enough $\eta$, we have that
$$
\frac{\eta}{1-e^{-C_5\eta}} \leq \frac{2}{C_5},
$$
hence, for all $t \leq K/\eta$,
\begin{equation}\label{eq:xt1}
    x_{t+1} \leq 2 \frac{\mathrm{poly}(\log d)}{\sqrt{d}} \frac{C_4 C_z}{C_5} \exp(C_5 K).
\end{equation}

Let us now go back to our original recursion \eqref{eq:orrec}, which contains the term $x_t^2$. We claim that this recursion satisfies a bound like \eqref{eq:xt1}. Assume by contradiction that it exceeds the bound
\begin{equation}\label{eq:xtcontr1}
    x_{t} \leq  4 \frac{\mathrm{poly}(\log d)}{\sqrt{d}} \frac{C_4 C_z}{C_5} \exp(C_5 K)
\end{equation}
for the first time at step $t'$.
Then, for all $t<t'$, \eqref{eq:xtcontr1} holds.
Noting that $x_{t}^2 \leq  4 \frac{\mathrm{poly}(\log d)}{\sqrt{d}} \frac{C_4 C_z}{C_5} \exp(C_5 K) x_t$ we define $C'_5 = C_2 C_z + 4 C_3\frac{\mathrm{poly}(\log d)}{\sqrt{d}} \frac{C_4 C_z}{C_5}\exp(C_5K) - c_1$.
By unrolling the recursion exactly as before, we obtain
\begin{equation}\label{eq:xtcontr2}
    x_{t+1} \leq 2 \frac{\mathrm{poly}(\log d)}{\sqrt{d}} \frac{C_4 C_z}{C'_5} \exp(C'_5 K) \leq 3\frac{\mathrm{poly}(\log d)}{\sqrt{d}} \frac{C_4 C_z}{C_5} \exp(C_5 K),
\end{equation}
for $d$ large enough.
Here, the second inequality follows for large $d$, since it is clear from the definitions that $\abs{C_5-C'_5}$ vanishes for large $d$.
This shows that we cannot violate \eqref{eq:xtcontr1}, thus \eqref{eq:xtcontr2} holds for all $t \leq K/\eta$.

\paragraph{Convergence of errors $x_t$ to zero.}
We now choose $K$ large enough so that
$$
z_t = C_z e^{- \eta  c_z t} < \frac{c_1}{2C_2}, \quad \forall t \geq K/\eta.
$$
Hence, the term corresponding to $\eta C_2z_tx_t$ can be pushed inside the $(1-\eta c_1)x_t$ term. Consequently, we can equivalently study the following dynamics 
\begin{equation}\label{eq:reduced_dyn1}
    x_{t+1} = (1-\eta c_1')x_t + \eta C_3 x_t^2 + \eta C_4 C_z \frac{\mathrm{poly}(\log d)}{\sqrt{d}} e^{-\eta c_z t},
\end{equation}
where $c_1' = c_1/2$.
Here, we initialize again at $t=0$, but now starting at
$$   
x_{0} =  C_6 \frac{\mathrm{poly}(\log d)}{\sqrt{d}},
$$
where $C_6 = 4 \frac{C_4 C_z}{C_5} \exp(C_5 K)$, corresponding to the bound in \eqref{eq:xtcontr1}. Rearranging we have
\begin{equation}\label{eq:xtdyn}
    x_{t+1} = x_t + \eta\left(-c_1'x_t+C_3 x_t^2 +  C_4 C_z \frac{\mathrm{poly}(\log d)}{\sqrt{d}} e^{-\eta c_z t}\right).
\end{equation}

As the last term inside the brackets vanishes when $d\to \infty$, we have two roots of the polynomial inside the brackets, corresponding to the fixed points of the iteration. The left root $r_l$ scales as
$$
r_l \leq C_l \frac{\mathrm{poly}(\log d)}{\sqrt{d}} e^{-\eta c_z t},
$$
and the right root $r_r$ as
$$
r_r \geq \frac{c_1'}{C_3} - C\frac{\mathrm{poly}(\log d)}{\sqrt{d}} e^{-\eta c_z t}.
$$
In addition, it is easy to see that both roots are non-negative.

Next, we prove that $x_t \leq C \frac{\mathrm{poly}(\log d)}{\sqrt{d}} $ for all $t$. We will show this by contradiction. At initialization we have
$$
x_{0} = C_6 \frac{\mathrm{poly}(\log d)}{\sqrt{d}} .
$$
Choose $A,B$ as follows:
$$
A:=\max\{C_l,C_6\}, \quad B = C_7A.
$$
We first note that, for small enough $\eta$ and large enough $d$, we can choose $C_7$ such that $x_{\tilde{t}} \leq A \frac{\mathrm{poly}(\log d)}{\sqrt{d}}$ implies $x_{\tilde{t}+1} \leq B \frac{\mathrm{poly}(\log d)}{\sqrt{d}}$. We now show that $x_t\leq B \frac{\mathrm{poly}(\log d)}{\sqrt{d}}$ for all $t$. To do so, assume by contradiction that $x_{t+1}> B\frac{\mathrm{poly}(\log d)}{\sqrt{d}}$. Then $x_t\in[A\frac{\mathrm{poly}(\log d)}{\sqrt{d}},B\frac{\mathrm{poly}(\log d)}{\sqrt{d}}]\subseteq [r_l,r_r]$, thus
$$
-c_1'x_t+C_3 x_t^2 +  C_4 C_z \frac{\mathrm{poly}(\log d)}{\sqrt{d}} e^{-\eta c_z t} < 0.
$$
Hence, from \eqref{eq:xtdyn} it follows that
$$x_{t +1} \leq x_t \leq B \frac{\mathrm{poly}(\log d)}{\sqrt{d}},$$
which gives us the desired contradiction.


Thus, for all $t$,
$$
x_t^2 \le B\frac{\mathrm{poly}(\log d)}{\sqrt{d}} x_t.
$$
This allows us to push the second term in \eqref{eq:reduced_dyn1} into the first one (for $d$ large enough), which reduces the recursion to
$$
x_{t+1} = \left(1-\eta c_1''\right)x_t + \eta C_4C_z\frac{\mathrm{poly}(\log d)}{\sqrt{d}} e^{-\eta c_z t},
$$
where $c_1'' \ge c_1'/2$.
By unrolling this last recursion and using  $x_{0}=C_6 \frac{\mathrm{poly}(\log d)}{\sqrt{d}}$, we have that, for $t \geq 1$,
\begin{align}
    x_{t} &= C_6 \frac{\mathrm{poly}(\log d)}{\sqrt{d}} (1-\eta c_1'')^t  +\eta C_4C_z \frac{\mathrm{poly}(\log d)}{\sqrt{d}}\sum_{\ell=1}^{t} (1-\eta c_1'')^{t-\ell} e^{-\eta c_z \ell} \\
    &\leq  C_6 \frac{\mathrm{poly}(\log d)}{\sqrt{d}} \exp(-\eta c_1'' t ) + \eta C_4C_z \frac{\mathrm{poly}(\log d)}{\sqrt{d}}\sum_{\ell=1}^{t} e^{-\eta( c_z \ell + c_1''(t-\ell))},
\end{align}
where the inequality follows from $1-x \leq e^{-x}$.
Since the term in the exponents of the sum is a linear function in $\ell$, its maximum value is attained in the endpoints. Thus, 
$$
x_t \leq C_6 \frac{\mathrm{poly}(\log d)}{\sqrt{d}} \exp(-\eta c_1'' t ) + \eta  C_4C_z \frac{\mathrm{poly}(\log d)}{\sqrt{d}} t \max\{e^{-\eta c_z t}, e^{ -\eta c_1''t}\}  ,
$$
which implies \eqref{eq:xtbound}.
\end{proof}

By Lemma \ref{lemma:xtexpconv} we know that
$$
\|\X_{t}\|_{op} \leq \frac{C}{\sqrt{d}} \cdot T e^{-cT},
$$
where $C$ is independent of $C_X$ by definition. Hence, we can pick $C_X$ such that, for sufficiently large $d$, the assumptions on $\X_t$ in \eqref{eq:assumptionZX} are satisfied.
With this in mind, we can use Lemma \ref{lemma:spectrum_updates} to bound the dynamics involving $\Z_t$ and Lemma \ref{lemma:xtexpconv} to claim that the error $\X_t$ vanishes at least geometrically fast. This concludes the proof of Theorem \ref{thm:GD-min_appendix}.

\section{Auxiliary Results}\label{appendix:aux_results}

\begin{lemma} \label{lem_D_diagonal}
Consider the matrix $\A_t = \U \Lam_t \U^\top$, where the matrix $\U$ is distributed according to the Haar measure and it is independent from the diagonal matrix $\Lam_t$. Further, assume that all the diagonal entries of $\Lam_t$ are bounded in absolute value by a constant. Then, the following results hold.
\begin{enumerate}
\item We have that, with probability at least $1-1/d^2$, 
\begin{equation}\label{eq:firstres}
\max_{i\neq j} |(\A_t)_{i,j}| \leq c \sqrt{\frac{\log d}{d}},    
\end{equation}
for some absolute constant $c>0$.
    \item Let $\D_t = \text{diag}(\A_t)$. Then, 
$$ \D_t = \alpha \I + \D_t', $$
where
$$\alpha = \frac{1}{n} {\rm Tr}(\Lam_t), $$
and $\D'_t$ is a diagonal matrix such that, with probability at least $1-1/d^2$, 
\begin{equation}\label{eq:secres}
    \max_{i \in [n]} |(\D'_t)_{i,i}|  \leq c \frac{\log d}{\sqrt{d}}.
\end{equation}

\item Assume that, for all $t \in \mathbb{N}$, 
\begin{equation}\label{eq:lambdaass}
\|\Lam_t\|_{op} \leq C e^{-c\eta t},    
\end{equation}
where $c, C> 0$ are absolute constants and $\eta=\Theta(1/\sqrt{d})$. Then, with probability at least $1 - 1/d^2$,
\begin{align}
\sup_{t\ge 0}\max_{i\neq j}\abs{(\A_t)_{i, j}} &\leq c \sqrt{\frac{\log d}{d}},\label{eq:thirdres}\\
\sup_{t\ge 0}\max_{i \in [n]} |(\D'_t)_{i,i}|  &\leq c \frac{\log d}{\sqrt{d}}.\label{eq:fourthres}
\end{align}
\end{enumerate}
\end{lemma}
\begin{proof}
We start by proving \eqref{eq:firstres}. Consider the metric measure space $(\mathbb{SO}(d), \|\cdot\|_F, \mathbb P)$. Here, $\mathbb{SO}(d)$ denotes the special orthogonal group containing all $d\times d$ orthogonal matrices with determinant $1$ (i.e., all rotation matrices), and $\mathbb P$ is the uniform probability measure on $\mathbb{SO}(d)$, i.e., the Haar measure. Given a diagonal matrix $\Lam_t$ and two indices $i, j\in [d]$, define $f: \mathbb{SO}(d)\to \mathbb R$ as
\begin{equation}
    f(\M) = (\M\Lam_t \M^\top)_{i, j}.
\end{equation}
Note that 
\begin{equation}
\begin{split}
    |f(\M)-f(\M')| &= |(\M\Lam_t \M^\top)_{i, j}-(\M'\Lam_t (\M')^\top)_{i, j}|\\
    &\le |(\M\Lam_t \M^\top)_{i, j}-(\M'\Lam_t \M^\top)_{i, j}| + |(\M'\Lam_t \M^\top)_{i, j}-(\M'\Lam_t (\M')^\top)_{i, j}|\\
    &\le |((\M-\M')\Lam_t \M^\top)_{i, j}| + |(\M'\Lam_t (\M-\M')^\top)_{i, j}|\\
    &\le \|(\M-\M')\Lam_t \M^\top\|_F + \|\M'\Lam_t (\M-\M')^\top\|_F\\
    &\le 2\|\M-\M'\|_F \|\Lam_t\|_{op} \|\M\|_{op} \le 2\|\M-\M'\|_F \|\Lam_t\|_{op},
\end{split}
\end{equation}
where in the fourth inequality we use that, for any two matrices $\A$ and $\B$, $\|\A\B\|_F\le \|\A\|_{op} \|\B\|_F$, and in the fifth inequality we use that  $\|\M\|_{op}=1$ as $\M\in \mathbb{SO}(d)$. Hence, $f$ has Lipschitz constant upper bounded by $2 \|\Lam_t\|_{op}$ and an application of Theorem 5.2.7 of \cite{vershynin2018high} gives that 
\begin{equation}\label{eq:p1}
    \mathbb P(|f(\U)-\mathbb E[f(\U)]|\ge u)\le 2 \exp\left(-c_1 \frac{d\u^2}{2 \|\Lam_t\|_{op}}\right), 
\end{equation}
where $c_1$ is a universal constant.

Let $\u_i$ denote the $i$-th row of $\U$. Then, 
\begin{equation}
    f(\U) = \langle \u_i, \Lam_t \u_j\rangle. 
\end{equation}
Suppose that $i\neq j$. Since $\U$ is distributed according to the Haar measure, $\u_i$ is uniform on the unit sphere and $\u_j$ is uniformly distributed on the unit sphere in the orthogonal complement of $\u_i$ (see Section 1.2 of \cite{meckes2019random}). Thus, $(\u_i, \u_j)$ has the same distribution as $(-\u_i, \u_j)$, which implies that, whenever $i\neq j$
\begin{equation}\label{eq:p2}
    \mathbb E[f(\U)] = 0.
\end{equation}
By combining \eqref{eq:p1}-\eqref{eq:p2} with a union bound over $i, j$, we have that
\begin{equation}\label{eq:p3}
    \mathbb P(\max_{i\neq j}|(\U\Lam_t \U^\top)_{i, j}|\ge u)\le 2d^2 \exp\left(-c_1 \frac{du^2}{2 \|\Lam_t\|_{op}}\right). 
\end{equation}
As $\|\Lam_t\|_{op}$ is upper bounded by a universal constant, the result \eqref{eq:firstres} readily follows. 

For the second part, note that 
\begin{equation}
    (\D_t)_{i, i} = \langle \u_i, \Lam_t \u_i\rangle.
\end{equation}
Furthermore, the following chain of equalities hold
\begin{equation}
\begin{split}
    \mathbb E[(\D_t)_{i, i}] &= \frac{1}{n}\sum_{i=1}^n\mathbb E[(\D_t)_{i, i}] 
    = \mathbb E\left[\frac{1}{n}\sum_{i=1}^n (\D_t)_{i, i}\right] = \frac{1}{n}{\rm Tr}(\D_t),
\end{split}
\end{equation}
where the first equality uses that the $\u_i$'s have the same (marginal) distribution, and the last term does not contain an expectation since ${\rm Tr}(\D_t) = {\rm Tr}(\A_t) =\sum_{i=1}^d (\Lam_t)_{i, i}$, which does not depend on $\U$. Therefore, by using \eqref{eq:p1} and by performing a union bound over $i\in [n]$, the result \eqref{eq:secres} follows. 

For the third part, by performing a union bound over $t\ge 0$ in \eqref{eq:p3}, we have that \eqref{eq:thirdres} holds with probability at least
\begin{equation}\label{eq:neweq}
    \begin{split}        
    2\sum_{t=0}^{\infty} \exp\left(-c_1\frac{du^2}{2\|\Lam_t\|_{op}}\right) &\leq 
    2\sum_{t=0}^{\infty} \exp\left(-c_2 \, d\,u^2 \, e^{C\eta t} \right) \\
    &\leq 
    2\sum_{t=0}^{\infty} \exp\left(-c_2 \, d\,u^2 \, e^{C\lfloor\eta t\rfloor} \right)\\
    &\leq 
    2\left\lceil\frac{1}{\eta}\right\rceil\sum_{t=0}^{\infty} \exp\left(-c_2 \, d\,u^2 \, e^{C t} \right)\\
    &\leq 
    C\sqrt{d}\sum_{t=0}^{\infty} \exp\left(-c_2 \, d\,u^2 \, e^{C t} \right),
    \end{split}
\end{equation}
where the first inequality follows from \eqref{eq:lambdaass} and the last one from $\eta=\Theta(1/\sqrt{d})$. 
Choosing $u=c\frac{\log d}{\sqrt{d}}$ we can get that $b:=\exp\left(-c_2\, d\,u^2 \right) < 1$ and, hence, the following holds
\begin{align*}
    \sum_{t=0}^{\infty} \exp\left(-c_2 \, d\,u^2 \right)^{e^{Ct}} 
    &\leq \sum_{t=0}^{\infty} \exp\left(-c_2 \, d\,u^2  \right)^{Ct+1} = \frac{b}{1-b^C} \leq \frac{1}{d^3},
\end{align*}
where the first inequality uses that $e^t\ge 1+t$ and the second inequality follows from the definition of $b$. This concludes the proof of \eqref{eq:thirdres}. The proof of \eqref{eq:fourthres} uses an analogous union bound on $t\ge 0$.
\end{proof}

\begin{lemma}\label{lemma:spectrum_updates} Let $\lambda^0=\{\lambda^0_1,\cdots,\lambda^0_n\}$ be a set of numbers in $\mathbb{R}$ such that 
$$
\lambda^0_{min} := \min_{i\in[n]}\lambda^0_i \ge \delta > 0, \quad \lambda^0_{max} := \max_{i\in[n]}\lambda^0_i \le M < +\infty, \quad \sum\limits_{j=1}^n \lambda^0_j = n.
$$
Let the values $\{\lambda^t_i\}_{i=1}^n$ be updated according to the equation below
\begin{equation}\label{spup}
    \lam_i^{t+1} = \lam_i^t + \eta \left( F(\lam^t_i) -  \lam^t_i \cdot  \frac{1}{n}\sum\limits_{j=1}^n F(\lam^t_j) \right) = G(\lam^t_i,\lambda^t),
\end{equation}
where $F(\cdot)$ is defined as per Lemma \ref{lemma:gradient_approx}, $\eta = \Theta\left(1/\sqrt{d}\right)$ and $\lambda^t:=\{\lambda^t_1,\cdots,\lambda^t_n\}$. Then, for large enough $d$, we have 
    \begin{equation*}
        \abs{\lam^{t+1}_i - 1} \leq (1 - c\delta \cdot \eta ) \abs{\lam_i^t-1}
    \end{equation*}
    and thus after $t$ iterations 
    \begin{equation*}
        \abs{\lam^t_i - 1} \leq\max\{(M-1), (1-\delta)\}\exp(-c\delta \cdot \eta t),
    \end{equation*}
    where $c, C>0$ are constants. 
\end{lemma}

\begin{proof}
    We first show by induction that $\sum\limits_{i=1}^n \lam^t_i = n$ holds for all $t$.
    In fact,
    \begin{align*}
        \sum\limits_{i=1}^n \lam_i^{t+1} 
        &= \sum\limits_{i=1}^n \lam_i + \eta \left(\sum\limits_{i=1}^n F(\lam^t_i) - \sum\limits_{i=1}^n \lam^t_i \cdot \frac{1}{n} \sum\limits_{j=1}^n F(\lam^t_j) \right) \\
        &= n + \eta \left(\sum\limits_{i=1}^n F(\lam^t_i) - \sum\limits_{j=1}^n F(\lam^t_j) \right) =n.
    \end{align*}
    Now, we will show the convergence of $\lam^t_{min}$ and $\lam^t_{max}$. To do so, we assume that $\lambda_{max}^t \le M$ and $\lambda_{min}^t \ge \delta$ holds at time step $t$ (we will verify this later). Define the function $g:\mathbb{R} \rightarrow \mathbb{R}$ as
    \begin{equation}\label{spup:g}
        g(x) := x + \eta \left( F(x) - x \cdot C\right).
    \end{equation}
    By taking the derivative, we have that, for sufficiently large $d$,
    $$
    g'(x) = 1 + \eta \left(F'(x) - C \right) > 0,
    $$
    as $\|F'\|_{\infty} \le C$. This implies that $g(\cdot)$ is a monotone increasing function, which gives that 
    \begin{equation}\label{spup:inc}
    \begin{split}
        &\max_{i\in[n]}g(\lambda^t_i) = g(\lambda_{max}^t),\\
        &\min_{i\in[n]}g(\lambda^t_i) = g(\lambda_{min}^t).
    \end{split}
    \end{equation}
    Note that the updates on $\lambda_i^t$ in \eqref{spup} have a common part for all $i\in[n]$, i.e.,
    $$
    \left|\frac{1}{n} \sum\limits_{j=1}^n F(\lam^t_j)\right| \le C,
    $$
    where we used that $\|F\|_{\infty} \le C$. In this view, by definition of $g$ and \eqref{spup:inc}, we have
    \begin{equation}\label{spup:ext}
    \begin{split}
        &\lambda_{max}^{t+1} = G(\lam^t_{max},\lambda^t),\\
        &\lambda_{min}^{t+1} = G(\lam^t_{min},\lambda^t),
    \end{split}
    \end{equation}
    which means that the min/max value at the previous step are mapped to the min/max value at the next step of \eqref{spup}. Using that $\frac{1}{n}\sum\limits_{i=1}^n \lam^t_i = 1$ we can write
    \begin{equation}\label{spup:rewrite}
    \begin{split}
        \lam_i^{t+1} &= \lam^t_i + \eta \left( \frac{1}{n}\sum\limits_{j=1}^n \lam^t_j \cdot F(\lam^t_i) - \lam^t_i \cdot \frac{1}{n}\sum_{j=1}^n F(\lam^t_j) \right) \\
        &= \lam^t_i + \eta \left(\frac{1}{n} \sum_{j=1}^n \left[ \frac{\lam^t_j \lam^t_i}{(\alpha + \lam^t_i)^2} - \frac{\lam^t_i \lam^t_j}{(\alpha + \lam^t_j)^2} \right]\right)\\
        &=  \lam^t_i + \eta \left(\frac{1}{n} \sum_{j=1}^n \lam^t_i \lam^t_j\left( \frac{(2\alpha + \lam^t_i + \lam^t_j)(\lam^t_j - \lam^t_i)}{(\alpha + \lam^t_i)^2(\alpha + \lam^t_j)^2}  \right)\right).
    \end{split}
    \end{equation}
    Recall that we assumed $\lambda_{max}^t \le M$ and $\lambda_{min}^t \ge \delta$. In this view, we get the following bound
    \begin{equation}\label{spup:lmaxbound}
        \lam^t_{max} \lam^t_j\left( \frac{(2\alpha + \lam^t_{max} + \lam^t_j)(\lam^t_{max}-\lam^t_j)}{(\alpha + \lam^t_{max})^2 (\alpha + \lam^t_j)^2}\right) \geq (\lam^t_{max} - \lam^t_j) \cdot \frac{2\alpha\delta }{(\alpha + M)^4},
    \end{equation}
    which is justified as follows
    \begin{align*}
        \lam^t_{max} \lam^t_j\left( \frac{(2\alpha + \lam^t_{max} + \lam^t_j)(\lam^t_{max}-\lam^t_j)}{(\alpha + \lam^t_{max})^2 (\alpha + \lam^t_j)^2}\right) &= (\lam^t_{max}-\lam^t_j) \cdot \left( \frac{(2\alpha + \lam^t_{max} + \lam^t_j)\lam^t_{max} \lam^t_j}{(\alpha + \lam^t_{max})^2 (\alpha + \lam^t_j)^2}\right) \\
        &\geq (\lam^t_{max}-\lam^t_j) \cdot \frac{2\alpha \cdot 1 \cdot \delta}{(\alpha + M)^2 (\alpha + M)^2},
    \end{align*}
    where we used that $\lambda^t_{max} \geq 1$ since $\sum_{i=1}^n \lambda^t_i = n$.
    Hence, using the previous observation about mapping of extremes in \eqref{spup:ext} and the observation above, we get from \eqref{spup:rewrite} that
    \begin{equation}
        \begin{split}
            \lambda_{max}^{t+1} &\leq \lambda_{max}^t - \eta \cdot \frac{1}{n}\sum_{j=1}^n \left[(\lam^t_{max} - \lam^t_j) \cdot \frac{2\alpha\delta }{(\alpha + M)^4}\right],
        \end{split}
    \end{equation}
    which leads to
    \begin{equation}
        \begin{split}
            \lambda_{max}^{t+1} - 1 &\leq \lambda_{max}^t - 1 - \eta \cdot \frac{1}{n}\sum_{j=1}^n \left[(\lam^t_{max} - \lam^t_j) \cdot \frac{2\alpha\delta }{(\alpha + M)^4}\right] \\
            &= \lambda_{max}^t - 1 - \eta \cdot \left[(\lam^t_{max} - 1) \cdot \frac{2\alpha\delta }{(\alpha + M)^4}\right]\\
            &=(\lambda_{max}^t - 1) \left(1 - \eta \cdot \frac{2\alpha\delta }{(\alpha + M)^4}\right) = (\lambda^t_{max}-1)(1 - c \delta \cdot \eta),
        \end{split}
    \end{equation}
    where we used that $\sum\limits_{j=1}^n \lambda_j^t=n$ in the first equality. Hence,  using that $\lambda_{max}^t \ge 1$ as $\sum_{j=1}^n \lambda_j^t = n$ we have
    \begin{equation}\label{eq:lmax}
    |\lambda_{max}^{t+1} - 1| = \lambda_{max}^{t+1} - 1 \leq |\lambda_{max}^{t} - 1|  \cdot (1-c\delta \cdot \eta).
        \end{equation}
    Similarly to the previous bound, we get that
    $$
    \lam^t_{min} \lam^t_j\left( \frac{(2\alpha + \lam^t_{min} + \lam^t_j)(\lam^t_{min}-\lam^t_j)}{(\alpha + \lam^t_{min})^2 (\alpha + \lam^t_j)^2}\right) \leq \lambda_j^t(\lam^t_{min} - \lam^t_j) \frac{2\alpha\delta }{(\alpha + M)^4},
    $$
    since $\lambda^t_{min} \leq \lambda_t$.
    Hence, using the previous observation about mapping of extremes in \eqref{spup:ext} and the observation above, we deduce from \eqref{spup:rewrite} that
    \begin{equation}
        \begin{split}
            \lambda_{min}^{t+1} - 1 &\geq (\lambda_{min}^{t} - 1) - \eta \cdot \frac{1}{n} \sum_{j=1}^n \left[\lam^t_j(\lam^t_{min} - \lam^t_j) \frac{2\alpha\delta }{(\alpha + M)^4}\right] \\
            &= (\lambda_{min}^{t} - 1) - \eta \cdot \lam^t_{min} \cdot \frac{2\alpha\delta }{(\alpha + M)^4} + \eta \cdot \frac{2\alpha\delta }{(\alpha + M)^4} \cdot \frac{1}{n} \sum_{j=1}^t \left(\lambda_i^t\right)^2\\
            &\geq (\lambda_{min}^{t} - 1) - \eta \cdot (\lambda_{min}^{t} - 1) \cdot \frac{2\alpha\delta }{(\alpha + M)^4} \\
            &= (\lam^t_{min} - 1) \cdot (1 - c \delta \cdot \eta),
        \end{split}
    \end{equation}
    where in the second inequality we used Jensen's inequality for $x^2$ as $\sum_{j=1}^n \lambda_j^t = n$. Hence, we get the following
        \begin{equation}\label{eq:lmin}
    |\lambda_{min}^{t+1}-1| = 1-\lambda_{min}^{t+1} \leq |\lam^t_{min}-1| \cdot (1 - c \delta \cdot \eta),
    \end{equation}
    since $\lambda_{min}^t \le 1$ as $\sum_{j=1}^n \lambda_j^t = n$.
    
    In this view, the assumptions $\lambda_{max}^t \le M$ and $\lambda_{min}^t \ge \delta$ follow from \eqref{eq:lmax} and \eqref{eq:lmin} since the extremes are getting closer to one after each iteration. Recalling that by the assumption on initialization
    $$
    \max_{i}|\lambda_i^0 - 1| \leq \max\{(M-1),(1-\delta)\},
    $$
    the claim follows.
\end{proof}

\section{Proofs for General Covariance}\label{appendix:water_fill_proofs}

\begin{lemma}\label{lemma:two_buckets} Assume that $\{\hat{\gamma}_i\}_{i \in [K]}, \{\hat{s}_i\}_{i \in [K]}$ minimize 
\begin{equation}\label{eq:popriskDreduced2lemma}
 -\frac{\left(\sum_{i=1}^K D_i\gamma_i\right)^2}{\left(g(1) \cdot n + \sum_{i=1}^K\frac{\gamma_i^2}{s_i}\right)}.
\end{equation} 
Then, for any $i < j$, we must have $\hat{s}_i = \min\{\hat{s}_i + \hat{s}_j, k_i\}$.
\end{lemma}

\begin{proof}[Proof of Lemma \ref{lemma:two_buckets}]
Since the $\{\hat{\gamma}_i\}_{i\in [K]}, \{\hat{s}_i\}_{i\in [K]}$ are optimal, if we fix two indices $i<j$ the corresponding $\hat{\gamma}_i, \hat{\gamma}_j, \hat{s}_i, \hat{s}_j$ are optimal among all $\gamma_i, \gamma_j, s_i, s_j$ satisfying 
\begin{equation}\label{eq:two_buckets_constraints}
\begin{cases}
0 < \gamma_i + \gamma_j = \gamma:=\hat{\gamma}_i + \hat{\gamma}_j \leq n, \\ 0 <s_i + s_j = s:= \hat{s}_i + \hat{s}_j \leq \min\{n, k_i + k_j\}.
\end{cases}
\end{equation}

Thus, we proceed by analysing the solution for two fixed indices under the constraints  \eqref{eq:two_buckets_constraints} (keeping all other $\hat{\gamma}_l, \hat{s}_l$ for $l \notin \{i, j\}$ fixed).
Note that, for each fixed $(\gamma_i,\gamma_j)$ satisfying the constraints \eqref{eq:two_buckets_constraints}, the following objective
\begin{equation}\label{eq:aux_two_bucket}
\begin{split}
    &\frac{\gamma^2_i}{s_i} + \frac{\gamma^2_j}{s_j} \rightarrow \min_{s_i, s_j}\\
    &\ \mathrm{s.t.}\quad s_i \leq k_i,\ s_j \leq k_j, \ s_i + s_j = s
\end{split}
\end{equation}
is equivalent to finding optimal ranks for \eqref{eq:popriskDreduced2lemma}. 
Importantly, in \eqref{eq:aux_two_bucket} we consider continuous $(s_i,s_j)$. This relaxation has the same minimum, since we will show that the optimal $s_i, s_j$ have integer values. 
We may also assume that $\gamma_j >0$ as otherwise clearly $s_i = \min\{s, k_i\}$ is optimal.

Since \eqref{eq:aux_two_bucket} is strictly convex (on the domain given by the constraints), we can find its unique minimizer by finding a solution to the KKT conditions:
$$-\frac{\gamma_i^2}{s_i^2} + (\lambda + \mu_i) = 0,\quad-\frac{\gamma_j^2}{s_j^2} + (\lambda + \mu_j) = 0, \quad \mu_i, \mu_j \geq 0,\quad\mu_i(s_i -k_i) = 0, \quad\mu_j(s_j-k_j) = 0 ,\quad s = s_i + s_j.$$
If $s_i = k_i$  or $s_j=0$, then the claim is readily obtained. We will now prove that, if this is not the case, then we can find new $\widetilde{s_i}, \widetilde{s_j}, \widetilde{\gamma_i}, \widetilde{\gamma_j}$ which achieve a better value.

We first show that for $s_i < k_i, 0<s_j < k_j$
\begin{equation}\label{eq:frel}
\frac{\gamma_i^2}{s_i} + \frac{\gamma_j^2}{s_j}= \gamma_i \frac{\gamma}{s} + \gamma_j \frac{\gamma}{s}= \frac{\gamma^2}{s}.    
\end{equation}
Note that, in this case, $\mu_i = \mu_j =0$, so the first two KKT conditions imply
$$\frac{\gamma_i}{s_i} = \sqrt{\lam} = \frac{\gamma_j}{s_j}.$$
Thus, we have
\begin{equation}\label{eq:aux_st}
   \frac{\gamma_i}{s_i} = \frac{\gamma_j}{s_j} = \frac{\gamma_i + \gamma_j}{s_i + s_j} = \frac{\gamma}{s},
\end{equation}
from which \eqref{eq:frel} is immediate.

For the case $s_j = k_j$ and $s_i<k_i$, we have that $\mu_j\ge \mu_i=0$, hence
$$\frac{\gamma_i}{s_i} = \sqrt{\lambda + \mu_i} \leq \sqrt{\lambda + \mu_j}=\frac{\gamma_j}{s_j}.$$
From the previous case, we know that without the  constraints on $k_i,k_j$ the optimal value in \eqref{eq:aux_two_bucket} is $\frac{\gamma^2}{s}$.
Thus,
$$ \frac{\gamma_i^2}{s_i} + \frac{\gamma_j^2}{s_j} \geq \frac{\gamma^2}{s}.$$

Now, for $\epsilon >0$, define $\widetilde{s}_i = s_i +\epsilon, \widetilde{s}_j = s_j -\epsilon$. Note that, as $s_i<k_i$ and $s_j>0$, we can choose $\epsilon$ small enough such that $0 < \widetilde{s}_i<k_i,0< \widetilde{s}_j<k_j$.
At this point, let us simply choose $\widetilde{\gamma}_i, \widetilde{\gamma}_j$ such that
$$
\frac{\widetilde{\gamma}_i}{\widetilde{s}_i} = \frac{\widetilde{\gamma}_j}{\widetilde{s}_j}
$$
which as in \eqref{eq:frel}, \eqref{eq:aux_st} implies that 
\begin{equation}\label{eq:better_params1}
    \frac{\widetilde{\gamma}_i^2}{\widetilde{s}_i} + \frac{\widetilde{\gamma}_j^2}{\widetilde{s}_j} = \frac{\gamma^2}{s} \leq \frac{\gamma_i^2}{s_i} + \frac{\gamma_j^2}{s_j}.
\end{equation}
We also have $\widetilde{\gamma}_i> \gamma_i$,
as otherwise
$$\frac{\widetilde{\gamma}_i}{\widetilde{s}_i} < \frac{\gamma_i}{s_i} \leq \frac{\gamma_j}{s_j} < \frac{\widetilde{\gamma}_j}{\widetilde{s}_j}  $$
would be a contradiction. This gives that 
$$D_i \gamma_i + D_j \gamma_j < D_i \widetilde{\gamma}_i + D_j\widetilde{\gamma}_j,$$
which implies that our new choice achieves a lower value for \eqref{eq:popriskDreduced2lemma}, thus giving the desired contradiction.



\end{proof}

\begin{lemma}\label{lem:LBD_KKT_general}
Assume that $f,f_i$ are differentiable strictly convex functions on $\R$ such that 
\begin{equation}\label{eq:ordering_app}
 f_i'(0) < f'_j(0)<0, \ i < j, \quad \lim_{m_i \rightarrow +\infty} f'_i(m_i) = + \infty, \quad \lim_{m_i \rightarrow -\infty} f'_i(m_i) = - \infty,
\end{equation}
and 
\begin{equation}\label{eq:young_app}
f(0) = f'(0) = 0, \quad \lim_{m \to +\infty} f'(m) = +\infty.
\end{equation}
Then, the objective given by
\begin{equation}\label{eq:general_objective_app}
\min_{ m_i \ge 0} f\left(m\right) + \sum_{i=1}^{K} f_i(m_i), \quad m = \sum_{i}^K m_i
\end{equation}
has a unique minimizer. It is uniquely characterised by being of the form
$(m_1,\ldots,m_M, 0,\ldots, 0)$ and satisfying
\begin{equation}\label{eq:opt_m_mi}
    m = \sum_{i=1}^M \left(\left( -f_i'\right)^{-1}\circ f'\right)(m), \quad m_i =\left(\left( -f_i'\right)^{-1}\circ f'\right)(m) \geq 0,\quad f'(m) +f_i'(m_i) \geq 0, \quad  i \in [M].
\end{equation}
Furthermore, it can be obtained via binary search by finding the largest index $M$, such that the corresponding $m_i$ are all strictly positive.
\end{lemma}

While the assumptions of this theorem might seem technical, most of them can be relaxed. However, we note that all such assumptions are fulfilled by the setting being studied and relaxing them would come at the cost of the readability of the proof of Lemma \ref{lem:LBD_KKT_general}. 

\begin{proof}[Proof of Lemma \ref{lem:LBD_KKT_general}] We start by showing that \eqref{eq:general_objective_app} has a unique minimizer. Recall that $f$ and $f_i$ are strictly convex functions, and, hence, their derivatives $f'$ and $f'_i$ are increasing.
From \eqref{eq:young_app}, we also obtain that $\lim_{m \to +\infty} f'(m) = +\infty$. By monotonicity, we have $f_i'(m_i) \geq f_i'(0) $. Therefore,
$$\lim_{m\to +\infty } f'(m) + \sum_{i=1}^{K} f'_i(m_i) = + \infty,$$
and thus 
$$\lim_{m\to +\infty } f(m) + \sum_{i=1}^{K} f_i(m_i) = + \infty.$$
As a consequence, the objective achieves its infimum. Therefore, as $f(m) + \sum_{i=1}^{K} f_i(m_i)$ is strictly convex, the minimum is unique.

Notice that Slater's condition is satisfied, since the feasible set of \eqref{eq:general_objective_app} has an interior point. Hence, $\{m_i\}_{i=1}^{K}$ is a unique minimizer of \eqref{eq:general_objective_app} if and only if it satisfies the following KKT conditions (for the ``if and only if'' statement, see for instance page 244 in \cite{boyd2004convex}):
\begin{enumerate}
    \item\label{eq:kkt_wt_stationary} Stationary condition: $f'(m) +f_i'(m_i) - \lambda_i = 0.$
    \item Primal feasibility: $m_i \geq 0.$
    \item Complementary slackness: $\lambda_i m_i = 0.$
    \item Dual feasibility: $\lambda_i \geq 0.$
\end{enumerate}
In particular, the uniqueness of the minimizer implies that the KKT conditions have a unique solution. 
Thus, we only need to show that the $m_i$ found by this procedure satisfy the above equations.

We now show that the active set $\mathcal{A}:=\{i:m_i > 0\}$ for the optimal $m_i$ is monotone, meaning that $\mathcal{A} = [M]$ for some $M \leq K$.
We prove the statement by contradiction. Assume that there exists $m_i = 0$ and $m_j > 0$ where $i < j$. Recall that $f_j'$ is strictly increasing, which by the ordering condition \eqref{eq:ordering_app} implies that
$$
f'_i(0) + f'\left(\sum_{\ell=1}^{K} m_\ell\right) < f'_j(m_j) + f'\left(\sum_{\ell=1}^{K} m_\ell\right).
$$
Hence, taking some sufficiently small mass from $m_j$ and redistributing it in $m_i$ will decrease the objective value in \eqref{eq:general_objective_app}, which concludes the proof.

Fix $M \leq K$. We now show that the solution of the following system of equations 
\begin{equation}\label{eq:search_system}
    f'(m) + f'_i(m_i) = 0, \quad \forall i \leq M
\end{equation}
exists and unique. 
Note that this system comes from the 1. and 3. KKT conditions.

As $f'_i$ is strictly monotone, its inverse exists and, hence, from \eqref{eq:search_system} we get 
\begin{equation}\label{eq:mi_computation}
m_i = (-f_i')^{-1}(f'(m)),
\end{equation}
which gives 
\begin{equation}\label{eq:m_computation}
m = \sum_{i=1}^M (-f_i')^{-1}(f'(m)).
\end{equation}
Let us argue the existence and uniqueness of the solution of equation \eqref{eq:m_computation} for a fixed $M$. Recall that $f'_i$ is increasing and, thus, $-f'_i$ is decreasing. The inverse of a decreasing function is decreasing, hence $(-f_i')^{-1}$ is decreasing. Recalling that $f'$ is increasing and that the composition of an increasing and a decreasing function is decreasing, it follows that $(-f_i')^{-1}(f'(m))$ is decreasing. By assumption $f'_i(0) < 0$ and $f'_i$ is increasing such that $\lim_{m_i \rightarrow +\infty} f_i(m_i) = +\infty$, therefore the value $(-f_i')^{-1}(0)$ is well-defined and 
$$
(-f_i')^{-1}(0) > 0.
$$
Thus, we have that
$$
g_M(m) =  \sum_{i=1}^M (-f_i')^{-1}(f'(m)) - m
$$
is a strictly decreasing function with
$$
\lim_{m\rightarrow+\infty} g_M(m) = -\infty, \quad g_M(0) > 0.
$$
In this view, the solution of \eqref{eq:m_computation} exists and unique.

Next, we elaborate on why \eqref{eq:mi_computation} is well-defined given the solution of \eqref{eq:m_computation}. Note that, by our assumptions,
$$
\lim_{m_i \rightarrow +\infty} f'_i(m_i) = + \infty, \quad \lim_{m_i \rightarrow -\infty} f'_i(m_i) = - \infty,
$$
hence, the same holds for $(-f'_i)^{-1}$, and, thus, due to continuity the quantity
$$
(-f'_i)^{-1}(x)
$$
is well-defined for any $x\in\mathbb{R}$. Given this, we readily have that the solution of the system \eqref{eq:search_system} exists and unique. Furthermore, this solution can be found using \eqref{eq:m_computation} and \eqref{eq:mi_computation}. Note also that \eqref{eq:m_computation} and \eqref{eq:mi_computation} agree with \eqref{eq:opt_m_mi}.

We now show that the following procedure finds the optimal active set $\mathcal{A}^{*} = [M^{*}]$. Let $m_i(M), \ i \leq M$ be a solution of \eqref{eq:search_system} for fixed value of $M \leq K$, and define $m(M) := \sum_{i=1}^M m_i(M)$. Using \eqref{eq:m_computation} and \eqref{eq:mi_computation} find the smallest $M$ such that the corresponding $m_M(M)$ is non-negative, then $M^{*} = M - 1$ if $M\geq 1$, otherwise, $m = m_i = 0, \ \forall i \in [K]$. If no such $M$ was found, $M^{*} = [K]$. To show that the described procedure in fact gives the optimal active set $\mathcal{A}^{*} = [M^{*}]$, we need to prove that
\begin{enumerate}
\item If $M < M^{*}$, then $m_i(M) \geq 0$. 
\item If $M > M^{*}$, then $m_M(M) \leq 0$. 
\end{enumerate}

Clearly, these two conditions imply that the active set of the minimizer is given by $[M^*]$, and it can be found via binary search.

We start by proving the first property. Note that, by the KKT conditions on the optimizer $M^{*}$, we have that 
$$
m_i(M^{*}) \geq 0.
$$

First assume that $m(M) > m(M^{*})$. By monotonicity, it follows from \eqref{eq:mi_computation} that
$$
m_i(M) < m_i(M^{*}),
$$
but 
$$
m(M^{*}) = \sum_{i=1}^M m_i(M^{*}) + \sum_{i=M+1}^{M^{*}} m_i(M^{*}) \geq \sum_{i=1}^M m_i(M^{*}) >  \sum_{i=1}^M m_i(M) = m(M),
$$
where we have used that $m_i(M^{*}) \geq 0$, which is a contradiction.
Thus, we have that $m(M) \leq m(M^{*})$.
Again, by \eqref{eq:mi_computation} and monotonicity, 
$$
 m_i(M) \geq m_i(M^{*}),
$$
and, hence, all $m_i(M)$ are non-negative. 

We finally argue the second property. We start by proving a weaker statement, i.e., there exists $i \geq M^{*} + 1$ such that $m_i(M) < 0$. Assume that $m(M) < m(M^{*})$. By \eqref{eq:mi_computation} and monotonicity
$$
m_i(M) > m_i(M^{*}),
$$
hence, the following holds:
$$
m(M) =  \sum_{i=1}^M m_i(M) =  \sum_{i=1}^{M^{*}} m_i(M) +   \sum_{i=M^{*}+1}^{M}  m_i(M) > \sum_{i=1}^{M^{*}} m_i(M^{*}) + \sum_{i=M^{*}+1}^{M}  m_i(M) = m(M^{*}) + \sum_{i=M^{*}+1}^{M}  m_i(M),
$$
which since $m(M) < m(M^{*})$ implies that $\sum_{i=M^{*}+1}^{M}  m_i(M)$ is a negative quantity. Thus, there exists $i \geq M^{*} + 1$ such that $m_i(M) < 0$. Assume now that $m(M) \geq m(M^{*})$. Recall that only the minimizer satisfies the KKT conditions, thus
$$
f'(m(M^{*})) + f'_M(0) \geq 0,
$$
which, as $f'$ is increasing, implies that
$$
f'(m(M)) + f'_M(0) \geq 0.
$$
By construction of $m_M(M)$, we know that
$$
f'(m(M)) + f'_M(m_M(M)) = 0,
$$
thus, by monotonicity of $f'_M$ we have $m_M(M) \leq 0$.

It remains to show that it suffices to check $m_M(M) \leq 0$ and not an arbitrary $m_i(M)$ for $i \geq M^{*} + 1$. 
Assume that $m_i(M) \leq 0$ for some $i \leq M$. 
Recall that by assumption
$$
f'_i(0) < f'_M(0) < 0,
$$
and by construction we have
$$
f_i'(m_i(M)) = f_M'(m_M(M)) = -f'(m(M)).
$$
Since $f'_i$ is a decreasing function, we get that $-f'(m(M)) < f_i'(0)$. Recalling that $f'_i(0)<f'_M(0)$, we get $-f'(m(M))<f_M'(0)$ and, hence, by monotonicity of $f_M'$ we obtain that $m_M(M) \leq 0$, which concludes the proof.
\end{proof}

\begin{lemma}\label{lem:LBD_KKT}
The minimizer of \eqref{eq:popriskDLB_cv} 
can be computed in $\log(K)$ steps via binary search by finding the smallest index $M^{*}$ such that
\begin{equation}\label{eq:bulbasaur_cond}
\frac{g(1)}{c_1^2n}\sum_{j=1}^{M^{*} + 1} s_j (D_{M^{*} + 1} - D_j) + D_{M^{*} + 1} \leq 0.
\end{equation}
Then, the optimal active set has the form $\mathcal{A} = [M^{*}]$ and corresponding non-zero $\beta_i\ $, for $i \leq M^{*}$, are computed as
\begin{equation}\label{eq:bulbasaur}
\beta_i =\frac{s_i}{c_1} \cdot \left(\frac{\frac{g(1)}{c_1^2n} \sum_{j\in\mathcal{A}}s_j\Delta_j + D_1 }{\frac{g(1)}{c_1^2n}\sum_{j\in\mathcal{A}}s_j + 1} - \Delta_i\right),
\end{equation}
where $\Delta_j = D_1 - D_j$.
\end{lemma}

\begin{proof}[Proof of Lemma \ref{lem:LBD_KKT}] 
By rescaling $g(x)$ as $\frac{g(x)}{c_1^2}$ and $\beta_i$ as $c_1 \beta_i$, we may without loss of generality assume that $c_1 = 1$.
From the results of Lemma \ref{lem:LBD_KKT_general}, by a direct computation, we get that for $\mathcal{A}=[M]$
$$
\beta_j(M) = m_j(M) = s_i \cdot \left(\frac{\frac{g(1)}{n}\sum_{i=1}^{M} s_i \Delta_{i} + D_1}{\frac{g(1)}{n}\sum_{i=1}^{M} s_i + 1} - \Delta_{j}\right), \ \forall j \leq M,
$$
thus, applying the described binary search procedure to find $M^*$ such that $$M^{*} +1 = \min\left(\argmin_{M} \ind [m_{M}(M) > 0]\right)$$ finishes the proof.

We now elaborate on the computations. For the compactness of the notation, we omit the dependence on active set in $m_i$'s and $m$. We apply Lemma \ref{lem:LBD_KKT_general} with
$$
f(x) = \frac{g(1)}{n} \cdot x^2, \quad f_i(x) = \frac{x^2}{s_i} - 2D_ix,
$$
which gives
$$
f'(x) = \frac{2g(1)}{n} \cdot x, \quad f_i'(x) = \frac{2x}{s_i} - 2D_i.
$$
Hence, we obtain that
$$
(-f'_i)^{-1}(x) = \frac{s_i \cdot (2D_i-x)}{2},
$$
and, thus, by \eqref{eq:m_computation} we obtain
$$
m = \sum_{i=1}^{M} (-f'_i)^{-1}(f'(m)) = -f'(m) \cdot \sum_{i=1}^{M} \frac{s_i}{2} + \sum_{i=1}^{M} D_is_i =  -\frac{g(1)}{n} \cdot m \cdot \sum_{i=1}^{M} s_i + \sum_{i=1}^{M} D_is_i.
$$
In this view, we get
$$
m = \frac{\sum_{i=1}^{M} D_is_i}{\frac{g(1)}{n} \sum_{i=1}^{M} s_i + 1},
$$
and, hence, since by \eqref{eq:mi_computation} the following holds
$$
m_j = (-f'_j)^{-1}(f'(m)),
$$
we get
\begin{equation*}
\begin{split}
m_j &= s_j \cdot \frac{2D_j - f'(m)}{2} = s_j \cdot \frac{2D_j \left(\frac{g(1)}{n} \sum_{i=1}^{M} s_i + 1 \right) - \frac{2g(1)}{n} \cdot \sum_{i=1}^{M} D_is_i }{2 \cdot \left(\frac{g(1)}{n} \sum_{i=1}^{M} s_i + 1 \right)} \\
&= s_j \cdot \frac{\frac{g(1)}{n} \sum_{i=1}^M D_j s_i + D_j - \frac{g(1)}{n} \sum_{i=1}^M D_i s_i +\frac{g(1)}{n} \sum_{i=1}^M D_1 s_i - \frac{g(1)}{n} \sum_{i=1}^M D_1 s_i - D_1 + D_1}{\frac{g(1)}{n} \sum_{i=1}^{M} s_i + 1} = \\
&=  s_j \cdot \left(\frac{\frac{g(1)}{n}\sum_{i=1}^M s_i \Delta_i + D_1}{\frac{g(1)}{n} \sum_{i=1}^{M} s_i + 1} - \Delta_j\right),
\end{split}
\end{equation*}
where $\Delta_j = D_1 - D_j$. It is easy to verify that the condition 
$$
\frac{g(1)}{n}\sum_{j=1}^{M^{*} + 1} s_j (D_{M^{*} + 1} - D_j) + D_{M^{*} + 1} \leq 0
$$
described in the statement of the lemma is equivalent to $\beta_{M^{*} + 1}(M^{*} + 1) = m_{M^{*} + 1}(M^{*} + 1) \leq 0$, which concludes the proof.
\end{proof}

\begin{proof}[Proof of Theorem \ref{thm:wtD_lb}]
    We start by showing how the lower bound reduces to the objective in \eqref{eq:popriskDLB_cv}.
Consider the following block decomposition of $\B$ in accordance with $\D$ as in \eqref{eq:weight_tying_noniso}
\begin{equation*}
    \B = [\bGamma_1\B_1|\cdots|\bGamma_K \B_K],
\end{equation*}
where $\B_j \in \mathbb{R}^{n\times k_j}$ with $\|(\B_j)_{i,:}\|_2=1$ and $\{\bGamma_j\}_{j=1}^K$ are diagonal matrices. 

Since we require $\|\B_{i,:}\|_2=1$, the $\bGamma_i$ must satisfy 
\begin{equation}\label{eq:gamma_constraint}
    \sum_{j=1}^K\bGamma_j^2 = \I.
\end{equation}
Thus, up to a multiplicative factor $1/d$ and an additive term $\tr{\D^2}$, the objective \eqref{eq:DPR_obj} can be written as:
\begin{equation}\label{eq:popriskDreduced}
    \beta^2\left(\tr{\M f(\M)}\right)- 2c_1\beta \cdot \sum_{i=1}^K D_i \cdot \tr{\bGamma_i^2},
\end{equation}
where $\M = \sum_{i=1}^K \M_i:= \sum_{i=1}^K \bGamma_i \B_i \B_i^\top \bGamma_i$. Recall that $f(x) = c_1^2 x + g(x)$, where $g$ is the sum of odd monomials. Hence, we will be able to lower bound the terms in the first trace of \eqref{eq:popriskDreduced} in a similar fashion to Proposition \ref{prop:highrate_min}.
Note that 
$$\tr{\M_i^2} = \langle \1, \M_i^{\circ 2} \1\rangle,$$
so
applying Theorem {A} in \cite{khare2021sharp} gives that
$$
(\bGamma_i\B_i\B_i^\top\bGamma_i)^{\circ{2}} \succeq \frac{1}{s_i}  \cdot \mathrm{Diag}(\bGamma_i^2) \mathrm{Diag}(\bGamma_i^2)^\top,
$$
where $s_i = \mathrm{rank}(\B_i\B_i^\top)$.
Thus, we have the bound
$$\tr{\M_i^2} \geq \frac{1}{s_i}\left(\tr{\bGamma_i^2}\right)^2$$
Since $xg(x)\geq 0$, we can lower bound the rest of the terms with the identity, i.e., 
$$
\tr{\M g(\M)} =\langle \1, \M \circ g(\M) \1 \rangle \geq g(1) \cdot n
$$
as $\mathrm{Diag}(\M) = \I$. Consequently, neglecting the cross-terms $\tr{\M_i \M_j}$ (as the trace of the product of PSD matrices is non-negative) we arrive at
$$
\tr{\M f(\M)}  \geq  g(1) \cdot n + c_1^2 \cdot\sum_{i=1}^K \frac{1}{s_i}\left(\tr{\bGamma_i^2}\right)^2.
$$
Defining $\gamma_i := \tr{\bGamma_i^2} \geq 0$, we arrive at the following lower bound on \eqref{eq:popriskDreduced}:
\begin{equation}\label{eq:popriskDLB}
    \beta^2 \left(g(1) \cdot n + \sum_{i=1}^K \frac{\gamma_i^2}{s_i}\right) - 2 \beta \cdot \sum_{i=1}^K D_i \gamma_i,
\end{equation}
where, with an abuse of notation, we rescale $g(1):=g(1)/c_1^2$ and $\beta := c_1\beta$.
Now, by choosing $\beta_i := \beta \gamma_i$ and using that  $\sum_{i=1}^K\gamma_i = n$ due to \eqref{eq:gamma_constraint}, the objective \eqref{eq:popriskDLB} is seen to be equivalent to \eqref{eq:popriskDLB_cv}. This shows that $  \eqref{eq:DPR_obj} \geq  \mathrm{LB}(\D)$.
We now give a brief outline of how one can obtain the optimal $s_i$ and $\beta_i$ for \eqref{eq:popriskDLB_cv}.

For finding the optimal $s_i$, it is more natural to still consider \eqref{eq:popriskDLB}.
Due to the block form \eqref{eq:weight_tying_noniso}, the $s_i$ have to satisfy the 
constraints in \eqref{eq:rank_constraints}.
Note that \eqref{eq:popriskDLB} evaluated at the optimal $\beta$ is equal to
\begin{equation}\label{eq:popriskDreduced2}
    \eqref{eq:popriskDLB} \geq -\frac{\left(\sum_{i=1}^K D_i\gamma_i\right)^2}{\left(g(1) \cdot n + \sum_{i=1}^K\frac{\gamma_i^2}{s_i}\right)}.
\end{equation}
The optimal $s_i$ for this objective are \emph{water-filled}, i.e.,
\begin{equation}\label{eq:water_filling}
    \begin{cases}
    \boldsymbol{s} = [n,0,\cdots,0], & n \leq k_1,\\
    \boldsymbol{s} = [k_1,k_2,\cdots,k_K], & d \leq n,\\
    \boldsymbol{s} = [k_1,\cdots,k_{\mathrm{id}(n) - 1},\mathrm{res}(n),0,\cdots,0] & \text{otherwise},
    \end{cases}
\end{equation}
where $\boldsymbol{s} = [s_1,\cdots,s_k]$ and $\mathrm{id}(n)$ denotes the first position at which
$$
\min\{n,d\} - \sum_{i=1}^{\mathrm{id}(n)} k_i < 0,
$$
and 
$$
\mathrm{res}(n) = \min\{n, d\} - \sum_{i=1}^{\mathrm{id}(n) - 1} k_i.
$$
This follows directly from Lemma \ref{lemma:two_buckets}.
It only remains to show that the optimal $\beta_i$ can be obtained via \eqref{eq:betaimain}, which is done in Lemma \ref{lem:LBD_KKT}. This concludes the proof.
\end{proof}

\begin{proof}[Proof of Proposition \ref{prop:Dlb_achievability}]
Except for terms of the form $\tr{\B_i \B_i^\top \B_j \B_j^\top}$, all the other terms can be estimated as in the proof of Proposition \ref{prop:highrate_min}. 
The only technical difference is that all the constants now depend on the ratios $\frac{k_i}{n}$.

We will show that, with probability at least $1- c \exp\left(-cd^\epsilon\right)$, for all $i \neq j$,
\begin{equation}\label{eq:trBlast}
 \tr{\B_i \B_i^\top \B_j \B_j^\top} \leq n^{\frac{1}{2}+\epsilon}.  
\end{equation}
Thus, by a simple union bound, we have that, with probability at least $1-\frac{c}{d^2}$, this bound holds jointly for all pairs $\B_i, \B_j$. 
It follows as in the proof of Lemma \ref{lem:optBresult} that we can write 
\begin{equation*}
  \B_i\B_i^\top = \P_i \U \D_i \U^\top \P_i,
\end{equation*}
where by abuse of notation we pushed the factor $\frac{n}{k_i}$ in $\D_i$ (which will only affect the constants $c,C$).
Here, $\P_i$ is a diagonal matrix such that, for any $\epsilon >0$, with probability at least $1-c \exp\left(-cd^\epsilon\right)$, we have that
 $$\opn{\P_i - \I} \leq n^{-\frac{1}{2} + \epsilon}.$$
To see this, first observe that $\Theta : (\R^{n \times n})^4 \mapsto \R$ given by 
$$ \Theta(\X_1,\X_2,\X_3,\X_4) = \tr{\X_1\U\D_i\U^\top \X_2 \X_3 \U\D_j\U^\top \X_4}$$ is  differentiable (as it is the composition of the trace function with 4-linear form).
 Since by construction 
 $$\tr{\U\D_i\U^\top  \U\D_j\U^\top} = \tr{\0} = 0,$$
 this implies that, with probability at least $1-\frac{c}{d^2}$,
 \begin{align*}
     0 &\leq \tr{\B_i\B_i^\top \B_j\B_j^\top} \\&= \tr{\P_i\U\D_i\U^\top \P_i \P_j \U\D_j\U^\top \P_j} \\
     &= \tr{\P_i\U\D_i\U^\top \P_i \P_j \U\D_j\U^\top \P_j}- \tr{\U\D_i\U^\top  \U\D_j\U^\top} \\
     &\leq Cnn^{-\frac{1}{2} + \epsilon},
 \end{align*}
 where in the last step we used that the derivative of the trace function is bounded by $n \cdot \opn{\cdot}$.
Thus, \eqref{eq:trBlast} holds.

By construction, the sum of all the cross terms is of the form 
$$\sum_{i \neq j}\tr{\M_i\M_j} ,$$
where $\M_i = \bGamma_i \B_i \B_i^\top \bGamma_i$, $\,\bGamma_i^2 = \frac{\gamma_i}{n} \I$ and $\sum_{i=1}^K\gamma_i = n$.
We have 
\begin{align*}
    \abs{\sum_{i \neq j}\tr{\M_i\M_j}} &= \abs{\sum_{i \neq j} \frac{\gamma_i \gamma_j}{n^2} \tr{\B_i\B_i^\top \B_j \B_j^\top}}\\
    &\leq \sum_{i\neq j} \frac{\gamma_i \gamma_j}{n^2}\abs{\tr{\B_i\B_i^\top \B_j \B_j^\top}}\\
    &\leq C\sum_{i\neq j} \frac{\gamma_i \gamma_j}{n^2} n^{\frac{1}{2} + \epsilon}\\
    &\leq Cn^{\frac{1}{2} + \epsilon},
\end{align*}
where in the third step we used a union bound on \eqref{eq:trBlast}  and in the last step we used $\sum_{i=1}^K\frac{\gamma_i}{n}=1$.
\end{proof}

\section{Details of Experiments and Additional Numerical Results}\label{appendix:numerics}

We first describe the training details 
and the whitening procedure that is used to preprocess natural images for MNIST (Figure \ref{figure:mnist_app}) and CIFAR-10 (Figures \ref{fig:cifar_white}, \ref{fig:augcomp} and \ref{fig:wcifapp}). Next, we give some remarks about the experiments concerning VAMP (Figure \ref{fig:comparison}) and about the discontinuous behaviour of the derivative of the lower bound highlighted in Figure \ref{fig:noniso_exps}. In addition, we present supplementary numerical experiments which cover extra classes of natural images.

\paragraph{Activation function and weight parameterization.} Note that the derivative of the sign activation is zero almost everywhere (except one point, which is the origin). In this view, we cannot use conventional gradient-based algorithms to find the optimal set of parameters for an autoencoder with the sign activation. We tackle this issue by using a straight-through estimator (see, for instance, \cite{yin2019understanding}) of the sign activation. During the forward pass
the activations of the first layer are computed for $\sigma(x) = \mathrm{sign}(x)$, while during the backward pass $\sigma(x) = \mathrm{tanh}(x/\tau)$ is used. Here, the temperature parameter $\tau>0$ controls how well the differentiable surrogate $\mathrm{tanh}(x/\tau)$ approximates $\mathrm{sign}(x)$, as 
$$
\lim_{\tau\rightarrow 0}\mathrm{tanh}(x/\tau) = \mathrm{sign}(x), \quad \forall x\in\mathbb{R}\setminus\{0\}.
$$
More precisely, the differentiable approximation becomes more accurate for smaller values of $\tau$. However, we also note that extremely small values of $\tau$ might cause numerical issues, since the  derivative of the differentiable surrogate diverges at the origin as $\tau \rightarrow 0$. For the numerical experiments, we 
pick $\tau\in[0.01,0.2]$, with the exact value depending on the specific setting.

Note that the constraint on the encoder weights $\|\B_{i,:}\|_2 = 1$ can be enforced via a simple reparameterization that forces the rows of $\B$ to lie on 
the unit sphere $\mathbb{S}^{d-1}$. More precisely, we use the following classical differentiable reparameterization of $\B^\top = [\b_1,\cdots,\b_n]$, where
$$
\b_i = \frac{\hat{\boldsymbol{b}}_i}{\norm{\hat{\boldsymbol{b}}_i}_2},
$$
with $\{\hat{\b}_i\}_{i=1}^n$ being the trainable parameters. We note that it is not clear a priori
whether we need to impose the constraints directly for the straight-through estimator, since during the forward pass we use the norm-agnostic $\mathrm{sign}$ function.

\begin{figure}[t]
\begin{tabular}{@{}c@{}cc@{}}
    \raisebox{-\height}{\includegraphics[width=0.4\textwidth]{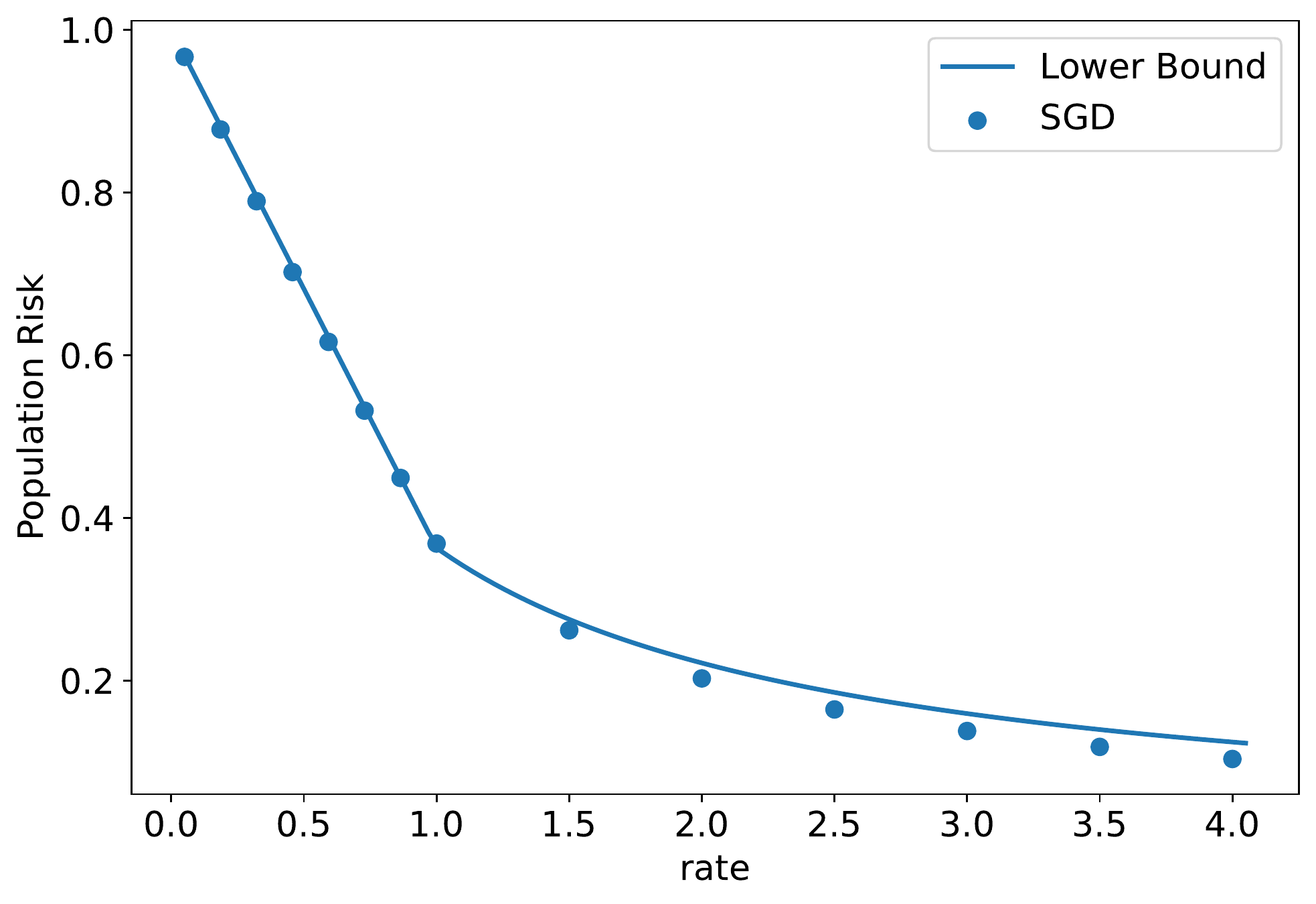}} & 
     \hspace{1.2em}\raisebox{-\height}{\includegraphics[width=0.4\textwidth]{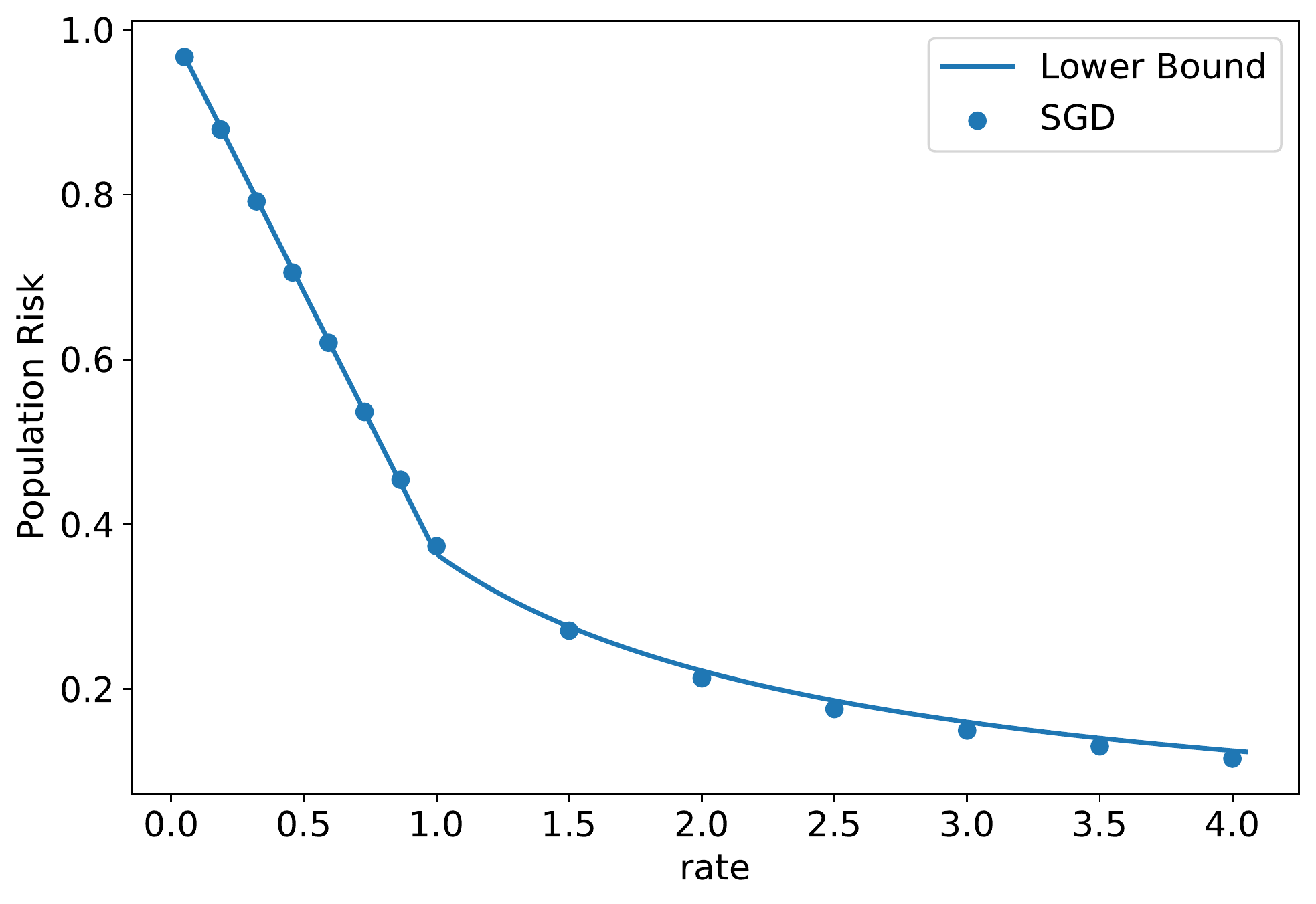}} &
    \hspace{1.0em}\begin{tabular}[t]{@{}cc@{}}
        \raisebox{-\height}{\includegraphics[width=0.13\textwidth]{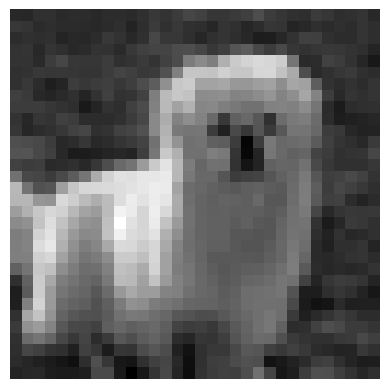}} &  \\[1.cm]
        \raisebox{-\height}{\includegraphics[width=0.13\textwidth]{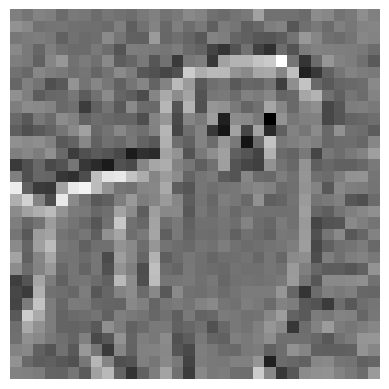}} & 
    \end{tabular}
\end{tabular}

\caption{Compression ($\sigma\equiv {\rm sign}$) of the CIFAR-10 ``dog'' class with a two-layer autoencoder. The data is \emph{whitened} so that $\boldsymbol{\Sigma}=\bI$: on top, an example of a grayscale image; on the bottom, the corresponding whitening. The blue dots are the population risk obtained via SGD, and they agree well with the solid line corresponding to the lower bounds of Theorem \ref{thm:tightlb_lowrate} and Proposition \ref{prop:highratelb}. Here, the effect of the number of augmentations used per image is shown. For the left plot each image was augmented $10$ times, while for the right plot each image was augmented $15$ times.}\label{fig:augcomp}\vspace{-1.2em}
\end{figure}

\paragraph{Augmentation and whitening.} For the experiments on natural images, we augment the data of each class $15$ times. This is done to emulate the optimization of the population risk, since the amount of initial data (approximately $5000$ samples per class) leads to a 
gap between empirical and population risks, 
especially for high rates. 
The effect of the data augmentation is represented in 
Figure \ref{fig:augcomp} for the whitened CIFAR-10 ``dog'' class. It can be seen that a mild amount of augmentation, i.e., $\times 10$ and $\times 15$, is already enough for our purposes, and the difference between the two plots is rather small. Notably, this amount of augmentation brings the dataset to the scale of the original data when all classes are considered (around $50000$ training examples).

The whitening procedure used in the experiments concerning isotropic data is performed as follows: given the \emph{centered} augmented data $\X \in \mathbb{R}^{\mathrm{n_{\rm samples}}\times d}$, we compute its empirical covariance matrix given by
$$
\hat{\bSigma} = \frac{1}{\mathrm{n_{\rm samples}} - 1} \cdot \sum_{i=1}^{\mathrm{n_{\rm samples}}} \X_{i,:} \X_{i,:}^\top,
$$
and then we multiply each input by the inverse square root of it, i.e.,
$$
\hat{\X}_{i,:} = \hat{\bSigma}^{-\frac{1}{2}} \X_{i,:}.
$$
The resulting whitened images
are represented in Figures \ref{fig:cifar_white}, \ref{fig:augcomp} and \ref{figure:mnist_app}.

In the experiments concerning non-isotropic data (Figures \ref{fig:cifar_white} and \ref{fig:nowhitecifarapp}), we center the data with the empirical mean and divide by a \emph{scalar} empirical variance computed across all the pixels, which is the standard preprocessing procedure widely used for computer vision tasks.

\paragraph{VAMP experiments.}
For the VAMP experiments, we implement the State Evolution (SE) recursion which exactly characterizes the limiting performance of VAMP as $d \to \infty$, see \cite{schniter2016vector,rangan2019vector} for an overview. We then plot the fixed point of said SE recursion.
A concrete description for VAMP is provided by Algorithm 2 in \cite{fletcher2018inference}, which however covers a more general multi-layer setting. 

\paragraph{``Jumps'' of the lower bound derivative.} The derivative switch described in Figure \ref{fig:noniso_exps} does not necessarily happen precisely at the point when the block is filled. A switch may occur at a later point since, even if $s_i > 0$, the corresponding optimal $\beta_i$ may be $0$. Intuitively, this phenomenon occurs in cases when it is still better to put more mass in the block where the rank is utilized to the fullest ($s_j = k_j$). This corresponds to the following condition on the derivatives of the objective \eqref{eq:popriskDLB_cv}: $$\frac{\partial \eqref{eq:popriskDLB_cv}}{\partial \beta_i}(0) > \frac{\partial \eqref{eq:popriskDLB_cv}}{\partial \beta_{j}}(\beta_{j}^{*}),$$ where $\beta_{i}^{*}$ stands for the optimal $\beta_{i}$ and $j$ denotes the first index at which $\beta_j^{*} >0$.
This behaviour occurs when
the spectrum $\D$ has a large variation in scale, e.g.,
$$
\D = [5, 0.02, 0.01].
$$
In this case, the last components will be utilized for $n$ significantly larger than $k_1$ ($n=k_1$ precisely characterizes the point where the rank of the first block of $\B$, i.e., $\B_1$, is the maximum possible). Note that, for this choice of $\D$, 
the plot of the derivative analogous to Figure \ref{fig:noniso_exps}
will not indicate such prominent ``jumps''. In fact, the contribution of the last components to the derivative value is less significant in comparison to the analogous quantity evaluated for the top-most eigenvalues.

\paragraph{Additional experimental data.} We also provide additional numerical simulations, similar to those presented in the body of the paper. In particular, we provide more class variations for the natural data experiments (MNIST and CIFAR-10).

\vspace{-1.2em}\begin{figure}[h]
\begin{tabular}{@{}cc@{}@{}cc@{}}
    \raisebox{-\height}{\includegraphics[width=0.35\textwidth]{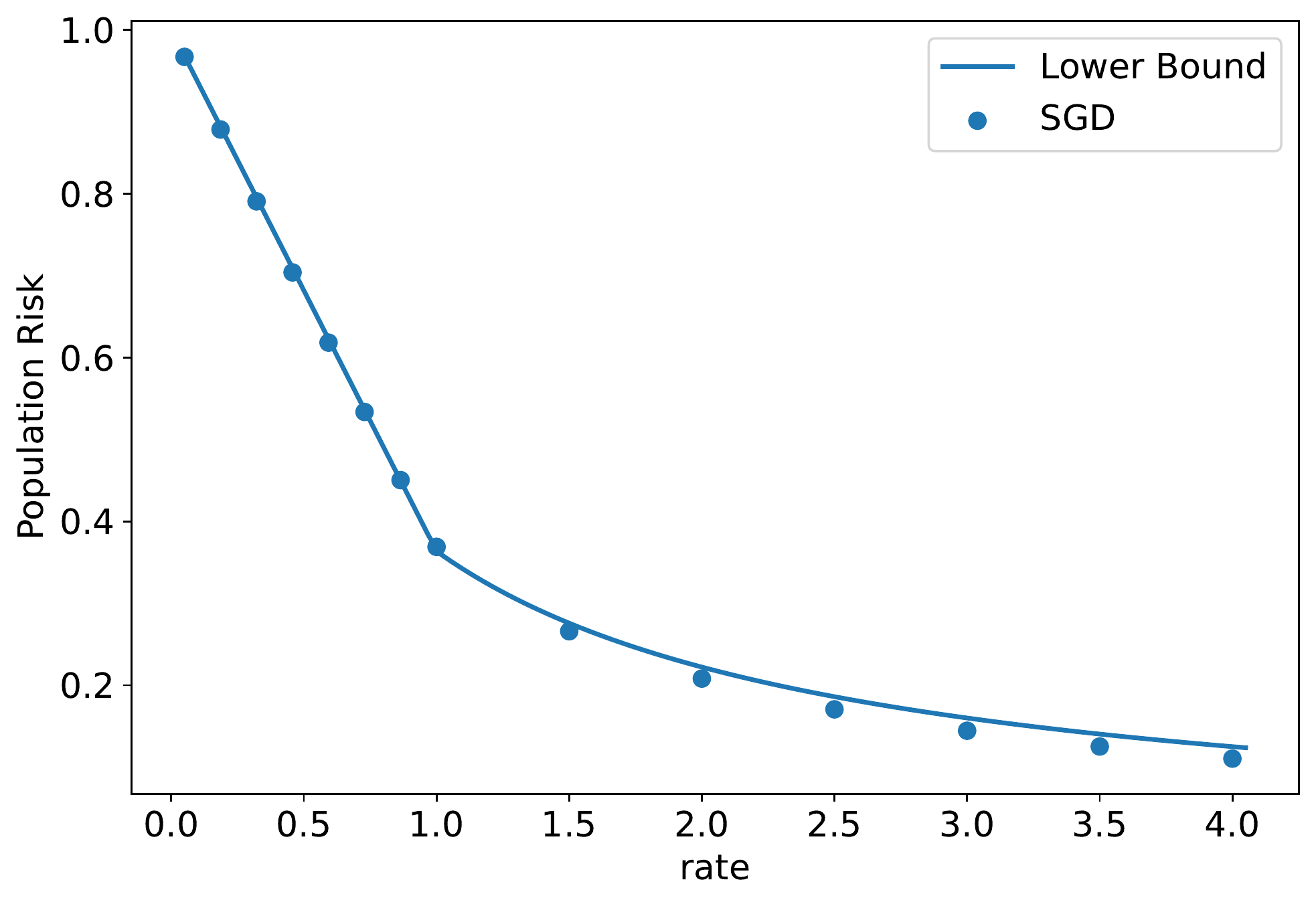}} & 
    \begin{tabular}[t]{@{}cc@{}}
        \raisebox{-\height}{\includegraphics[width=0.11\textwidth]{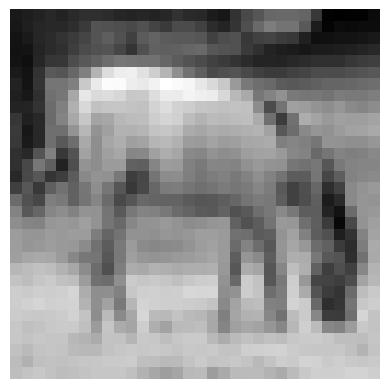}} &  \\[1.cm]
        \raisebox{-\height}{\includegraphics[width=0.11\textwidth]{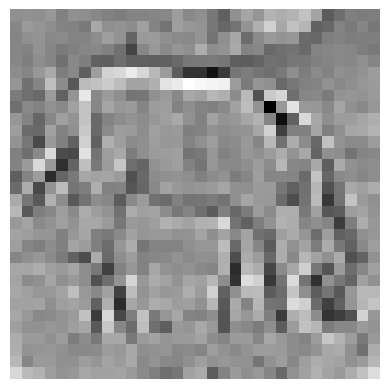}} & 
    \end{tabular}
    \raisebox{-\height}{\includegraphics[width=0.35\textwidth]{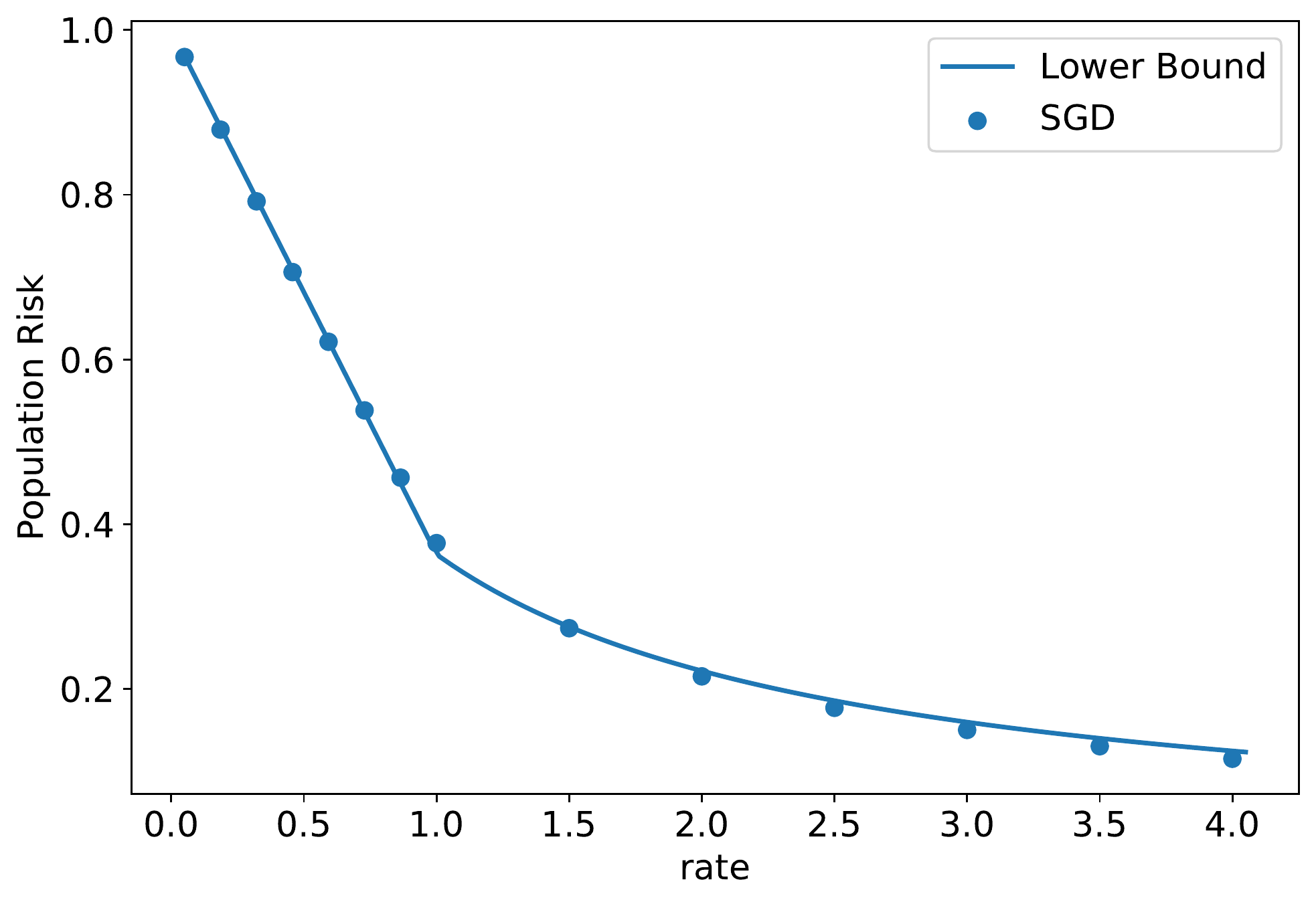}} &
    \hspace{1.2em}\begin{tabular}[t]{@{}cc@{}}
        \raisebox{-\height}{\includegraphics[width=0.11\textwidth]{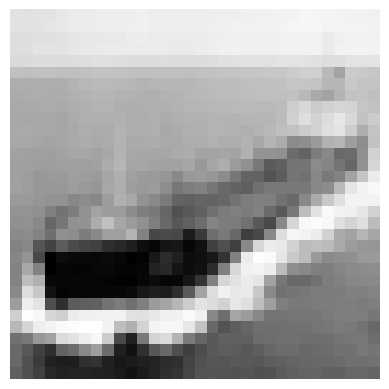}} &  \\[1.cm]
        \raisebox{-\height}{\includegraphics[width=0.11\textwidth]{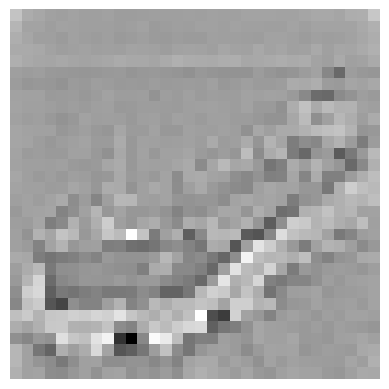}} & 
    \end{tabular}
\end{tabular}
\caption{Compression ($\sigma\equiv {\rm sign}$) of the CIFAR-10 ``horse'' class (left) and ``ship'' class (right) with a two-layer autoencoder. The data is \emph{whitened} so that $\boldsymbol{\Sigma}=\bI$: on top, an example of a grayscale image; on the bottom, the corresponding whitening. The blue dots are the population risk obtained via SGD, and they agree well with the solid line corresponding to the lower bounds of Theorem \ref{thm:tightlb_lowrate} and Proposition \ref{prop:highratelb}. Here, in both cases the amount of augmentations per image is equal to $15$.}\label{fig:wcifapp}\vspace{-2em}
\end{figure}

\begin{figure}[h]
\begin{tabular}{@{}cc@{}@{}cc@{}}
    \raisebox{-\height}{\includegraphics[width=0.35\textwidth]{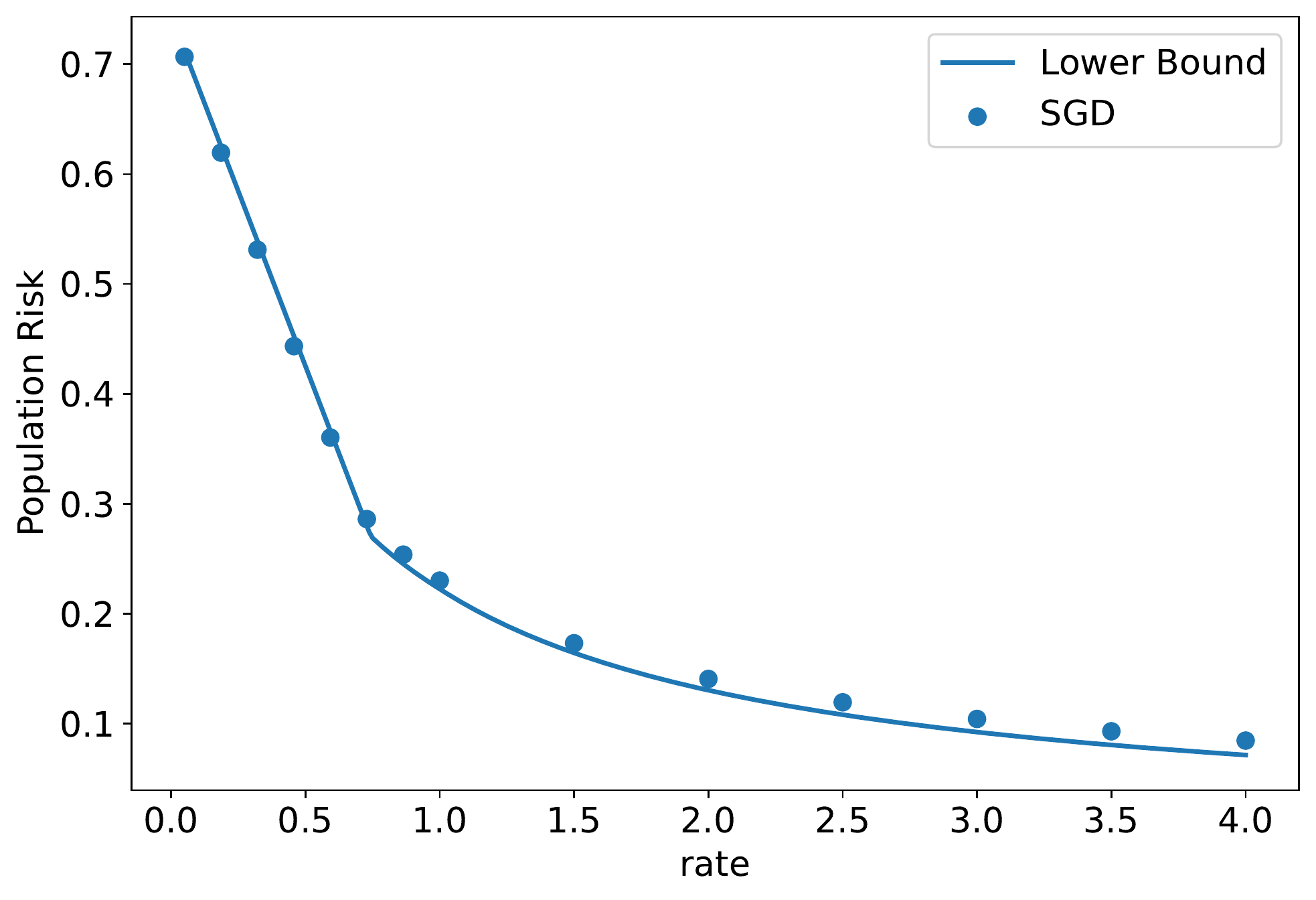}} & 
    \begin{tabular}[t]{@{}cc@{}}
        \raisebox{-\height}{\includegraphics[width=0.11\textwidth]{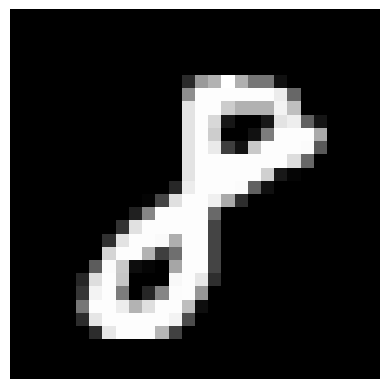}} &  \\[1.cm]
        \raisebox{-\height}{\includegraphics[width=0.11\textwidth]{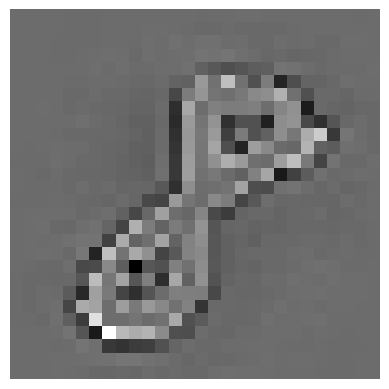}} & 
    \end{tabular}
    \raisebox{-\height}{\includegraphics[width=0.35\textwidth]{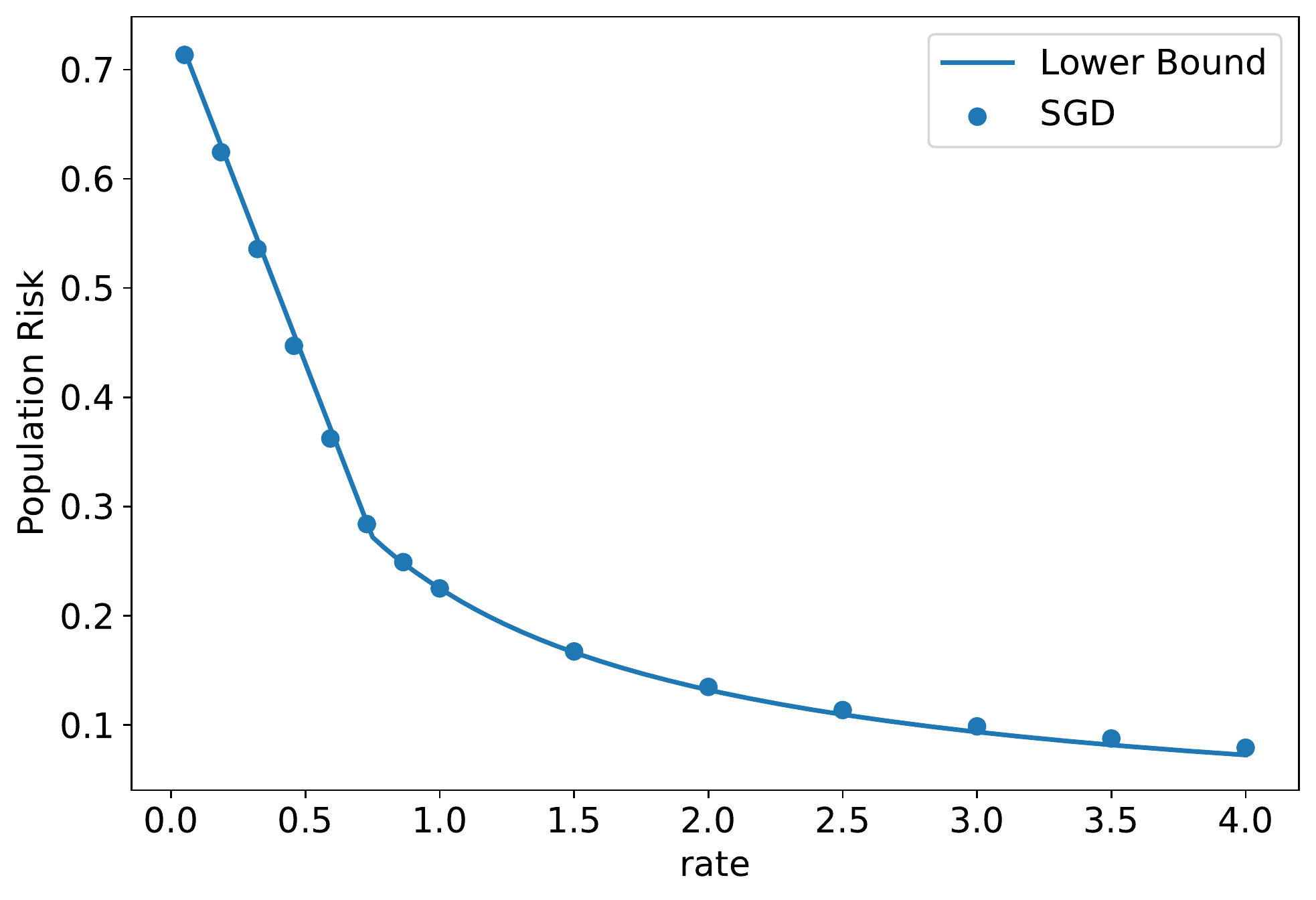}} &
    \hspace{1.2em}\begin{tabular}[t]{@{}cc@{}}
        \raisebox{-\height}{\includegraphics[width=0.11\textwidth]{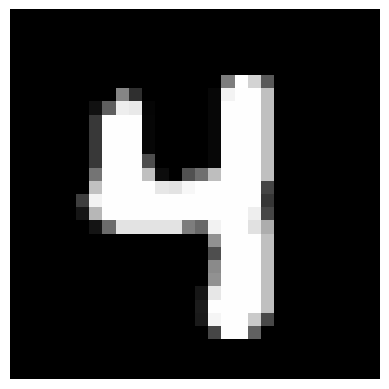}} &  \\[1.cm]
        \raisebox{-\height}{\includegraphics[width=0.11\textwidth]{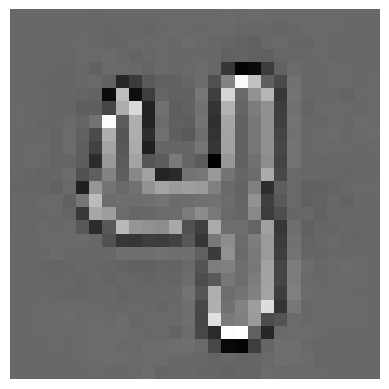}} & 
    \end{tabular}
\end{tabular}
\caption{
Compression ($\sigma\equiv {\rm sign}$) of the MNIST ``8'' class (left) and ``4'' class (right) with a two-layer autoencoder. The data is \emph{whitened} so that $\boldsymbol{\Sigma}=\bI$: on top, an example of a grayscale image; on the bottom, the corresponding whitening. The blue dots are the population risk obtained via SGD, and they agree well with the solid line corresponding to the lower bounds of Theorem \ref{thm:tightlb_lowrate} and Proposition \ref{prop:highratelb}. Here, in both cases the amount of augmentations per image is equal to $10$.
}\label{figure:mnist_app}
\end{figure}

\makeatletter
\setlength{\@fptop}{0pt}
\makeatother
\begin{figure}[t!]
\begin{tabular}{@{}cc@{}@{}cc@{}}
    \raisebox{-\height}{\includegraphics[width=0.35\textwidth]{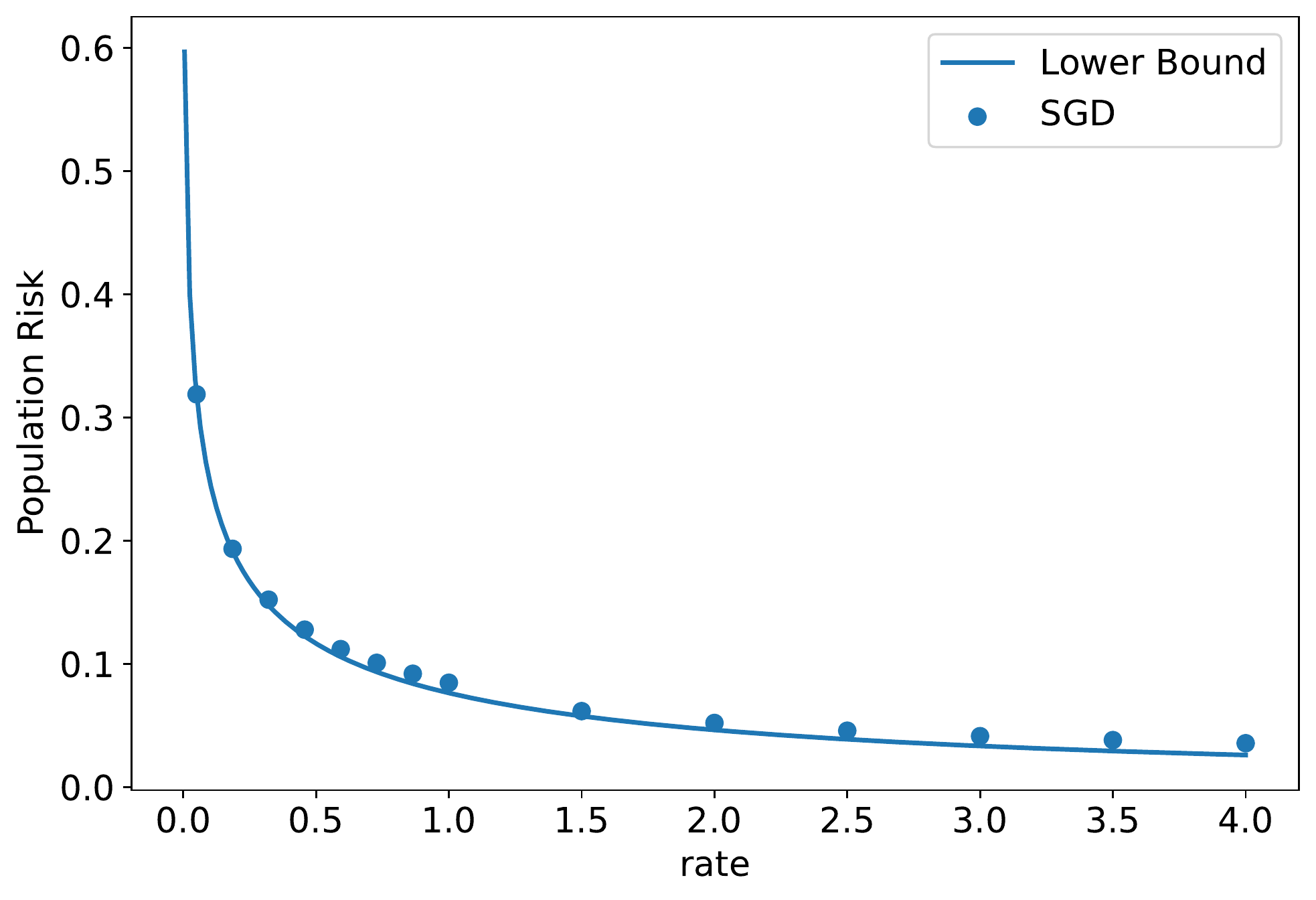}} & 
    \begin{tabular}[t]{@{}cc@{}}
        \raisebox{-\height}{\includegraphics[width=0.11\textwidth]{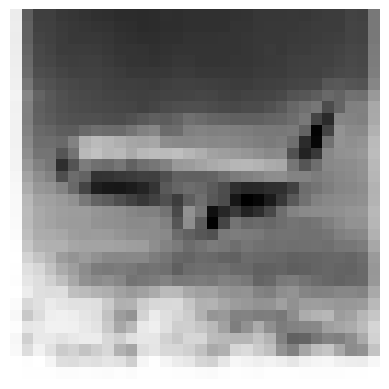}} &  \\[1.cm]
        \raisebox{-\height}{\includegraphics[width=0.11\textwidth]{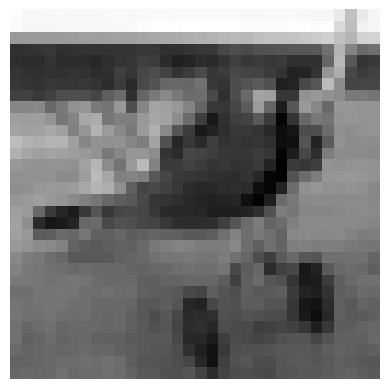}} & 
    \end{tabular}
    \raisebox{-\height}{\includegraphics[width=0.35\textwidth]{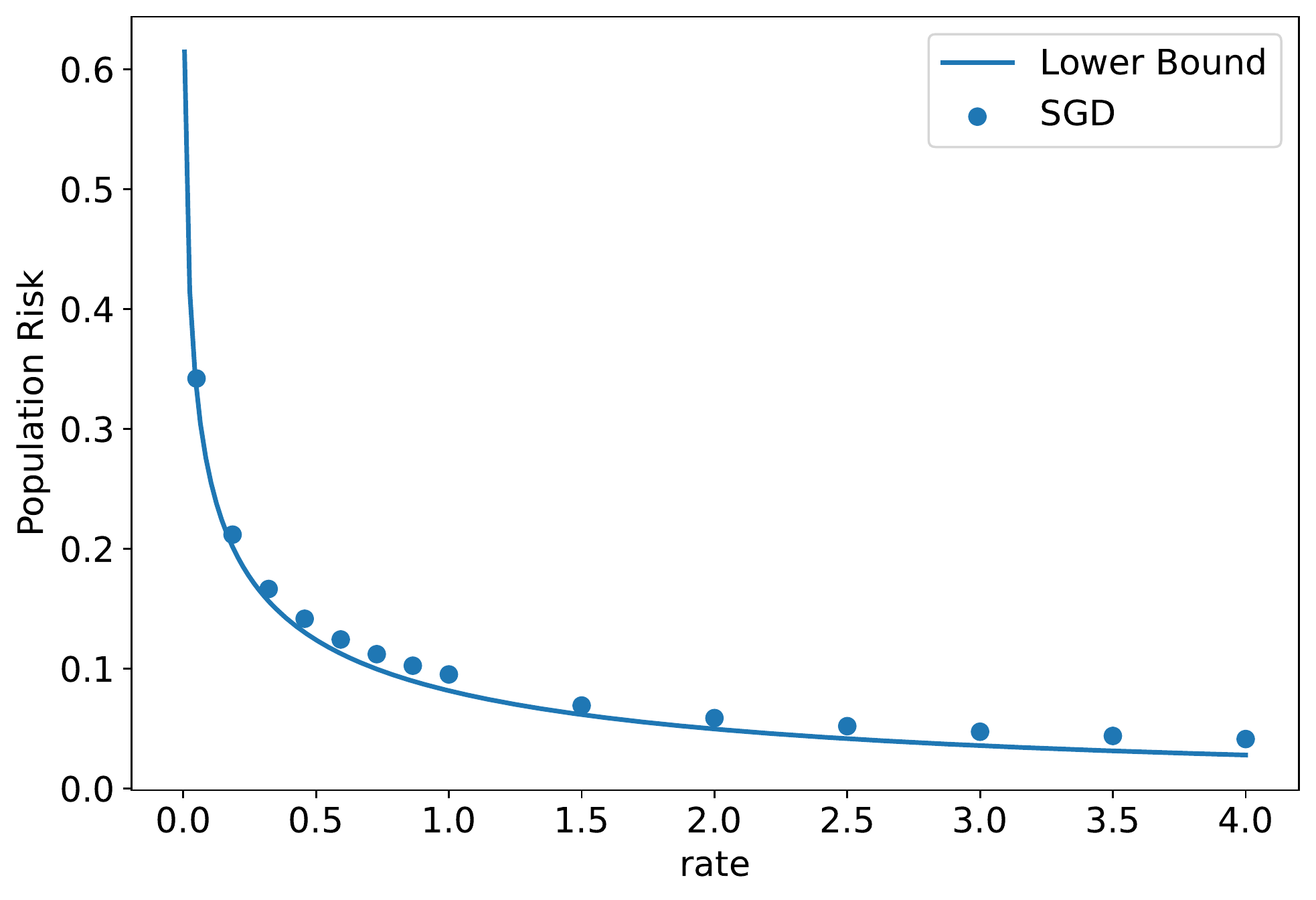}} &
    \hspace{1.2em}\begin{tabular}[t]{@{}cc@{}}
        \raisebox{-\height}{\includegraphics[width=0.11\textwidth]{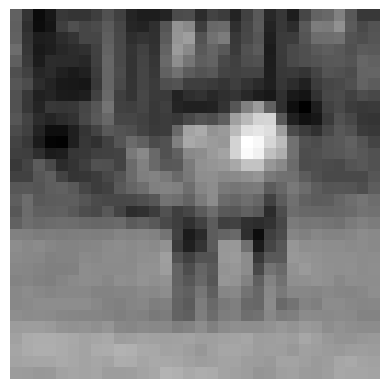}} &  \\[1.cm]
        \raisebox{-\height}{\includegraphics[width=0.11\textwidth]{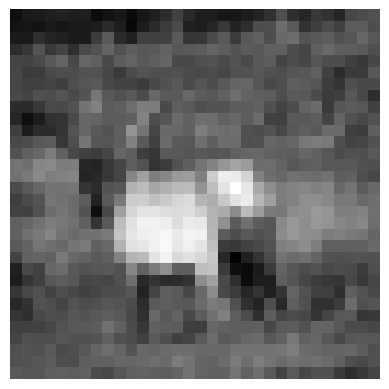}} & 
    \end{tabular}
\end{tabular}
\caption{
Compression ($\sigma\equiv {\rm sign}$) of the CIFAR-10 ``airplane'' class (left) and ``deer'' class (right) with a two-layer autoencoder. The data is \emph{not whitened} ($\boldsymbol{\Sigma}\neq \bI$). The blue dots are the SGD population risk, and they are close to the lower bound of Theorem \ref{thm:wtD_lb}.
Here, in both cases the amount of augmentations per image is equal to $15$.}\label{fig:nowhitecifarapp}
\end{figure}

\end{document}